\documentclass{article}
\PassOptionsToPackage{table}{xcolor}
\usepackage{iclr2027_conference,times}

\usepackage[hypertexnames=false]{hyperref}
\usepackage{xurl}
\usepackage{graphicx}
\usepackage{amsmath,amssymb,amsthm}
\usepackage{algorithm}
\usepackage{algpseudocode}
\usepackage{mathtools}
\usepackage{float}
\usepackage{enumitem}
\usepackage{booktabs}
\usepackage{multirow}
\usepackage{tabularx}
\usepackage{wrapfig}

\usepackage[normalem]{ulem}
\newcommand{\tightuline}[1]{{\setlength{\ULdepth}{1.1pt}\uline{#1}}}

\newcommand{\R}{\mathbb{R}}

\newcommand{\Np}{\mathbb{N}^+}
\newcommand{\Z}{\mathbb{Z}}
\newcommand{\clip}{\operatorname{clip}}
\newcommand{\round}{\operatorname{round}}
\newcommand{\ctx}{\mathrm{ctx}}

\newcommand{\poly}{\operatorname{poly}}

\newcommand{\enc}{\operatorname{enc}}
\newcommand{\dec}{\operatorname{dec}}

\newcommand{\VanillaFM}{\textsc{VanillaFM}}
\newcommand{\SelfCondFM}{\textsc{SelfCondFM}}
\newcommand{\RKFM}{\textsc{RKFM}}

\newtheorem{theorem}{Theorem}
\newtheorem{corollary}{Corollary}
\newtheorem{lemma}{Lemma}
\newtheorem{proposition}{Proposition}
\theoremstyle{definition}

\newcommand{\Sref}[1]{\S\ref{#1}}
\newcommand{\Tref}[1]{Table~\ref{#1}}

\setlist[itemize]{leftmargin=*}
\setlist[enumerate]{leftmargin=*}
\setlist[itemize]{leftmargin=*, nosep}
\setlist[enumerate]{leftmargin=*, nosep}

\title{Self-Conditioning Enables Deep Sequential Reasoning in Flow Matching}
\author{\textbf{Zeyang Zhang}$^{1,2}$\thanks{Equal contribution.} \quad
  \textbf{Linxiang Gao}$^{2}$\footnotemark[1] \quad
  \textbf{Meiqi Gu}$^{2}$\footnotemark[1] \quad
  \textbf{Xingyan Chen}$^{2}$ \quad
  \textbf{Minrui Luo}$^{4,2}$ \\\textbf{Kaifeng Lyu}$^{2}$\thanks{Corresponding authors.} \quad
  \textbf{Tianxing He}$^{2,3,1}$\footnotemark[2] \\
  $^1$Shanghai Qi Zhi Institute \\
  $^2$Institute for Interdisciplinary Information Sciences, Tsinghua University \\
  $^3$Xiongan AI Institute \quad $^4$Georgia Institute of Technology \\
  \texttt{\{zhangzey24,glx24,gumq24,chenxy24\}@mails.tsinghua.edu.cn} \\
  \texttt{mluo89@gatech.edu}\quad \texttt{\{klyu,hetianxing\}@mail.tsinghua.edu.cn}
}
\date{\today}

\iclrfinalcopy

\begin{document}

\maketitle

\lhead{Preprint}

\begin{abstract}
Recent work has explored visual reasoning through image generation, where models solve visual-reasoning problems by directly generating solution images. Intuitively, iterative samplers such as flow matching are well suited to this task, since each sampling step may resemble one reasoning step. However, in this work, from a circuit complexity perspective, we show that this intuition is misleading. We prove that as long as the model outputs satisfy certain mild regularity conditions, the trajectory of the vanilla flow-matching sampler can be simulated by a constant number of sequential rounds of parallel model evaluations. Consequently, additional sampling steps cannot increase the expressive power of a constant-depth backbone model under the standard complexity-theoretic model of transformers. 
We then reveal that \emph{self-conditioning}, which conditions the model on the previous step's estimate of the final image, plays a previously unrecognized role in removing this bottleneck.
Controlled experiments on diverse visual-reasoning tasks show that flow matching with self-conditioning achieves higher accuracy and greater reasoning depth than vanilla flow matching, and its performance improves with more sampling steps. Theoretically, we prove that with self-conditioning, a constant-depth backbone can achieve expressive power equivalent to that of space-bounded Turing machines. 
Our study suggests a design principle for future visual reasoning research: carry more information between sampling steps, so that additional steps translate into deeper reasoning.\footnote{Our code is available in the \href{https://github.com/imzzy1201/self-conditioning-deep-reasoning/}{GitHub repository}.}
\end{abstract}

\section{Introduction}

Recent advances in visual generative models have sparked growing interest in
\emph{generative visual reasoning}, where models plan and reason directly in image space~\citep{su2025thinking,hu2024visual,tong2026thinking}. This paradigm
is natural for spatial planning, combinatorial optimization, and structured rendering~\citep{he2025diffthinker,dai2026endocot}.
Intuitively, flow matching~\citep{lipman2022flow,liu2022flow} is well suited to this setting: it turns noise into an image through $K$ evaluations of a velocity model, providing $K$ sequential rounds for reasoning to unfold. \citet{he2025diffthinker} put this intuition into practice, solving visual reasoning tasks via flow-matching image generation.

However, in this work we show that this apparent reasoning depth is an illusion.
Through \emph{circuit complexity} analysis, we find that the expressive power of vanilla flow matching ($\VanillaFM$) does not grow with the number of sampling steps $K$.
The intuition lies in how flow matching generates images: the sampling process discretizes an ordinary differential equation (ODE), making only a small incremental update of scale $O(1/K)$ to its state at each step.
Because the full trajectory is governed by this continuous ODE, we show that it can be simulated by an alternative ODE solver in a constant number of steps, regardless of how large $K$ is.
Consequently, running $\VanillaFM$ with a constant-depth backbone leaves its expressive power collapsed to constant reasoning depth (\Sref{sec:vanillabarrier}).

To overcome this barrier, we uncover a previously unrecognized role of \emph{self-conditioning}~\citep{chen2022analog}.
In flow matching with self-conditioning ($\SelfCondFM$), the model still simulates the ODE, but its input is augmented with a prediction of its own sampling trajectory's endpoint, i.e., an estimate of the final image
(\Sref{sec:scmechanism}).
This endpoint prediction acts as an internal scratchpad on which the model records and refines intermediate partial solutions.
Unlike the continuous state, which changes only by $O(1/K)$, this scratchpad can change abruptly between steps.
Therefore, the resulting trajectory can no longer be simulated by an ODE solver in a constant number of steps.
This prevents the collapse of expressive power, enabling deep sequential reasoning~(\Sref{sec:scpower}).

\begin{figure*}[t]
    \centering
    \includegraphics[width=\textwidth]{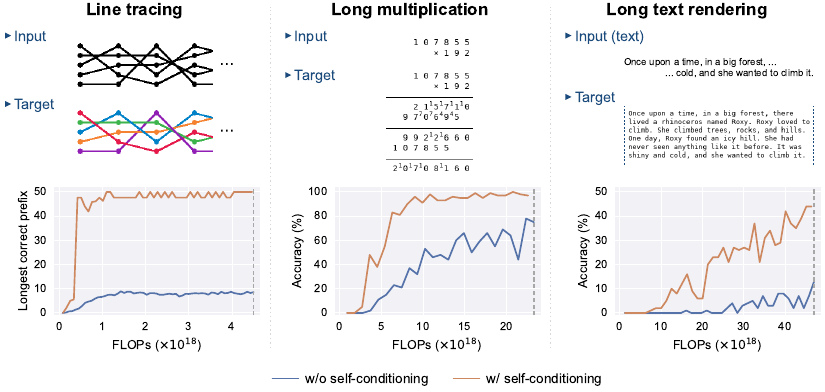}
    \vspace{-0.25in}
    \caption{Controlled experiments on three synthetic tasks: line tracing, long multiplication, and long text rendering (\Sref{sec:theory_aligned_tasks}). The top row shows an example input and target output for each task, and the bottom row plots performance against training FLOPs. Across all three tasks, self-conditioning improves the performance of flow matching by large margins.}
    \label{fig:theory_aligned_tasks}
    \vspace{-0.1in}
\end{figure*}

We establish this advantage of self-conditioning both theoretically and empirically.
Theoretically, we prove that self-conditioning can lift the expressive power of flow matching from constant-depth circuits to space-bounded Turing machines (\Sref{sec:expressivity}).
Empirically, controlled experiments on diverse synthetic tasks demonstrate that self-conditioning substantially improves the performance of flow matching (Figure~\ref{fig:theory_aligned_tasks}), and its advantage also extends to broader visual-reasoning tasks (\Sref{sec:experiments}).
Our contributions are as follows:

\begin{itemize}[leftmargin=*]
    \item \textbf{Self-conditioning as an internal scratchpad.} We reveal a previously unrecognized role of self-conditioning in generative reasoning: rather than serving merely as an empirical regularizer, its endpoint prediction provides an internal scratchpad that carries rich information, enabling deep sequential computation across sampling steps.
    \item \textbf{Circuit complexity separation.} We formalize this mechanism through circuit complexity analysis. We prove that as long as the model outputs satisfy certain mild regularity conditions, $\VanillaFM$ with an $\mathsf{AC}^0$ or $\mathsf{TC}^0$ backbone collapses to constantly many sequential model evaluations, deciding only languages in the backbone's circuit complexity class regardless of the number of sampling steps (\Sref{sec:vanillabarrier}). In contrast, $\SelfCondFM$ enables an $\mathsf{AC}^0$ model to simulate space-bounded Turing machines, deciding $\mathsf{P/poly}$ with polynomial steps and $\mathsf{PSPACE/poly}$ with at most exponential steps (\Sref{sec:scpower}). These results extend to more general settings (\Sref{sec:theoryextensions}).
    \item \textbf{Empirical validation.} We support these theoretical results through controlled experiments on three synthetic tasks with known circuit complexity: $\mathsf{NC}^1$-complete line tracing, and $\mathsf{TC}^0$-complete long multiplication and long text rendering (\Sref{sec:theory_aligned_tasks}). Under matched training FLOPs, $\SelfCondFM$ outperforms $\VanillaFM$ by large margins (Figure~\ref{fig:theory_aligned_tasks}). Moreover, in agreement with our theory, $\SelfCondFM$ keeps improving with more sampling steps, whereas $\VanillaFM$ saturates early (Figure~\ref{fig:nfe_scaling}). Self-conditioning also brings consistent improvements on five broader visual-reasoning tasks from~\citet{he2025diffthinker} (\Sref{sec:diffthinker_tasks}).
\end{itemize}

\section{Preliminaries}
\label{sec:preliminaries}

We first review flow matching and circuit complexity (\Sref{sec:background}), then formalize the modeling conventions and regularity conditions underlying our theoretical analysis (\Sref{sec:fmsetup}).

\subsection{Background}
\label{sec:background}

\paragraph{Flow matching.}
\label{sec:flowmatching}

Flow matching~\citep{lipman2022flow,liu2022flow} is a generative modeling framework. Given a data distribution $p_{\text{data}}$ and a prior $p_{\text{prior}}$, we define the probability path $p_t$ as the distribution of $\tilde x_t:=tx_{\text{data}}+(1-t)\epsilon$ for independent random vectors $x_{\text{data}}\sim p_{\text{data}}$ and $\epsilon\sim p_{\text{prior}}$. The velocity field that generates this path is given by $v(x,t):=\mathbb{E}_{(x_{\text{data}},\epsilon)\sim p(\cdot\mid \tilde x_t=x)}[x_{\text{data}}-\epsilon]$. Flow matching aims to find a parameterized field $u_\theta(x,t)$ to approximate $v(x,t)$, so that the solution of the ODE $\frac{\mathrm{d}x}{\mathrm{d}t}=u_\theta(x,t)$ with $x|_{t=0}\sim p_{\text{prior}}$ follows the marginals $p_t$ for $t\in[0,1]$.

Flow-matching models are trained by minimizing the conditional flow-matching loss $\mathcal{J}_{\text{CFM}}:=\mathbb{E}_{t\sim\mathcal{U}[0,1],\epsilon\sim p_{\text{prior}},x_{\text{data}}\sim p_{\text{data}}}\|u_\theta(\tilde x_t,t)-(x_{\text{data}}-\epsilon)\|_2^2$. \citet{lipman2022flow} prove that minimizing $\mathcal{J}_{\text{CFM}}$ is equivalent to minimizing $\mathbb{E}_{t,\epsilon,x_{\text{data}}}\|u_\theta(\tilde x_t,t)-v(\tilde x_t,t)\|_2^2$, so the trained model regresses onto the target field $v$. At inference time, a sample is generated by solving the ODE with a numerical scheme such as Euler's method~\citep{hairer1993solving,liu2022flow,esser2024scaling}.

For conditional generation, such as text-to-image and image editing, the model is conditioned on context $\ctx$, with $p_{\text{data}}$ now a joint distribution over pairs $(\ctx,x_{\text{data}})$. This amounts to learning a velocity field $u_\theta(x,t;\ctx)$ minimizing $\mathcal{J}_{\text{cond}}:=\mathbb{E}_{t,\epsilon,(\ctx,x_{\text{data}})}\|u_\theta(\tilde x_t,t;\ctx)-(x_{\text{data}}-\epsilon)\|_2^2$.

\paragraph{Circuit complexity.}
\label{sec:circuitcomplexity}
We measure the expressive power of a model by the complexity of the languages it can decide. A \emph{language} is a subset $\mathcal{L}\subseteq\{0,1\}^*$, and to \emph{decide} \(\mathcal L\) is to determine whether a given input \(w\in\{0,1\}^m\) belongs to \(\mathcal L\). We abstract the model as a \emph{circuit family} $(\mathcal{C}_m)_{m\ge1}$, where each $\mathcal{C}_m$ is a Boolean circuit that handles inputs of length $m$~\citep{arora2009computational}. The \emph{size} of a circuit is the number of gates and its \emph{depth} is the length of the longest input-to-output path. Depth measures sequential computation, whereas size measures total work.

For $k\ge0$, $\mathsf{NC}^k$, $\mathsf{AC}^k$, and $\mathsf{TC}^k$ denote the non-uniform classes of languages decided by polynomial-size, depth-$O(\log^k m)$ circuit families with bounded-fan-in Boolean, unbounded-fan-in Boolean, and unbounded-fan-in Boolean plus MAJORITY gates, respectively. We call circuits of unbounded-fan-in AND, OR, NOT, and MAJORITY gates \emph{majority circuits}. The class $\mathsf{NC}:=\bigcup_{k\ge0}\mathsf{NC}^k$ comprises the efficiently parallelizable problems. The \emph{advice classes} $\mathsf{P/poly}$ and $\mathsf{PSPACE/poly}$ consist of languages decided by polynomial-time and polynomial-space Turing machines, respectively, with polynomial-length advice depending only on the input length. Standard relations among these classes, such as $\mathsf{NC}\subseteq\mathsf{P/poly}\subseteq\mathsf{PSPACE/poly}$, are collected in Appendix~\ref{app:proofconventions}.

These classes bound the power of common neural network architectures: a constant-depth log-precision transformer without chain-of-thought (CoT) lies in $\mathsf{TC}^0$~\citep{merrill2022saturated,merrill2023parallelism,chiang2024transformers}, and variants such as unique hard-attention transformers~\citep{hao2022formal} and constant-precision transformers~\citep{li2024chain} lie in $\mathsf{AC}^0$ instead. Surpassing these bounds requires more layers or, as with CoT, additional sequential calls~\citep{li2024chain}. This motivates studying whether flow matching, which sequentially invokes such a network many times, extends what the model can decide.

\subsection{Theoretical Setup and Notation}
\label{sec:fmsetup}

To formally analyze the expressive power of flow matching, we establish our theoretical setup below, adapting the flow-matching framework into a concrete computational model.

\paragraph{Modeling flow matching.} We consider flow matching under finite-precision arithmetic with $b$ fractional bits. The representable numbers form $F_b:=\{2^{-b}k:k\in\Z\}$. All arithmetic operations are evaluated over the reals and then rounded coordinate-wise via $\round_b(x):=2^{-b}\lfloor 2^b x+1/2\rfloor$.

A \emph{model} $u$ is a map acting on binary encodings of arbitrary length. It receives an encoded triple $(x,t,\ctx)\in F_b^n\times F_b\times F_b^*$, where $n$, $b$, and the length of $\ctx$ are non-negative integers read from the encoding, and outputs the velocity $u(x,t;\ctx)\in F_b^n$. On an input problem $w\in\{0,1\}^m$, vanilla flow matching $\VanillaFM(w;u,K)$ sets $\ctx=w$ and runs the Euler method for $K$ steps, starting from $\tilde{x}_0\sim p_{\text{prior}}$, where $p_{\text{prior}}$ is an arbitrary prior distribution whose support we additionally restrict to $[-1,1]^n$ (Algorithm~\ref{alg:vanillafm}).

\paragraph{Model regularity and acceptance.} A model $u$ is \emph{$B$-bounded} if $\|u(x,t;\ctx)\|_\infty\le B$ for all $x\in F_b^n$, $t\in[0,1]\cap F_b$, and $\ctx\in F_b^*$, and \emph{$L$-Lipschitz continuous} if $\|u(x,t;\ctx)-u(x',t;\ctx')\|_\infty\le L\max\{\|x-x'\|_\infty,\|\ctx-\ctx'\|_\infty\}$ for all $x,x'\in F_b^n$, $t\in[0,1]\cap F_b$, and equal-length $\ctx,\ctx'\in F_b^*$. The Lipschitz condition is joint in $(x,\ctx)$ and uniform over $t$. For simplicity we measure Lipschitzness in the $\ell_\infty$ norm, but the analysis carries over to any $\ell_p$ norm; see \Sref{sec:theoryextensions}.

Given an inference algorithm $\mathcal{A}$ and a constant $c>0$ independent of $m$, we say that \emph{$u$ accepts a language $\mathcal{L}$ with margin $c$ under $\mathcal{A}$} if, for every $m$ and $w\in\{0,1\}^m$, the output $x_K=\mathcal{A}(w;u,K)$ satisfies, with probability $1$ over the randomness in $\mathcal{A}$, $(x_K)_j\ge c$ for all coordinates $j$ when $w\in\mathcal{L}$ and $(x_K)_j\le-c$ for all $j$ when $w\notin\mathcal{L}$. Acceptance with a margin represents a robust decision, in which every output coordinate is separated from the threshold $0$ by at least $c$. For asymptotic analysis, the precision $b$, state dimension $n$, and step count $K$ may grow with the problem size $m$.

\section{Self-Conditioning Increases Expressive Power}
\label{sec:expressivity}

In this section, we first show that, under the mild regularity conditions below, vanilla flow-matching inference is confined to constant-depth circuit classes (\Sref{sec:vanillabarrier}), which restricts the problems it can solve. We then consider self-conditioning (\Sref{sec:scmechanism}), a simple mechanism that allows the same models to attain substantially greater expressive power (\Sref{sec:scpower}).

\subsection{Vanilla Flow Matching Is Confined to Constant Depth}
\label{sec:vanillabarrier}

Although $\VanillaFM$ makes $K$ sequential model evaluations, as long as the model outputs satisfy certain mild regularity conditions, \textbf{this sequential depth is an illusion}:
the same decision can be reached in $O(L)$ sequential rounds of parallel evaluations of the model, regardless of how large $K$ is. When $L$ is a constant, these are constantly many rounds, so the decided language stays within the model's own complexity class. Below we view $K$ as a function of the problem size $m$, and state the theorem under the setup of~\Sref{sec:fmsetup}.

\begin{theorem}[Upper bounds on the expressiveness of vanilla flow matching]\label{thm:vanillaupperbound}
Let $k\ge0$ be a constant integer, $\mathsf{D}\in\{\mathsf{NC}^{k+1},\mathsf{AC}^k,\mathsf{TC}^k\}$ be a circuit complexity class, $u$ be a model that can be computed by circuits in $\mathsf{D}$, and $\mathcal{L}$ be any language that can be accepted by $u$ with margin $1$ under $\VanillaFM$.
Assume that for some constants $B>0$ and $L\ge1$ independent of $m$, the model $u$ is $B$-bounded and $L$-Lipschitz continuous. 
Then as long as the state dimension satisfies $n=\poly(m)$, and the number of fractional bits satisfies
\begin{equation}\label{eq:vanillaprecision}
    \bigl(1+\Omega(1)\bigr)\log_2 K(m) \le b(m) \le O(\log m),
\end{equation}
we have $\mathcal{L}\in\mathsf{D}$.
\end{theorem}

See Appendix~\ref{app:pfvanilla} for the proof. The case $k=0$ recovers the canonical constant-depth classes:  under vanilla flow matching, log-precision softmax-attention transformers in $\mathsf{TC}^0$~\citep{merrill2022saturated,merrill2023parallelism,chiang2024transformers} remain in $\mathsf{TC}^0$, and the unique hard-attention and constant-precision variants in $\mathsf{AC}^0$~\citep{hao2022formal,li2024chain} remain in $\mathsf{AC}^0$. Thus, ODE iteration alone adds no expressive power in this setting.

The core technique behind the proof is a parallel Picard iteration~\citep{hairer1993solving}. The ODE under which $\VanillaFM$ samples can amplify an initial error by a factor of $\mathrm e^L$. A naive simulation that reruns the Euler method with fewer steps to reproduce the full $K$-step trajectory therefore takes $\mathrm e^{O(L)}$ steps to reach constant accuracy. We instead iteratively update all $K$ positions in parallel using the previous iterate, making the propagated error decay at a factorial rate across rounds up to an additive floor, so that $O(L)$ rounds suffice. This yields the following lemma.

\begin{lemma}\label{lem:picard}
    Under the same setup as Theorem~\ref{thm:vanillaupperbound}, assume that $u$ is $B$-bounded for a constant $B>0$ and $L(m)$-Lipschitz continuous with $L(m)\ge1$. Then there exist a constant $c_0=c_0(B)>0$, independent of $m$, and an inference algorithm that runs for $O(L(m))$ iterations, making at most $K(m)$ parallel calls to $u$ per iteration, such that as long as the number of fractional bits satisfies
\begin{equation}\label{eq:picardprecision}
    b(m)\ge \log_2 K(m)+(\log_2\mathrm e)L(m)+c_0\qquad\text{for every }m,
\end{equation}
$u$ accepts $\mathcal{L}$ with margin $1/2$ under that algorithm.
\end{lemma}

See Appendix~\ref{app:pfpicard} for the proof. For constant $L$, these are constantly many rounds, and the arithmetic accompanying the model calls is shallow enough to be absorbed into the model's own circuit class, yielding Theorem~\ref{thm:vanillaupperbound}. The $O(L)$ bound also lets us extend the analysis to a growing $L$~(\Sref{sec:theoryextensions}).

\begin{figure}[t]
\begin{minipage}[t]{0.49\textwidth}
\begin{algorithm}[H]
\caption{Vanilla Flow Matching}
\label{alg:vanillafm}
\begin{algorithmic}[1]
\Function{VanillaFM}{$w;u,K$}
    \State $t_i\leftarrow\round_b(i/K)$ for $i=0,\ldots,K$
    \State Sample $\tilde{x}_0\sim p_{\text{prior}}$; $x_0\leftarrow\round_b(\tilde{x}_0)$
    \State
    \For{$i=0,\ldots,K-1$}
        \State $v_i\leftarrow u(x_i,t_i;w)$
        \State $x_{i+1}\leftarrow x_i+(t_{i+1}-t_i)\cdot v_i$
        \State
    \EndFor
    \State \Return $x_K$
\EndFunction
\end{algorithmic}
\end{algorithm}
\end{minipage}
\hfill
\begin{minipage}[t]{0.49\textwidth}
\begin{algorithm}[H]
\caption{Self-Conditioning Flow Matching}
\label{alg:selfcondfm}
\begin{algorithmic}[1]
\Function{SelfCondFM}{$w;u,K$}
    \State $t_i\leftarrow\round_b(i/K)$ for $i=0,\ldots,K$
    \State Sample $\tilde{x}_0\sim p_{\text{prior}}$; $x_0\leftarrow\round_b(\tilde{x}_0)$
    \State $z_0\leftarrow\mathbf0^n$
    \For{$i=0,\ldots,K-1$}
        \State $v_i\leftarrow u(x_i,t_i;\operatorname{concat}(w,z_i))$
        \State $x_{i+1}\leftarrow x_i+(t_{i+1}-t_i)\cdot v_i$
        \State $z_{i+1}\leftarrow x_i+(1-t_i)\cdot v_i$
    \EndFor
    \State \Return $x_K$
\EndFunction
\end{algorithmic}
\end{algorithm}
\end{minipage}
\end{figure}

\subsection{Self-Conditioning}
\label{sec:scmechanism}

\begin{wrapfigure}{t}{0.465\textwidth}
    \vspace{-0.8em}
    \centering
    \includegraphics[width=0.85\linewidth]{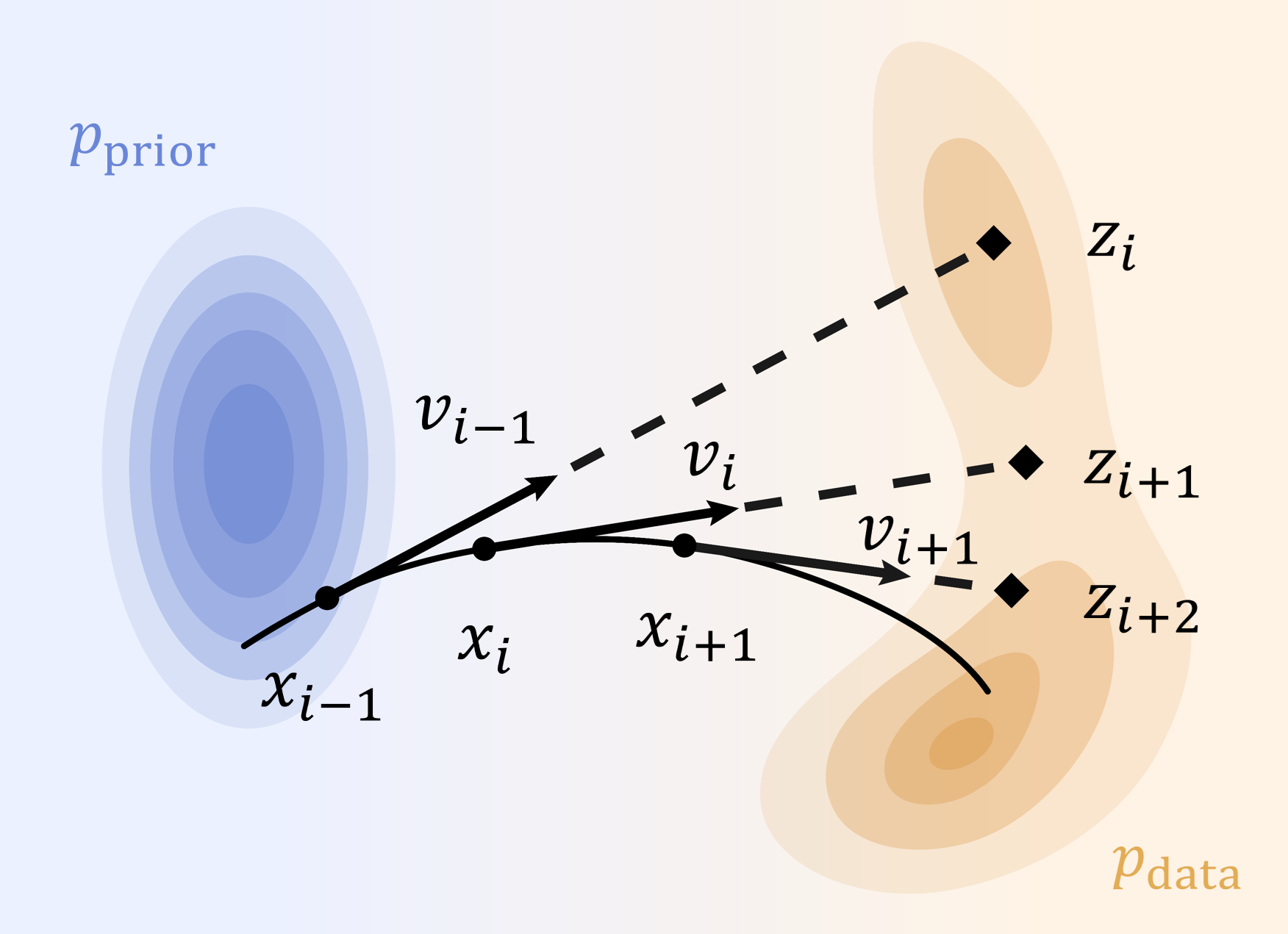}
    \caption{Overview of self-conditioning. Each prediction $z_{i+1}$ is obtained by extrapolating the state $x_i$ along its velocity $v_i$ to $t=1$.}
    \label{fig:overview}
    \vspace{-0.8em}
\end{wrapfigure}

The expressivity bottleneck in Theorem~\ref{thm:vanillaupperbound} stems from how information flows along the continuous ODE: consecutive steps in $\VanillaFM$ transfer information solely through the state $x_i$, which changes only by a small increment scaled by a step size of $O(1/K)$: $x_{i+1}=x_i+(t_{i+1}-t_i)v_i$. This small ODE update is what allows the Picard iteration in Lemma~\ref{lem:picard} to compress all $K$ sampling steps into $O(L)$ iterations. 
One potential way to overcome this bottleneck is to make a larger update per step, for example, by introducing an extra input component that can change abruptly between sampling steps.

We find that a simple mechanism achieves this: \textbf{\emph{self-conditioning}}~\citep{chen2022analog}, originally proposed for discrete diffusion models and later adapted to other settings~\citep{strudel2022self,stark2023harmonic,gupta2024photorealistic}. As stated in Algorithm~\ref{alg:selfcondfm}, self-conditioning feeds the model an additional endpoint prediction $z_i\in F_b^n$, which is a linear approximation of $x_{\text{data}}$, obtained by extrapolating along the previous velocity: $z_i:=x_{i-1}+(1-t_{i-1})v_{i-1}$ for $i\ge 1$, with $z_0=\mathbf0^n$.

For training, we use a two-pass heuristic inspired by prior work~\citep{chen2022analog}. Given $(w,x_{\text{data}})\sim p_{\text{data}}$, $t\sim\mathcal{U}[0,1]$, $\epsilon\sim p_{\text{prior}}$, and $\tilde x_t=tx_{\text{data}}+(1-t)\epsilon$, let
\[
    v^{(1)}:=u_\theta\bigl(\tilde x_t,t;\operatorname{concat}(w,\mathbf{0}^n)\bigr),
    \qquad
    v^{(2)}:=u_\theta\bigl(\tilde x_t,t;\operatorname{concat}(w,\operatorname{sg}[\tilde x_t+(1-t)v^{(1)}])\bigr),
\]
where $\operatorname{sg}[\cdot]$ stops the gradient. The first pass imitates the initial sampling step, at which $z_0=\mathbf0^n$, and the second imitates later steps. Both passes are weighted equally via the following objective:
\[
    \mathcal{J}_{\text{SC}}:=\tfrac{1}{2}\,\mathbb{E}_{t,\epsilon,(w,x_{\text{data}})}\bigl[\|v^{(1)}-(x_{\text{data}}-\epsilon)\|_2^2+\|v^{(2)}-(x_{\text{data}}-\epsilon)\|_2^2\bigr].
\]
Consequently, by training on predictions that closely recover $z_i$, the model learns to leverage the actual $z_i$ generated in the preceding sampling step, requiring only one extra forward pass during training. Figure~\ref{fig:overview} provides an overview of the mechanism.

\paragraph{Self-conditioning in practice.} Self-conditioning is particularly natural for visual-space reasoning: the fed-back prediction $z_i$ approximates the clean sample $x_{\text{data}}$ and is therefore near the target data manifold in image space. Modern image generation models already accept images as conditions~\citep{batifol2025flux,wu2025qwen}, so it is reasonable to expect knowledge from image-conditioned pretraining to transfer to self-conditioning, even though that pretraining involves no self-conditioning. The initialization $z_0=\mathbf0^n$ is a placeholder for cleaner theoretical analysis. In actual implementation, for better efficiency, it is omitted both in the first training pass $v^{(1)}=u_\theta(\tilde x_t,t;w)$ and at the first sampling step, without affecting empirical performance.

\subsection{Self-Conditioning Reaches \texorpdfstring{$\mathsf{PSPACE/poly}$}{PSPACE/poly}}
\label{sec:scpower}

We now formally show that self-conditioning expands the expressiveness of flow matching. Since the extra input $z_{i+1}$ is a fresh prediction rather than an $O(1/K)$ increment on $x_i$, it is unconstrained by the step size and can change abruptly between sampling steps. Thus, the convergence behind Lemma~\ref{lem:picard} is lost, and the constant-depth barrier disappears altogether: \textbf{constant-depth models with self-conditioning can decide $\mathsf{P/poly}$ in polynomially many steps and $\mathsf{PSPACE/poly}$ in at most exponentially many steps}.

\begin{theorem}[Lower bounds on the expressiveness of flow matching with self-conditioning]\label{thm:selfcondlowerbound}
    For any language $\mathcal{L}$ in $\mathsf{PSPACE/poly}$,
    there exists a model $u$ that accepts $\mathcal{L}$ with margin $1$ under $\SelfCondFM$ and satisfies the following conditions: (a) $u$ can be computed by $\mathsf{AC}^0$ circuits; (b) $u$ is $4$-bounded and $8$-Lipschitz continuous; (c) $b=O(\log m)$, $n=\poly(m)$, $K\le2^{\poly(m)}$.
If moreover $\mathcal{L}\in\mathsf{P/poly}$, then the same result holds with~\eqref{eq:vanillaprecision} satisfied.
\end{theorem}

The proof makes the sampler simulate a Turing machine with advice. The prediction $z_i$ holds the machine's tape and state, and each model call advances the machine by one step, so $K$ calls run the machine for $\Omega(K)$ steps. See Appendix~\ref{app:pfselfcond} for the construction. 

In the $\mathsf{P/poly}$ case, all the regularity hypotheses of Theorem~\ref{thm:vanillaupperbound} are retained, so the two theorems concern the same model class and the gap comes from the inference algorithm alone. For $\mathsf{D}\in\{\mathsf{AC}^0,\mathsf{TC}^0\}$, the two samplers therefore \textbf{separate} across the barrier $\mathsf{P/poly}\setminus\mathsf{D}$: given any language $\mathcal{L} \in \mathsf{P/poly}\setminus\mathsf{D}$, Theorem~\ref{thm:selfcondlowerbound} supplies an $\mathsf{AC}^0$ model accepting it under $\SelfCondFM$, whereas Theorem~\ref{thm:vanillaupperbound} leaves no $\mathsf{D}$ model satisfying its hypotheses capable of accepting it under $\VanillaFM$. For $\mathsf{D}=\mathsf{AC}^0$, the separation always holds since $\mathsf{AC}^0 \subsetneq \mathsf{TC}^0$
~\citep{furst1984parity}; for $\mathsf{D}=\mathsf{TC}^0$, it holds under the conjecture $\mathsf{TC}^0\subsetneq\mathsf{NC}^1$~\citep{vollmer1999introduction}, which implies that $\mathsf{TC}^0 \subsetneq \mathsf{P/poly}$.

\subsection{Extensions}
\label{sec:theoryextensions}

Beyond the main results above, the analysis extends to more general settings. Lemma~\ref{lem:picard} directly implies an $\mathsf{NC}$ upper bound for $\VanillaFM$ when the Lipschitz constant $L$ grows as $L(m)=\poly(\log m)$, or $\mathsf{TC}^k$ upper bounds when further $L=O(\log^k m)$ for any constant integer $k\ge0$ (Appendix~\ref{app:polylogl}). Restricting $u$ to be Lipschitz continuous under $\ell_p$ norms rather than $\ell_\infty$ norms preserves both Theorem~\ref{thm:vanillaupperbound} and Theorem~\ref{thm:selfcondlowerbound} (Appendix~\ref{app:lplipschitz}). Replacing the Euler method with any fixed explicit Runge--Kutta method~\citep{hairer1993solving} whose nodes lie in $[0,1]$, such as Heun's method~\citep{karras2022elucidating}, leaves the asymptotic results for $\VanillaFM$ unchanged (Appendix~\ref{app:rk}).

On the other hand, the $O(L(m))$ bound of Lemma~\ref{lem:picard} is tight with respect to circuit depth: any polynomial-size majority circuit of depth $d$ can be simulated under $\VanillaFM$ in $O(d)$ sequential rounds by an $O(d)$-Lipschitz continuous model $u$ computed by $\mathsf{TC}^0$ circuits. The bound of Lemma~\ref{lem:picard} therefore cannot be improved to $o(L(m))$ rounds of the same form unless the depth hierarchy for polynomial-size majority circuits collapses (Appendix~\ref{app:depthmatch}).

\section{Experiments}
\label{sec:experiments}

The analysis in \Sref{sec:expressivity} establishes an expressivity separation under the stated assumptions: \(\SelfCondFM\) can decide languages beyond the reach of \(\VanillaFM\). Whether models trained on finite data realize this expressive advantage remains an empirical question. Therefore, we conduct controlled experiments to empirically validate the expressivity separation
(\Sref{sec:theory_aligned_tasks}). We further extend the evaluation to the five structured visual-reasoning tasks
used by \citet{he2025diffthinker} through controlled comparisons and larger-scale experiments (\Sref{sec:diffthinker_tasks}). Finally, we provide supporting analyses that assess the applicability of our theory to trained models and characterize the inference behavior of \(\VanillaFM\) and \(\SelfCondFM\) (\Sref{sec:theory_aligned_analysis}).

\subsection{Experimental Setup}
\label{sec:exp_protocol}

Within each task, \(\VanillaFM\) and \(\SelfCondFM\) share the same backbone architecture, training dataset, optimizer, and noise schedule. The only difference is that \(\SelfCondFM\) passes the endpoint prediction from one backbone model evaluation to the next as an additional condition (\Sref{sec:scmechanism}). During training, \(\SelfCondFM\) makes two forward passes per example instead of one, costing extra compute, so we match the two methods on cumulative training FLOPs. Appendix~\ref{app:architecture} provides the complete configurations and FLOPs accounting.

Since both \(\VanillaFM\) and \(\SelfCondFM\) make one backbone call per sampling step, we compare them by default at the same number of function evaluations (NFE), which directly measures sequential computation and equals the step count \(K\) in Algorithms~\ref{alg:vanillafm} and~\ref{alg:selfcondfm}. Unless otherwise stated, we use 50 NFEs; the NFE-scaling study in \Sref{sec:theory_aligned_tasks} varies the number of evaluations.

\subsection{Empirical Validation of the Expressivity Separation}
\label{sec:theory_aligned_tasks}

To empirically validate the expressivity separation, we first compare the two methods on three tasks with natural step-by-step solutions: \emph{line tracing}, \emph{long multiplication}, and \emph{long text rendering} (Figure~\ref{fig:theory_aligned_tasks}). Each task, formally defined and analyzed in
Appendix~\ref{app:task_complexity}, has circuit complexity matching the
separation barrier identified in \Sref{sec:vanillabarrier} for either
\(\mathsf{AC}^0\) or \(\mathsf{TC}^0\) backbones (the latter case assumes
\(\mathsf{TC}^0\neq\mathsf{NC}^1\)). Consequently, we expect these tasks to lie beyond the expressive power of vanilla flow matching with the corresponding backbone.

\paragraph{Tasks.} We denote the problem scale for each task by \(m\). Line tracing (\(\mathsf{NC}^1\)-complete) requires coloring five long polylines that go through $m$ permutation stages; long multiplication (\(\mathsf{TC}^0\)-complete) entails completing a step-by-step long multiplication worksheet to compute the product of two integers, each with at most $m$ digits; and long text rendering (\(\mathsf{TC}^0\)-complete) requires rendering some paragraphs containing a total of $m$ words onto a fixed-width canvas. We compare the methods on finite-scale synthetic datasets. For line tracing, we report the mean longest correct prefix (LCP), which is the average number of initial stages through which all five polylines remain correctly colored; for the other tasks, we report instance-level accuracy.
See Appendix~\ref{app:tasks_setup} for more details.

\paragraph{Backbone architectures.} All three tasks use \texttt{FLUX.2}-based architectures with their pretrained VAE kept frozen for image encoding and decoding~\citep{blackforestlabs2026flux2klein}. For line tracing, we train a model from scratch; for the other two tasks, we fine-tune \texttt{FLUX.2-klein-base-4B} to leverage its prior knowledge of character recognition and generation.

\paragraph{Results.}
Figure~\ref{fig:theory_aligned_tasks} plots performance against training FLOPs. On line tracing, both methods are evaluated at 100 NFEs. \(\SelfCondFM\) reaches the maximum LCP of 50 at \(1.10\times10^{18}\) training FLOPs, while \(\VanillaFM\) plateaus near 8. On long multiplication, \(\SelfCondFM\) exceeds 90\% accuracy within half the budget, a level \(\VanillaFM\) never reaches over the evaluated compute range. On long text rendering, \(\SelfCondFM\)'s accuracy rises with training compute, while \(\VanillaFM\)'s stays near zero for most of training. Self-conditioning thus \textbf{converges faster and performs better on all three tasks}, consistent with our theoretical separation. See Appendix~\ref{app:multi-seed-diagnostic} for multi-seed evaluation and Appendix~\ref{app:diagnostic_examples} for generated examples.

\paragraph{NFE scaling.} We further examine how the two methods use additional inference steps. Our theory in~\Sref{sec:expressivity} suggests
a contrasting behavior: under its assumptions, \(\VanillaFM\) is confined to
constant reasoning depth, whereas \(\SelfCondFM\) can realize sequential
reasoning depth that grows with the number of sampling steps. We test whether
this predicted difference emerges on the line tracing task by varying the NFE.

\begin{wrapfigure}{r}{0.48\textwidth}
    \centering
    \includegraphics[width=0.9\linewidth]{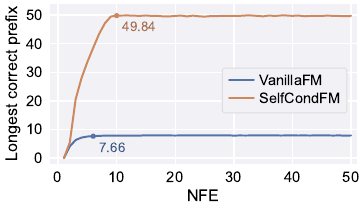}
    \caption{\(\SelfCondFM\)'s performance scales strongly with NFE, while that of \(\VanillaFM\) saturates early.}
    \label{fig:nfe_scaling}
    \vspace{-3em}
\end{wrapfigure}

We select two checkpoints with matched training FLOPs, and evaluate both at different NFEs.
Figure~\ref{fig:nfe_scaling} shows that both methods begin with a mean LCP of 0 at \(\text{NFE}=1\), but diverge as NFE increases. \(\SelfCondFM\) reaches 49.84 at \(\text{NFE}=10\), whereas \(\VanillaFM\) reaches 7.66 at \(\text{NFE}=6\) and remains below 8.0 through \(\text{NFE}=50\).
The extra steps thus yield deeper reasoning for \(\SelfCondFM\), confirming that self-conditioning lets the model make better use of more sampling steps. Full details and inference-FLOP-aligned comparisons are provided in Appendix~\ref{app:nfe_scaling}.

\subsection{Evaluation on Structured Visual-Reasoning Tasks}
\label{sec:diffthinker_tasks}

\begin{table*}[t]
    \centering
    \renewcommand{\arraystretch}{1.02}
    \caption{Controlled experiments on the DiffThinker tasks (accuracy, \%). Under matched training FLOPs and NFE, \(\SelfCondFM\) outperforms \(\VanillaFM\) on every task.}
    \label{tab:diffthinker}
    \footnotesize
    \begin{tabular*}{\textwidth}{@{\extracolsep{\fill}}l*{6}{r@{\hspace{1.5pt}}l}@{}}
        \toprule
        \textbf{Method}
        & \multicolumn{2}{l}{\textbf{VSP-Super}}
        & \multicolumn{2}{l}{\textbf{Maze}}
        & \multicolumn{2}{l}{\textbf{TSP}}
        & \multicolumn{2}{l}{\textbf{Sudoku}}
        & \multicolumn{2}{l}{\textbf{Jigsaw}}
        & \multicolumn{2}{l}{\textbf{Avg.}} \\
        \midrule
        \(\VanillaFM\)
        & 89.0 & & 64.3 & & 51.0 & & 60.5 & & 64.7 & & 65.9 & \\
        \(\SelfCondFM\)
        & \textbf{92.0} & {\scriptsize(+3.0)}
        & \textbf{70.7} & {\scriptsize(+6.4)}
        & \textbf{66.7} & {\scriptsize(+15.7)}
        & \textbf{80.0} & {\scriptsize(+19.5)}
        & \textbf{74.3} & {\scriptsize(+9.6)}
        & \textbf{76.7} & {\scriptsize(+10.8)} \\
        \bottomrule
    \end{tabular*}
\end{table*}

\begin{table*}[t]
    \centering
    \setlength{\tabcolsep}{3pt}
    \renewcommand{\arraystretch}{1.06}
    \caption{Large-scale comparison on the DiffThinker tasks (accuracy, \%). Each column header is a difficulty level, following the evaluation of \citet{dai2026endocot}, from which the baseline numbers are taken. The best results are in bold and the second best are underlined. For unified training, each method uses a single adapter trained jointly on data from all four tasks.}
    \label{tab:diffthinker_large}
    \footnotesize
    \begin{tabularx}{\textwidth}{l*{13}{>{\centering\arraybackslash}X}}
        \toprule
        \multirow{2}{*}{\textbf{Method}}
        & \multicolumn{2}{c}{\textbf{VSP-Super}} & \multicolumn{3}{c}{\textbf{Maze}}
        & \multicolumn{3}{c}{\textbf{TSP}} & \multicolumn{4}{c}{\textbf{Sudoku}}
        & \multirow{2}{*}{\textbf{Avg.}} \\
        \cmidrule{2-3} \cmidrule(lr){4-6} \cmidrule(lr){7-9} \cmidrule(lr){10-13}
        {} & \multicolumn{1}{c}{16} & \multicolumn{1}{c}{32}
        & \multicolumn{1}{c}{8} & \multicolumn{1}{c}{16} & \multicolumn{1}{c}{32}
        & \multicolumn{1}{c}{12} & \multicolumn{1}{c}{15} & \multicolumn{1}{c}{18}
        & \multicolumn{1}{c}{45} & \multicolumn{1}{c}{40} & \multicolumn{1}{c}{35} & \multicolumn{1}{c}{30} & {} \\
        \midrule
        \rowcolor{gray!12}
        \multicolumn{14}{l}{\emph{Task-specific training}} \\
        Qwen3-VL-8B (SFT)
        & 61 & 8 & \tightuline{53} & \tightuline{37} & 0 & 59 & 60 & 43
        & 30 & 17 & 2 & 0 & 30.8 \\
        Qwen3-VL-8B (GRPO)
        & 0 & 0 & 0 & 0 & 0 & 0 & 0 & 0
        & 0 & 0 & 0 & 0 & 0.0 \\
        DiffThinker
        & \tightuline{99} & 80 & \textbf{100} & \textbf{100} & 65 & 76 & 72 & 59
        & \tightuline{97} & \tightuline{94} & 55 & 13 & 75.8 \\
        EndoCoT
        & \tightuline{99} & \tightuline{85} & \textbf{100} & \textbf{100} & \tightuline{90}
        & \tightuline{77} & \textbf{77} & \textbf{73}
        & \textbf{100} & \textbf{100} & \tightuline{95} & \tightuline{64} & \tightuline{88.3} \\
        \textbf{Ours}
        & \textbf{100} & \textbf{86} & \textbf{100} & \textbf{100} & \textbf{95}
        & \textbf{79} & \tightuline{74} & \tightuline{72}
        & \textbf{100} & \textbf{100} & \textbf{98} & \textbf{76} & \textbf{90.0} \\
        \midrule
        \rowcolor{gray!12}
        \multicolumn{14}{l}{\emph{Unified training}} \\
        DiffThinker
        & \tightuline{97} & \tightuline{84} & \tightuline{98} & \tightuline{99} & \tightuline{66}
        & \tightuline{64} & 49 & 34 & \tightuline{94} & 45 & 26 & 37 & 66.1 \\
        EndoCoT
        & \textbf{100} & 80 & 97 & 98 & 52
        & \tightuline{64} & \tightuline{55} & \tightuline{46}
        & \textbf{100} & \tightuline{88} & \tightuline{80} & \tightuline{59} & \tightuline{76.6} \\
        \textbf{Ours}
        & \textbf{100} & \textbf{86} & \textbf{100} & \textbf{100} & \textbf{96}
        & \textbf{78} & \textbf{75} & \textbf{65}
        & \textbf{100} & \textbf{100} & \textbf{97} & \textbf{85} & \textbf{90.2} \\
        \bottomrule
    \end{tabularx}
\end{table*}

We next investigate whether the gains observed in \Sref{sec:theory_aligned_tasks} extend to the five structured visual-reasoning tasks of \citet{he2025diffthinker}: large-scale visual spatial planning (VSP-Super), Maze, the traveling salesperson problem (TSP), Sudoku, and Jigsaw.
See Appendix~\ref{app:diffthinker_setup} for detailed setup and illustrations.

\paragraph{Controlled experiments.} Following the setup
in \Sref{sec:exp_protocol}, we train \(\VanillaFM\) and \(\SelfCondFM\) as \texttt{FLUX.2-klein-base-4B} LoRA adapters on the same data for each task, and evaluate them at matched training FLOPs and NFE.

As reported in \Tref{tab:diffthinker}, \(\SelfCondFM\) improves over \(\VanillaFM\) on every task, raising the average accuracy from \(65.9\%\) to \(76.7\%\). Larger gains appear on
TSP (+15.7\%) and Sudoku (+19.5\%), while notable improvements are also observed on the other tasks. Thus self-conditioning extends its advantage from our validation tasks to the task suite of \citet{he2025diffthinker}. Appendix~\ref{app:structured_examples} shows some generated examples.

\paragraph{Large-scale comparisons.} To evaluate the method at a larger scale, we fine-tune \(\SelfCondFM\) as a LoRA adapter on the 20B-parameter \texttt{Qwen-Image-Edit-2511} backbone~\citep{qwenteam2025qwenimageedit2511} 
following the setup of \citet{dai2026endocot}, and compare it with the published baselines. Jigsaw is reported separately in Appendix~\ref{app:jigsaw_benchmark}, because EndoCoT~\citep{dai2026endocot} does not evaluate it.

As shown in \Tref{tab:diffthinker_large}, task-specific \(\SelfCondFM\) adapters achieve \(90.0\%\) average accuracy, outperforming the purely visual-generative DiffThinker baseline (\(75.8\%\)). EndoCoT (\(88.3\%\)) additionally guides image generation using latent CoT in a multimodal language model, so we include it as a reference rather than a directly comparable baseline. A single adapter trained jointly on the four tasks reaches \(90.2\%\), matching the per-task adapters and exceeding the unified DiffThinker (\(66.1\%\)) and EndoCoT (\(76.6\%\)) results by large margins.

\subsection{Additional Empirical Analyses}
\label{sec:theory_aligned_analysis}

Finally, we provide some empirical analyses to illustrate how the trained models behave during inference (detailed in Appendix~\ref{app:secondary_analysis}):
\begin{itemize}[leftmargin=*,nosep]
    \item \textbf{Lipschitzness.} We estimate the Lipschitz constant of a \(\VanillaFM\) model trained on mixed-length line tracing instances. The estimate levels off at a roughly constant value of 400 for \(m\ge30\), consistent with the bounded-Lipschitz regime assumed in \Sref{sec:vanillabarrier} (Appendix~\ref{app:lipschitzness}).
    \item \textbf{Attention allocation.} The endpoint prediction \(z_i\) in \(\SelfCondFM\) receives a nontrivial share of attention throughout the generation, with a larger share early on when \(x_i\) is still noisy. This indicates that the model indeed uses $z_i$ to guide its generation (Appendix~\ref{app:attention_allocation}).
    \item \textbf{Parallel and revisable reasoning.} Reasoning with flow matching has some intrinsic advantages, and \(\SelfCondFM\) preserves them well. The model can form multiple local results in parallel within one trajectory and correct a wrong mid-trajectory generation in place 
    (Appendix~\ref{app:trajectory_analysis}).
\end{itemize}

\section{Related Work}

\paragraph{Generative visual reasoning.}
Multimodal large language models (MLLMs) typically treat images as static inputs and externalize reasoning through textual CoT.
The emerging \emph{thinking with images} paradigm instead turns vision into an intermediate workspace that can be repeatedly inspected or modified~\citep{su2025thinking}. Early approaches manipulate images through tool calls~\citep{su2025openthinkimg,su2026pixel,lai2026mini}, while other approaches internalize visual intermediates within the model, in the form of generated images, latent visual states, or learned visual representations~\citep{xu2026visual,yang2026machine,zhang2025latent,qin2025chain,wang2026monet,gu2026thinkmorph,li2026latentvisual,cheng2026omnir1}. These intermediates give the model a richer workspace for reasoning than text alone.

Video generation offers another medium for reasoning, representing evolving intermediate states as a sequence of frames~\citep{wiedemer2025video,tong2026thinking}. Recent works further suggest that video diffusion models~\citep{ho2022video,xing2024survey} reason across denoising steps: they explore candidate solutions, correct earlier mistakes, and commit early to high-level plans~\citep{wang2026demystifying,newman2026early}.

Recent work moves reasoning into image generation itself: DiffThinker uses direct image-to-image generation~\citep{he2025diffthinker}, while EndoCoT
bridges latent states in an MLLM's CoT to image generation~\citep{dai2026endocot}. However, the computational structure of generative visual reasoning remains underexplored.
\citet{liu2025perfect} establishes a $\mathsf{TC}^0$ upper bound for optimal
denoising diffusion probabilistic models, but requires the model to match the score function exactly, an assumption that rarely holds for practical trained models.

\paragraph{Self-conditioning.}
Self-conditioning feeds a model's endpoint prediction back into a
subsequent denoising step. It was introduced by \citet{chen2022analog} for diffusion over discrete data,
and has since been adapted to continuous diffusion language modeling, image and video generation, and molecular flow
matching~\citep{strudel2022self,jabri2023scalable,gupta2024photorealistic,stark2023harmonic}.
Prior work primarily studies generation quality and efficiency empirically; we instead treat self-conditioning as maintaining a computational state and analyze its expressive power.

\paragraph{Computational Expressiveness.}
The expressive power of transformers~\citep{vaswani2017attention} has been studied extensively through circuit complexity, formal languages, and logic~\citep{strobl2024formal}. Depending on the attention mechanism and numerical precision, constant-depth transformers are upper-bounded by \(\mathsf{AC}^0\) or \(\mathsf{TC}^0\)~\citep{hao2022formal,merrill2022saturated,merrill2023parallelism,chiang2024transformers}, and such circuit bounds extend to modern components such as rotary positional embeddings~\citep{chen2025circuit}. \citet{chiang2023tighter} further tighten both the upper and lower bounds through a logical characterization.

Repeated model calls can lift computation beyond these single-evaluation bounds. With chain-of-thought~\citep{wei2022chain}, a constant-depth transformer writes its intermediate results into the autoregressive context and can thereby solve inherently serial problems~\citep{li2024chain,feng2023towards,merrill2024expressive}. \citet{amiri2025lower} complement this with lower bounds on the number of chain-of-thought steps that hard-attention transformers need for various algorithmic problems. Beyond autoregressive generation, \citet{feng2025theoretical} establish theoretical bounds for masked diffusion language models, proving that while constant sampling steps suffice for distribution matching, sequential reasoning requires sampling steps that scale with sequence length.

\section{Conclusion and Discussion}
\label{sec:conclusion}

This work reveals a previously unrecognized role of self-conditioning in generative reasoning.
We show that vanilla flow matching cannot deepen sequential reasoning through additional sampling steps. Self-conditioning breaks this bottleneck by passing endpoint predictions across steps, which lets additional sampling steps translate into reasoning depth. Experiments across diverse visual-reasoning tasks confirm this separation: self-conditioning achieves better performance and consistently benefits from more sampling steps.

These theoretical guarantees, however, rely on certain regularity assumptions. In particular, the upper bound in Theorem~\ref{thm:vanillaupperbound} requires a fairly strong Lipschitz restriction. Moreover, the Lipschitz estimates in \Sref{sec:theory_aligned_analysis} are relatively large and provide empirical rather than certified bounds (Appendix~\ref{app:lipschitzness}). Since the dependence on the Lipschitz constant is tight up to constant factors (\Sref{sec:theoryextensions}), future work may explore alternative forms of regularity conditions.

Although we focus on image-space visual reasoning, our theoretical analysis does not depend on the data modality. It therefore applies wherever flow matching serves as the sampling backbone, including video generation~\citep{kong2024hunyuanvideo,ifriqi2026flowception}, continuous diffusion language modeling~\citep{chen2026langflow}, and visuomotor policy learning~\citep{black2024pi}. Whether self-conditioning delivers similar empirical gains
remains an open question. Nevertheless, this perspective may inform future sampler design in settings where iterative generation is expected to reason.

\section*{AI Disclosure}

For this work, we used generative AI tools to (i) generate synthetic datasets for the tasks in \Sref{sec:theory_aligned_tasks}; (ii) provide insights into reducing the number of sequential rounds from $\mathrm e^{O(L)}$ to $O(L)$ by replacing the Euler method with Picard iterations in \Sref{sec:vanillabarrier}; (iii) assist in writing the proofs; and (iv) implement the training codebase. We did not use generative AI tools to develop theoretical models or conceptual frameworks, formulate mathematical claims, design or provide feedback on research methodology or experiments, support qualitative or thematic data analysis, or interpret results. The tasks of proposing or refining hypotheses, translating text, and cleaning or reformatting datasets were not applicable to this work. Additionally, we used generative AI tools to create and refine figures, to assist in drafting and editing the paper at the paragraph level to improve readability, to identify relevant literature, and to format references.

We have reviewed all AI-assisted work. All AI-assisted proof writing was reviewed by at least two authors other than the original writers. All AI-assisted language editing was reviewed by at least one author other than those who originally drafted the corresponding text. The core experimental codebase and synthetic datasets were reviewed by at least one author. All figures, relevant literature, and the final reference formatting were verified by at least one author. We take full responsibility for the final content of this work, including all text, claims, and artifacts produced with the aid of generative AI.

\section*{Reproducibility Statement}

For the theoretical results, \Sref{sec:fmsetup} specifies the theoretical setup, and \Sref{sec:expressivity} states the assumptions of each theorem. Appendices~\ref{app:proofs} and~\ref{app:extensions} give complete proofs of the main theorems and their extensions, and Appendix~\ref{app:task_complexity} formally defines the three synthetic tasks in~\Sref{sec:theory_aligned_tasks} and proves their circuit-complexity classifications. For the experiments, Appendix~\ref{app:architecture} describes the backbones, the implementation of self-conditioning, and the training-FLOP accounting. Appendix~\ref{app:tasks_setup} details the data generation, training hyperparameters, and evaluation protocol for each synthetic task, Appendix~\ref{app:multi-seed-diagnostic} reports results over multiple seeds, and Appendix~\ref{app:nfe_scaling} elaborates the selection of checkpoints and the evaluation for our NFE-scaling study. Appendix~\ref{app:diffthinker_setup} describes the data preparation, evaluators, and training configurations for the structured visual-reasoning tasks in~\Sref{sec:diffthinker_tasks}, which follow \citet{dai2026endocot} and \citet{he2025diffthinker}.

\bibliography{input/references}
\bibliographystyle{iclr2027_conference}

\clearpage
\appendix
\tableofcontents
\clearpage
\section{Proofs of Main Expressivity Theorems}
\label{app:proofs}

This appendix proves the two main theoretical results from~\Sref{sec:expressivity}: Theorem~\ref{thm:vanillaupperbound} and Theorem~\ref{thm:selfcondlowerbound}. Building on the theoretical setup in~\Sref{sec:fmsetup}, Appendix~\ref{app:proofconventions} first fixes the notation and the encoding in which circuits read and write elements of $F_b$, and collects the preliminary facts. Appendix~\ref{app:vanillaub} then proves Theorem~\ref{thm:vanillaupperbound} with the help of Lemma~\ref{lem:picard}, and Appendix~\ref{app:scondlb} proves Theorem~\ref{thm:selfcondlowerbound}. The extension results of~\Sref{sec:theoryextensions} are proved separately in Appendix~\ref{app:extensions}.

\subsection{Conventions and Preliminary Facts}
\label{app:proofconventions}

Throughout this appendix, $m$ denotes the problem size and $\mathsf{D}\in\{\mathsf{NC}^{k+1},\mathsf{AC}^k,\mathsf{TC}^k\}$ denotes a circuit complexity class for a constant integer $k\ge0$; we abbreviate $b=b(m)$, $n=n(m)$, $K=K(m)$, and $\delta:=2^{-b}$. All circuit families are non-uniform.

\paragraph{Relations among complexity classes.} It is known that $\mathsf{NC}^k\subseteq\mathsf{AC}^k\subseteq\mathsf{TC}^k\subseteq\mathsf{NC}^{k+1}$ for all $k\ge0$~\citep{arora2009computational}. Consequently, $\mathsf{NC}=\bigcup_k\mathsf{AC}^k=\bigcup_k\mathsf{TC}^k$. Furthermore, $\mathsf{AC}^0\subsetneq\mathsf{TC}^0$ because $\mathrm{PARITY}\in\mathsf{TC}^0\setminus\mathsf{AC}^0$~\citep{furst1984parity,hastad1986almost}, whereas $\mathsf{TC}^0\subsetneq\mathsf{NC}^1$ is widely conjectured but open~\citep{vollmer1999introduction}. By definition, $\mathsf{P}\subseteq\mathsf{P/poly}$ and $\mathsf{PSPACE}\subseteq\mathsf{PSPACE/poly}$. Moreover, the class $\mathsf{P/poly}$ is exactly the class of languages decidable by polynomial-size families of Boolean circuits~\citep{arora2009computational}, so $\mathsf{NC}\subseteq\mathsf{P/poly}\subseteq\mathsf{PSPACE/poly}$. Both inclusions are widely conjectured to be strict~\citep{arora2009computational}.

\paragraph{Encoding of scalars.} Fix a binary encoding of the scalars of $F_b$ in which a scalar of magnitude at most $2^\ell$ occupies $O(b+\ell)$ bits. Every scalar in the runs of this appendix has magnitude at most $\poly(m)$, i.e., $\ell=O(\log m)$, and all circuit constructions have $b=\poly(m)$. Their encodings thus have $\poly(m)$ bits. When $b$ is additionally constrained to $O(\log m)$, the encodings have $O(\log m)$ bits.

\paragraph{Circuits for models.} We say that a model $u$ is computed by $\mathsf{D}$ circuits if, under the encoding above, its restriction to inputs of each total encoded length is computed by a $\mathsf{D}$ circuit of size polynomial in that length. The Boolean input $w$ needs no convention of its own, since $\{0,1\}\subset F_b$.

\paragraph{Finite-precision arithmetic.} The grid $F_b=\{k2^{-b}:k\in\Z\}$ is closed under addition, under subtraction, and under multiplication by an integer, so these operations need no rounding. Iterated sums are independent of the association order. For every $x\in \R$ we have $|\round_b(x)-x|\le\delta/2$, and $\round_b$ is nondecreasing and fixes $F_b$ (i.e., $\round_b(x)=x$ for all $x\in F_b$).

\paragraph{Time grid.} Both samplers (Algorithms~\ref{alg:vanillafm} and~\ref{alg:selfcondfm}) use discrete time steps $t_i=\round_b(s_i)$ with $s_i:=i/K$; write $h_i:=t_{i+1}-t_i$. Since $\round_b$ is nondecreasing and fixes $0$ and $1$, we have
\[
    t_0=0,\qquad t_K=1,\qquad h_i\ge0,\qquad \sum_{i=0}^{K-1}h_i=1,
\]
and $|t_i-s_i|\le\delta/2$ gives $|h_i-1/K|\le\delta$.

\paragraph{Hardwiring the initialization.} In both samplers the initial state $x_0=\round_b(\tilde x_0)$ ranges over the finite set $(F_b\cap[-1,1])^n$, since $p_{\text{prior}}$ is supported on $[-1,1]^n$, and the rest of the run is deterministic given $x_0$. If a model accepts $\mathcal{L}$ with margin $c$ under a fixed sampler, the acceptance criterion holds for every $x_0$ occurring with positive probability. Since $p_{\text{prior}}$ does not depend on $w$, fixing one such value $\hat x_0=\hat x_0(m)$ for each $m$ turns the sampler into a deterministic algorithm with the same guarantee. We do this whenever a sampler is regarded as a circuit family.

\paragraph{The table bound.} A function $f$ of few bits is shallow, however complicated it is: writing each output bit of an arbitrary $f:\{0,1\}^j\to\{0,1\}^{j'}$ as a disjunctive normal form over the $j$ input literals gives an unbounded-fan-in circuit of depth two and size $O(j'2^j)$, hence of size $\poly(m)$ once $j=O(\log m)$ and $j'=\poly(m)$. We call this the \emph{table bound}; it exploits no structure of $f$, only how many bits $f$ reads and writes. In the proofs below, we frequently invoke it for maps on constantly many scalars, each of $O(\log m)$ bits.

\subsection{Expressiveness Upper Bound for \texorpdfstring{$\VanillaFM$}{VanillaFM}}
\label{app:vanillaub}

The $K$ Euler steps of $\VanillaFM$ are sequential, so a circuit that follows them has depth growing with $K$. We replace them by a Picard iteration~\citep{hairer1993solving} that updates all $K$ positions at once, and whose propagated error decays factorially across iteration rounds, up to an additive error floor (Appendix~\ref{app:recurrence}). Therefore, $O(L)$ rounds already approximate the sampler's output to within a constant error; this proves Lemma~\ref{lem:picard} in Appendix~\ref{app:pfpicard}. Aside from the $K$ parallel model calls, each round computes normalized prefix sums of the velocities, which Appendix~\ref{app:prefixsums} approximates in $\mathsf{AC}^0$ up to a constant error, small enough for acceptance with margin (see \Sref{sec:fmsetup}). Appendix~\ref{app:pfvanilla} substitutes these approximate sums into the Picard iteration to prove Theorem~\ref{thm:vanillaupperbound}.

\subsubsection{The Error Recurrence}
\label{app:recurrence}

Appendices~\ref{app:pfpicard} and~\ref{app:pfvanilla} both compare an approximate trajectory with the exact one through the same error recurrence, so we bound it here once and for all. The error accumulates over multiple steps. At each step, the error at each position is bounded by a fixed linear transformation of the sum of the errors at all preceding positions in the previous step, plus an additive error. Consequently, the initial error decays factorially across steps, while the additive errors accumulate to a magnitude of order $\mathrm{e}^L$.

\begin{lemma}\label{lem:picarderror}
Let $K\in\Np$ be a positive integer, $B>0$, $L\ge1$, and $\alpha,\beta\ge0$ be real numbers, and let $C:=\lceil\mathrm eL+4B\rceil$. Recall that $s_i:=i/K$ (Appendix~\ref{app:proofconventions}). Let $(e_i^{(r)})_{0\le i\le K,\,0\le r\le C}$ be nonnegative reals such that, for all $i$ and all $r<C$,
\begin{align}
    e_i^{(0)}&\le Bs_i+\alpha,\label{eq:picardbase}\\
    e_i^{(r+1)}&\le\frac{L}{K}\sum_{j=0}^{i-1}e_j^{(r)}+\beta.\label{eq:picardrec}
\end{align}
Then for all $0\le i\le K$ and $0\le r\le C$,
\begin{equation}\label{eq:picardinduction}
    e_i^{(r)}\le\frac{BL^rs_i^{r+1}}{(r+1)!}+\frac{\alpha L^rs_i^r}{r!}+\beta\sum_{q=0}^{r-1}\frac{L^qs_i^q}{q!},
\end{equation}
where $0^0:=1$ and the sum is read as empty when $r=0$. In particular, $e_K^{(C)}\le\frac14+\alpha+\beta\mathrm e^L$.
\end{lemma}

\begin{proof}
We prove~\eqref{eq:picardinduction} by induction on $r$. For $r=0$ the right-hand side is $Bs_i+\alpha$, so~\eqref{eq:picardinduction} is~\eqref{eq:picardbase}. Assume that~\eqref{eq:picardinduction} holds for $r$ and all $i$. Since the function $\vartheta\mapsto\vartheta^q$ is nondecreasing on $[0,\infty)$ and $s_{j+1}-s_j=1/K$, every integer $q\ge0$ satisfies the Riemann bound
\begin{equation}\label{eq:riemann}
    \frac1K\sum_{j=0}^{i-1}s_j^q=\frac{1}{K} \sum_{j=0}^{i-1} (j/K)^q \le\int_0^{s_i}\vartheta^q\,\mathrm d\vartheta=\frac{s_i^{q+1}}{q+1}.
\end{equation}
Summing~\eqref{eq:picardinduction} over $j<i$ and applying~\eqref{eq:riemann} termwise gives
\begin{align*}
    \frac{L}{K}\sum_{j=0}^{i-1}e_j^{(r)}
    &\le\frac{1}{K}\sum_{j=0}^{i-1} \left(\frac{BL^{r+1}s_j^{r+1}}{(r+1)!}+\frac{\alpha L^{r+1}s_j^r}{r!}+\beta\sum_{q=0}^{r-1}\frac{L^{q+1}s_j^q}{q!}\right)\\
    &\le\frac{BL^{r+1}s_i^{r+2}}{(r+1)!\,(r+2)}+\frac{\alpha L^{r+1}s_i^{r+1}}{r!\,(r+1)}+\beta\sum_{q=0}^{r-1}\frac{L^{q+1}s_i^{q+1}}{q!\,(q+1)}\\
    &=\frac{BL^{r+1}s_i^{r+2}}{(r+2)!}+\frac{\alpha L^{r+1}s_i^{r+1}}{(r+1)!}+\beta\sum_{q=1}^{r}\frac{L^{q}s_i^{q}}{q!}.
\end{align*}
By~\eqref{eq:picardrec}, adding $\beta=\beta L^0s_i^0/0!$ to this bound yields~\eqref{eq:picardinduction} for $r+1$.

For the last claim, put $i=K$ and $r=C$ in~\eqref{eq:picardinduction}, so that $s_K=1$. Stirling's inequality $C!\ge(C/\mathrm e)^C$ together with $C\ge\mathrm eL$ gives $L^C/C!\le(\mathrm eL/C)^C\le1$, whence $\alpha L^C/C!\le\alpha$ and
\[
    \frac{BL^C}{(C+1)!}=\frac{B}{C+1}\cdot\frac{L^C}{C!}\le\frac{B}{4B+1}\le\frac14.
\]
The remaining sum satisfies $\sum_{q<C}L^q/q!\le\sum_{q\ge0}L^q/q!=\mathrm e^L$, and adding the three terms yields
\[
    e_K^{(C)}\le \frac{BL^C}{(C+1)!}+\frac{\alpha L^C}{C!}+\beta\sum_{q=0}^{C-1}\frac{L^q}{q!}\le\frac14+\alpha+\beta\mathrm e^L,
\]
which completes the proof.
\end{proof}

\subsubsection{Proof of Lemma~\ref{lem:picard}}
\label{app:pfpicard}

Up to rounding, the $K$ Euler updates of Algorithm~\ref{alg:vanillafm} state that the trajectory $(x_i)_{i=0}^K$ solves the fixed-point system $x_i=x_0+\sum_{j<i}h_ju(x_j,t_j;w)$, with $h_j:=t_{j+1}-t_j$ as in Appendix~\ref{app:proofconventions}. Algorithm~\ref{alg:parallelpicard} solves this fixed-point system by a Picard iteration~\citep{hairer1993solving}: iteration $r$ recomputes all positions $y_0^{(r+1)},\ldots,y_K^{(r+1)}$ in parallel from the previous iterate $(y_i^{(r)})_{i=0}^K$, replacing the rounded step sizes $h_j$ by the uniform nominal step $1/K$, implemented by dividing each prefix sum $S_i^{(r)}$ by $K$ (line~8 of Algorithm~\ref{alg:parallelpicard}). We show that $C=O(L)$ iterations bring its endpoint $y_K^{(C)}$ within $\frac12$ of $x_K$, thereby proving Lemma~\ref{lem:picard}.

\begin{algorithm}[H]
\caption{Parallel Picard Inference}
\label{alg:parallelpicard}
\begin{algorithmic}[1]
\Function{ParallelPicard}{$w;u,K,C$}
    \State $t_i\leftarrow\round_b(i/K)$ for $i=0,\ldots,K$
    \State Sample $\tilde{x}_0\sim p_{\text{prior}}$; $x_0\leftarrow\round_b(\tilde{x}_0)$
    \State $y_i^{(0)}\leftarrow x_0$ for $i=0,\ldots,K$
    \For{$r=0,\ldots,C-1$}
        \State \textbf{in parallel for} $j=0,\ldots,K-1$: \ $v_j^{(r)}\leftarrow u(y_j^{(r)},t_j;w)$
        \State $S_i^{(r)}\leftarrow\sum_{j=0}^{i-1}v_j^{(r)}$ for $i=0,\ldots,K$
        \State \textbf{in parallel for} $i=0,\ldots,K$: \ $y_i^{(r+1)}\leftarrow\round_b\bigl(x_0+S_i^{(r)}/K\bigr)$
    \EndFor
    \State \Return $y_K^{(C)}$
\EndFunction
\end{algorithmic}
\end{algorithm}

\begin{proof}[Proof of Lemma~\ref{lem:picard}]
Fix $m$ and $w\in\{0,1\}^m$, and write $L=L(m)$. Let $c_0=c_0(B)>0$ be a constant to be chosen at the end of the proof, and take the inference algorithm to be $\textsc{ParallelPicard}(w;u,K,C)$ with
\[
    C:=\lceil\mathrm eL+4B\rceil.
\]
Both $\VanillaFM$ and $\textsc{ParallelPicard}$ draw $\tilde x_0$ from $p_{\text{prior}}$ and round it, so we may couple the two runs by giving them the same $x_0$. The estimates below hold for every realization of $x_0$, which is the only random quantity in either run, and therefore hold with probability~$1$.

\paragraph{The two trajectories.} Let $(x_i)_{i=0}^K$ be the $\VanillaFM$ trajectory and $v_j:=u(x_j,t_j;w)$. In line~7 of Algorithm~\ref{alg:vanillafm}, the subtraction $t_{i+1}-t_i$ and the addition are exact in $F_b$ (Appendix~\ref{app:proofconventions}). Only the product needs rounding, so $x_{i+1}=x_i+\round_b(h_iv_i)$, and each step contributes a rounding residual $\round_b(h_iv_i)-h_iv_i$ of size at most $\delta/2$. Summing these residuals into $\rho_i:=\sum_{j<i}\bigl(\round_b(h_jv_j)-h_jv_j\bigr)$ gives
\begin{equation}\label{eq:vanillatraj}
    x_i=x_0+\sum_{j=0}^{i-1}h_jv_j+\rho_i
    \qquad \left(\|\rho_i\|_\infty\le\frac{i\delta}{2}\le\frac{K\delta}{2}\right).
\end{equation}
Likewise the prefix sum $S_i^{(r)}$ in line~7 of Algorithm~\ref{alg:parallelpicard} is a sum of elements of $F_b$ and hence exact, and $x_0\in F_b$, so line~8 rounds only once. Writing $\rho_i^{(r)}$ for this rounding residual, line~8 gives
\begin{equation}\label{eq:picardtraj}
    y_i^{(r+1)}=x_0+\frac1K\sum_{j=0}^{i-1}v_j^{(r)}+\rho_i^{(r)}
    \qquad \left(\|\rho_i^{(r)}\|_\infty\le\frac{\delta}{2}\right).
\end{equation}
We bound the difference between~\eqref{eq:vanillatraj} and~\eqref{eq:picardtraj} in two steps: we first verify the hypotheses~\eqref{eq:picardbase} and~\eqref{eq:picardrec} of Lemma~\ref{lem:picarderror} for it, and then apply the lemma. The two expressions differ in three respects, which the first step tracks: the weights $1/K$ against $h_j$, the velocities $v_j^{(r)}$ against $v_j$, and the rounding residuals $\rho_i$ and $\rho_i^{(r)}$.

\paragraph{Verifying the conditions of Lemma~\ref{lem:picarderror}.} Put $e_i^{(r)}:=\|y_i^{(r)}-x_i\|_\infty$, and set $\alpha:=\tfrac12(B+1)K\delta$ and $\beta:=(B+1)K\delta$. For $r=0$, using $h_j\ge0$, $\sum_{j<i}h_j=t_i$, $t_i\le s_i+\delta/2$ (Appendix~\ref{app:proofconventions}), and the $B$-boundedness $\|v_j\|_\infty\le B$ of $u$, we have
\[
    e_i^{(0)}=\|x_0-x_i\|_\infty\le Bt_i+\frac{K\delta}{2}\le Bs_i+\frac{(B+K)\delta}{2}\le Bs_i+\alpha,
\]
where the last step uses $B+K\le(B+1)K$ for $K\ge1$. This is~\eqref{eq:picardbase} in Lemma~\ref{lem:picarderror}. For $0\le r<C$, subtract~\eqref{eq:vanillatraj} from~\eqref{eq:picardtraj}:
\[
    y_i^{(r+1)}-x_i=\frac1K\sum_{j=0}^{i-1}\bigl(v_j^{(r)}-v_j\bigr)+\sum_{j=0}^{i-1}\Bigl(\frac1K-h_j\Bigr)v_j+\rho_i^{(r)}-\rho_i.
\]
The two velocities in the first sum are evaluated at the same time $t_j$ and the same condition $w$: $v_j^{(r)}=u(y_j^{(r)},t_j;w)$ and $v_j=u(x_j,t_j;w)$. Therefore, $L$-Lipschitz continuity of $u$ gives $\|v_j^{(r)}-v_j\|_\infty\le Le_j^{(r)}$. By $B$-boundedness and $|h_j-1/K|\le\delta$ (Appendix~\ref{app:proofconventions}), the second sum is bounded by $B\sum_{j<K}|1/K-h_j|\le BK\delta$. By the residual bounds in~\eqref{eq:vanillatraj} and~\eqref{eq:picardtraj}, the last two terms are bounded by $(K+1)\delta/2\le K\delta$. Thus,
\[
    e_i^{(r+1)}\le\frac{L}{K}\sum_{j=0}^{i-1}e_j^{(r)}+BK\delta+K\delta=\frac{L}{K}\sum_{j=0}^{i-1}e_j^{(r)}+\beta,
\]
which is~\eqref{eq:picardrec} in Lemma~\ref{lem:picarderror}.

\paragraph{Applying Lemma~\ref{lem:picarderror}.} With~\eqref{eq:picardbase} and~\eqref{eq:picardrec} verified, Lemma~\ref{lem:picarderror}, applied with the parameters $B$, $L$, $\alpha$, $\beta$ above and the same $C=\lceil\mathrm eL+4B\rceil$, yields $e_K^{(C)}\le\frac14+\alpha+\beta\mathrm e^L$. Since $\delta=2^{-b}$, the precision hypothesis~\eqref{eq:picardprecision} is equivalent to $K\delta\le2^{-c_0}\mathrm e^{-L}$, so, using $\mathrm e^{-L}\le1$,
\[
    \alpha+\beta\mathrm e^L\le\frac{(B+1)2^{-c_0}}{2}+(B+1)2^{-c_0}=\frac32(B+1)2^{-c_0}.
\]
Choose $c_0=c_0(B)$ so that $\frac32(B+1)2^{-c_0}\le\frac14$; then $\|y_K^{(C)}-x_K\|_\infty=e_K^{(C)}\le\frac12$. 

Since $u$ accepts $\mathcal{L}$ with margin $1$ under $\VanillaFM$, every coordinate of $x_K$ is at least $1$ when $w\in\mathcal{L}$ and at most $-1$ otherwise, so every coordinate of the returned value $y_K^{(C)}$ is at least $\frac12$ in the first case and at most $-\frac12$ in the second. Thus $u$ accepts $\mathcal{L}$ with margin $\frac12$ under Algorithm~\ref{alg:parallelpicard}, which runs for $C=O(L)$ iterations because $L\ge1$ and $B$ is a constant, and makes $K$ parallel calls to $u$ in each iteration.
\end{proof}

Aside from the model calls in line~6, one iteration of Algorithm~\ref{alg:parallelpicard} performs only the prefix sums of line~7 and the normalization of line~8. When $K=\poly(m)$ and $b=\poly(m)$, the prefix sums are iterated additions of polynomially many numbers of $\poly(m)$ bits each, while each normalization uses addition and division by the hardwired $K$, and $\mathsf{TC}^0$ circuits compute all these operations exactly~\citep{vollmer1999introduction,hesse2002uniform}.
For constant $L$, Lemma~\ref{lem:picard} together with this observation essentially gives the $\mathsf{TC}^0$ part of Theorem~\ref{thm:vanillaupperbound} (Appendix~\ref{app:pfvanilla} gives the complete proof for every $\mathsf D$). The same iteration gives the growing-$L$ extension in \Sref{sec:theoryextensions}, i.e., Corollary~\ref{cor:vanillapolylogl}.

Picard iterations have also been used to accelerate diffusion sampling~\citep{shih2023parallel}, with a worst-case guarantee of as many iterations as sampling steps and fewer iterations observed empirically. We instead use Picard iteration as a proof device for expressivity: Lemma~\ref{lem:picard} shows that the $K$ sequential steps of $\VanillaFM$ collapse into $O(L)$ parallel rounds regardless of $K$, which caps the sequential depth that $\VanillaFM$ can exploit.

As a side remark, this collapse relies on approximating the endpoint only to constant accuracy, which suffices for acceptance with margin. When exponentially small additive error is required, even a Lipschitz continuous ODE can be hard to solve: \citet{kawamura2010lipschitz} constructs an ODE with a polynomial-time computable, Lipschitz continuous right-hand side whose solution is $\mathsf{PSPACE}$-complete, with the answer encoded in the solution at precision $2^{-\poly(m)}$.

\subsubsection{Approximate Normalized Prefix Sums in \texorpdfstring{$\mathsf{AC}^0$}{AC0}}
\label{app:prefixsums}

The exact prefix sums in line~7 of Algorithm~\ref{alg:parallelpicard}, however, lie beyond $\mathsf{AC}^0$: the low-order bit of a sum is the parity of the summands' low-order bits~\citep{furst1984parity,hastad1986almost}. To obtain the $\mathsf{AC}^0$ part of Theorem~\ref{thm:vanillaupperbound}, we approximate the sums, which suffices because the Picard iteration needs each normalized prefix sum only to within a constant additive error, which merely enlarges $\beta$ in~\eqref{eq:picardrec}. Quantizing the summands to constantly many levels makes each normalized prefix sum an affine function of a \emph{count of ones}, and $\mathsf{AC}^0$ approximates such a count to within any constant fraction of the string length~\citep{ajtai1990approximate}. Formally, we have the following lemma.

\begin{lemma}\label{lem:ac0prefix}
Let $M\ge1$ and $\varepsilon\in(0,1)$ be constant reals. Let $N=N(m)\le\poly(m)$ be a positive integer, and let the number of fractional bits $b=b(m)=O(\log m)$ satisfy $2^{-b}\le\varepsilon$. There is a family of $\mathsf{AC}^0$ circuits of size $\poly(m)$ that maps any input $g_0,\ldots,g_{N-1}\in F_b\cap[-M,M]$ to outputs $A_0,\ldots,A_N\in F_b$ with
\[
    \Bigl|A_i-\frac1N\sum_{j=0}^{i-1}g_j\Bigr|\le\varepsilon
    \qquad\text{for all }0\le i\le N.
\]
\end{lemma}

\begin{proof}
Fix a constant $R:=\lceil8M/\varepsilon\rceil$. For each $j$, let $k_j\in\{0,\ldots,R\}$ be the integer nearest to $(g_j+M)R/(2M)$ and put $\hat g_j:=\frac{2M}{R}k_j-M$, so that $|\hat g_j-g_j|\le M/R\le\varepsilon/8$. Write $\chi_j\in\{0,1\}^R$ for the unary encoding of $k_j$. Since $g_j\in F_b\cap[-M,M]$ occupies $O(\log m)$ bits, the table bound (as in Appendix~\ref{app:proofconventions}) maps $g_j$ to $\chi_j$ at depth two. For each $i$, the string $\chi_0\cdots\chi_{i-1}0^{(N-i)R}$ has length $NR$ and exactly $S_i:=\sum_{j<i}k_j$ ones. $\mathsf{AC}^0$ circuits can count these ones up to any constant fraction $\gamma>0$ of the length, i.e., output $\hat S_i$ with $|\hat S_i-S_i|\le\gamma NR$~\citep{ajtai1990approximate}. Taking $\gamma:=\varepsilon/(16M)$ gives $\frac{2M}{NR}|\hat S_i-S_i|\le\varepsilon/8$. As $\frac1N\sum_{j<i}\hat g_j=\frac{2M}{NR}S_i-\frac{iM}{N}$, set
\[
    A_i:=\round_b\Bigl(\frac{2M}{NR}\hat S_i-\frac{iM}{N}\Bigr).
\]
Here $i$ and $N$ are hardwired and $\hat S_i$ occupies $O(\log m)$ bits, so the table bound maps $\hat S_i$ to $A_i$ at depth two as well. Composing these three stages, each of constant depth and size $\poly(m)$, gives the claimed $\mathsf{AC}^0$ family. It remains to bound the error of $A_i$. Inserting the intermediate quantities $\frac{2M}{NR}\hat S_i-\frac{iM}{N}$ and $\frac1N\sum_{j<i}\hat g_j$ and applying the triangle inequality,
\begin{align*}
    \Bigl|A_i-\frac1N\sum_{j<i}g_j\Bigr|
    &\le\Bigl|A_i-\Bigl(\frac{2M}{NR}\hat S_i-\frac{iM}{N}\Bigr)\Bigr|
    +\frac{2M}{NR}\bigl|\hat S_i-S_i\bigr|
    +\frac1N\sum_{j<i}\bigl|\hat g_j-g_j\bigr|\\
    &\le2^{-b-1}+\frac{\varepsilon}{8}+\frac{\varepsilon}{8}
    \le\frac{\varepsilon}{2}+\frac{\varepsilon}{4}\le\varepsilon,
\end{align*}
the three terms being the rounding, counting, and quantization errors, and the second inequality using $2^{-b}\le\varepsilon$.
\end{proof}

\subsubsection{Proof of Theorem~\ref{thm:vanillaupperbound}}
\label{app:pfvanilla}

To place the $\VanillaFM$ sampler in $\mathsf{D}$, we split on the number of steps: a bounded $K$ is handled by unrolling the $K$ Euler steps directly, and otherwise we substitute the approximate prefix sums of Lemma~\ref{lem:ac0prefix} into the Picard iteration of Lemma~\ref{lem:picard} (Algorithm~\ref{alg:parallelpicard}). The substitution perturbs the error recurrence~\eqref{eq:picardrec} by one constant per round, small enough for the endpoint to stay within a constant acceptance margin, and it keeps the circuit at constant depth over the model calls.

\begin{proof}[Proof of Theorem~\ref{thm:vanillaupperbound}]
By~\eqref{eq:vanillaprecision} there is a constant $c>0$ such that $b\ge(1+c)\log_2K$ for all but finitely many $m$. Let $c_1=c_1(B)>0$ satisfy
\begin{equation}\label{eq:cdchoice}
    \Bigl(\frac{3B}{2}+1\Bigr)2^{-c_1}\le\frac18,
\end{equation}
and split the proof into the two cases $b<\log_2K+(\log_2\mathrm e)L+c_1$ and $b\ge\log_2K+(\log_2\mathrm e)L+c_1$. Since the circuit families are non-uniform, each $m$ may use the construction of its own case.

\paragraph{Small $K$.} Suppose first that $b<\log_2K+(\log_2\mathrm e)L+c_1$. Then $K$ is bounded: only finitely many such $m$ have $b<(1+c)\log_2K$ by~\eqref{eq:vanillaprecision}, and every other one satisfies $(1+c)\log_2K\le b<\log_2K+(\log_2\mathrm e)L+c_1$, hence $K<2^{((\log_2\mathrm e)L+c_1)/c}$, a constant. 

We then reproduce the sampler directly. Fix $m$ and a hardwired initialization $\hat x_0$ as in Appendix~\ref{app:proofconventions}. The circuit runs the $K$ Euler steps of Algorithm~\ref{alg:vanillafm} from $x_0:=\hat x_0$ with the time grid hardwired: step $i$ evaluates $v_i:=u(x_i,t_i;w)$ with one copy of $u$ and then sets $x_{i+1}:=x_i+\round_b(h_iv_i)$ coordinate-wise, where $h_i:=t_{i+1}-t_i$. This reproduces the $\VanillaFM$ trajectory exactly, so its endpoint has every coordinate at least $1$ when $w\in\mathcal L$ and at most $-1$ otherwise, and the sign of the first coordinate decides $\mathcal L$.

The circuit lies in $\mathsf D$. Each of the constantly many steps makes one call to the $\mathsf D$ circuit $u$. Each coordinate-wise update is a map on constantly many scalars of $O(\log m)$ bits, since $b=O(\log m)$ and, with $K$ bounded in this case, every state and velocity has magnitude bounded by a constant. The table bound places this map in $\mathsf{AC}^0\subseteq\mathsf D$ (Appendix~\ref{app:proofconventions}). Composing constantly many $\mathsf D$ circuits keeps the result in $\mathsf D$, so $\mathcal L\in\mathsf D$ in this case.

\paragraph{Large $K$.} Otherwise $b\ge\log_2K+(\log_2\mathrm e)L+c_1$. We construct the circuit by running Algorithm~\ref{alg:parallelpicard} with its exact prefix sums replaced by the approximate ones of Lemma~\ref{lem:ac0prefix}, and then show that it decides $\mathcal L$ and lies in $\mathsf D$. Set $M:=\max\{B,1\}$ and $\eta:=\tfrac18\mathrm e^{-L}$, and use the same $C=\lceil\mathrm eL+4B\rceil$ as in Appendix~\ref{app:pfpicard}. This lower bound on $b$ is equivalent to $K\delta\le2^{-c_1}\mathrm e^{-L}$, which with~\eqref{eq:cdchoice} gives $\delta\le K\delta\le\eta$ and, by~\eqref{eq:vanillaprecision}, $K\le2^b=\poly(m)$.

\paragraph{The circuit.} Fix $m$ and a hardwired initialization $\hat x_0$ as in Appendix~\ref{app:proofconventions}, and let $\Phi_m$ be the circuit of Lemma~\ref{lem:ac0prefix} for $N:=K$, $M$, and $\varepsilon:=\eta$, whose hypotheses hold by the previous paragraph and~\eqref{eq:vanillaprecision}. On input $w$, the circuit runs $C$ rounds of Algorithm~\ref{alg:parallelpicard} from $y_i^{(0)}:=\hat x_0$, with the time grid hardwired and lines~7 and~8 replaced by $y_i^{(r+1)}:=\hat x_0+A_i^{(r)}$. Here $A_0^{(r)},\ldots,A_K^{(r)}$ are the outputs of $\Phi_m$ applied coordinate-wise to the velocities $v_j^{(r)}:=u(y_j^{(r)},t_j;w)$, $j=0,\ldots,K-1$. Each $A_i^{(r)}$ therefore approximates the normalized prefix sum $\frac1K\sum_{j<i}v_j^{(r)}$. $B$-boundedness puts the velocities in $F_b\cap[-M,M]$ as Lemma~\ref{lem:ac0prefix} requires. Their normalized prefix sums also lie in $[-M,M]$, so every scalar here has constant magnitude and hence $O(\log m)$ bits, as the table bound below requires.

\paragraph{Correctness.} We show that the circuit decides $\mathcal L$ by the sign of $(y_K^{(C)})_1$. Compare with the $\VanillaFM$ run from $x_0=\hat x_0$, with $v_j$ and $\rho_i$ as in~\eqref{eq:vanillatraj}, and put $e_i^{(r)}:=\|y_i^{(r)}-x_i\|_\infty$. Since $\hat x_0,A_i^{(r)}\in F_b$, the new update involves no rounding, so subtracting~\eqref{eq:vanillatraj} from $y_i^{(r+1)}=\hat x_0+A_i^{(r)}$ gives
\[
    y_i^{(r+1)}-x_i
    =\Bigl(A_i^{(r)}-\tfrac1K\textstyle\sum_{j<i}v_j^{(r)}\Bigr)
    +\tfrac1K\sum_{j<i}(v_j^{(r)}-v_j)
    +\sum_{j<i}\Bigl(\tfrac1K-h_j\Bigr)v_j-\rho_i.
\]
The first term is bounded by $\eta$ by Lemma~\ref{lem:ac0prefix}; the remaining three terms are bounded by $\frac LK\sum_{j<i}e_j^{(r)}$, $BK\delta$, and $K\delta/2$ exactly as in Appendix~\ref{app:pfpicard}, with $\rho_i^{(r)}$ absent. Thus~\eqref{eq:picardbase} holds with the same $\alpha:=\tfrac12(B+1)K\delta$ and~\eqref{eq:picardrec} with $\beta:=\eta+(B+\tfrac12)K\delta$, and Lemma~\ref{lem:picarderror} gives
\[
    e_K^{(C)}\le\frac14+\alpha+\beta\mathrm e^L
    =\frac14+\eta\mathrm e^L+\Bigl(\frac{B+1}{2}+\Bigl(B+\frac12\Bigr)\mathrm e^L\Bigr)K\delta
    \le\frac14+\frac18+\Bigl(\frac{3B}{2}+1\Bigr)2^{-c_1}\le\frac12,
\]
using $\eta\mathrm e^L=\frac18$, then $K\delta\mathrm e^L\le2^{-c_1}$ and $K\delta\le2^{-c_1}$, and finally~\eqref{eq:cdchoice}. As in Appendix~\ref{app:pfpicard}, acceptance with margin $1$ under $\VanillaFM$ at the positive-probability initialization $\hat x_0$ (Appendix~\ref{app:proofconventions}) now gives $(y_K^{(C)})_1\ge\frac12$ if $w\in\mathcal{L}$ and $(y_K^{(C)})_1\le-\frac12$ otherwise.

\paragraph{Size and depth.} The run consists of $C$ rounds, each consisting of $K$ parallel calls to $u$ surrounded by $\mathsf{AC}^0$ circuits of size $\poly(m)$, namely $n$ copies of $\Phi_m$, one per state coordinate, and, by the table bound, the addition of $\hat x_0$. Since $C=O(1)$ and $K$, $n$ are $\poly(m)$, the whole circuit has size $\poly(m)$ and depth $O(1)$ over the copies of $u$. The surrounding arithmetic lies in $\mathsf{AC}^0\subseteq\mathsf{D}$, and composing these $C$ layers of parallel copies of the $\mathsf{D}$ circuit $u$ keeps the result in $\mathsf{D}$. 
Since the sign test $(y_K^{(C)})_1\ge0$ reads only the $O(\log m)$ bits of one scalar, the table bound computes it in $\mathsf{AC}^0$, correctly outputting $\mathbf{1}\{w\in\mathcal{L}\}$. Hence $\mathcal{L}\in\mathsf{D}$.
\end{proof}

\subsection{Expressiveness Lower Bound for \texorpdfstring{$\SelfCondFM$}{SelfCondFM}}
\label{app:scondlb}

We now prove Theorem~\ref{thm:selfcondlowerbound} by making $\SelfCondFM$ simulate a Turing machine, using the endpoint prediction $z_i$ as the Turing machine's tape. Appendix~\ref{app:tapesim} puts a space-bounded Turing machine with advice into a form in which the new value of each cell is determined by that cell and its two neighbors. Appendix~\ref{app:pfselfcond} then builds a model that applies this rewrite rule to $z_i$ coordinate-wise, so that each output coordinate reads constantly many inputs and the model stays in $\mathsf{AC}^0$. Since the extrapolation $z_{i+1}=x_i+(1-t_i)v_i$ (line~8 of Algorithm~\ref{alg:selfcondfm}) carries the current velocity with sufficient magnitude while $t_i$ is small, the tape can advance one machine step per sampling step, and taking $K$ large enough lets the simulated Turing machine halt. The tape then freezes on the verdict, which keeps the velocity constant and decides the sign of $x_K$.

\subsubsection{Local Simulation of Space-Bounded Machines}
\label{app:tapesim}

To prepare the machine for the coordinate-wise simulation of Appendix~\ref{app:pfselfcond}, we flatten it into a cellular automaton: the initial tape is the input followed by one hardwired string carrying the advice, and the machine fills the tape with one of two verdict symbols before halting.

\begin{lemma}\label{lem:local}
Let $\mathcal{L}$ be any language in $\mathsf{PSPACE/poly}$. There exist an alphabet $\Sigma\supseteq\{0,1\}$ of constant size containing two symbols $\mathsf{acc}$ and $\mathsf{rej}$, a boundary symbol $\#\notin\Sigma$, a local rule $\varphi:(\Sigma\cup\{\#\})^3\to\Sigma$, a tape length $P(m)=\poly(m)$, a horizon $H(m)\le2^{\poly(m)}$, and strings $\hat a_m\in\Sigma^{P(m)-m}$, such that the following holds for every $w\in\{0,1\}^m$. The sequence
\[
    c^{(1)}:=\operatorname{concat}(w,\hat a_m),\qquad
    c^{(i+1)}_p:=\varphi\bigl(c^{(i)}_{p-1},c^{(i)}_p,c^{(i)}_{p+1}\bigr)\quad(1\le p\le P(m)),
\]
with the convention $c^{(i)}_0=c^{(i)}_{P(m)+1}=\#$, satisfies $c^{(i)}=\mathsf{acc}^{P(m)}$ for all $i\ge H(m)$ if $w\in\mathcal{L}$, and $c^{(i)}=\mathsf{rej}^{P(m)}$ otherwise. If moreover $\mathcal{L}\in\mathsf{P/poly}$, then $H(m)=\poly(m)$.
\end{lemma}

\begin{proof}
Take a deterministic Turing machine $\mathcal{M}$ that decides $w\in\mathcal{L}$ on the cells $1,\ldots,P(m)$, for some $P(m)=\poly(m)$, from the tape initialization $\operatorname{concat}(w,\diamond,a_m)$, with $\diamond$ a separator and $a_m$ an advice string of length $P(m)-m-1$. Being deterministic and halting, $\mathcal{M}$ never repeats one of its at most $2^{O(P(m))}$ configurations, so it halts within $T(m)\le2^{\poly(m)}$ steps, and within $T(m)\le \poly(m)$ steps in the $\mathsf{P/poly}$ case, where $\mathcal{M}$ can be taken polynomial-time. We may assume that $\mathcal{M}$ starts in an initial state $q_0$ with its head on the separator and, before halting, overwrites every cell with $\mathsf{acc}$ or $\mathsf{rej}$ according to its verdict.

Fold the head and its state into the alphabet: $\Sigma:= \Gamma\cup(\mathcal{Q}\times\Gamma)$, where $\Gamma\supseteq\{0,1,\diamond,\mathsf{acc},\mathsf{rej}\}$ is the tape alphabet and $\mathcal{Q}$ the state set. The whole run of $\mathcal{M}$ is then local: each cell's next content is determined by itself and its two neighboring cells. We encode this one-step update by the rule $\varphi$, with $\#$ read as an end-marker and a halting head dissolving into its verdict. Since a cell changes only when the head acts on it, the headless all-$\mathsf{acc}$ and all-$\mathsf{rej}$ configurations are fixed points. Taking $\hat a_m:=\operatorname{concat}((q_0,\diamond),a_m)$ makes $c^{(1)}$ the initial configuration, so $c^{(i)}$ is the configuration after $i-1$ steps. A halting head thus appears by $c^{(T(m)+1)}$ and dissolves one step later, leaving the tape filled with the verdict, a fixed point. Hence $H(m):=T(m)+2$ works.
\end{proof}

\subsubsection{Proof of Theorem~\ref{thm:selfcondlowerbound}}
\label{app:pfselfcond}

To make the sampler simulate the machine of Lemma~\ref{lem:local}, we store the binary encoding of $c^{(i)}$ in the prediction $z_i$, one bit per coordinate, and take $K=16H$. The simulation then runs at times $t_i\le\frac18$, where the extrapolation $z_{i+1}=x_i+(1-t_i)v_i$ delivers the velocity with sufficient magnitude, so $z_{i+1}$ decodes exactly to the next configuration computed by the model. Once the tape freezes on the verdict, the velocity stays pinned at $\pm4$, 
ensuring that every coordinate of $x_K$ clears the acceptance margin with the sign of the verdict.

\begin{proof}[Proof of Theorem~\ref{thm:selfcondlowerbound}]
Let $\Sigma$, $\varphi$, $P=P(m)$, $H=H(m)$, and $\hat a_m$ be as in Lemma~\ref{lem:local}. Put $\kappa:=\lceil\log_2|\Sigma|\rceil$ and fix an injection $\enc:\Sigma\to\{0,1\}^\kappa$ with $\enc(\mathsf{acc})=\mathbf1^\kappa$ and $\enc(\mathsf{rej})=\mathbf0^\kappa$, calling $\enc(\sigma)$ the \emph{codeword} of $\sigma$. Extend $\enc$ to configurations $c\in\Sigma^P$ cell by cell, $\enc(c):=\operatorname{concat}(\enc(c_1),\ldots,\enc(c_P))$, and set
\[
    n:=\kappa P,\qquad K:=16H.
\]
The codeword of cell $p$ thus occupies the \emph{block} of coordinates $\kappa(p-1)+1,\ldots,\kappa p$, so, with $p:=\lceil j/\kappa\rceil$, coordinate $j$ carries bit $q:=j-\kappa(p-1)$ of the codeword of cell $p$. Let $N_p:=\{p-1,p,p+1\}\cap\{1,\ldots,P\}$ and let $J_j:=\bigcup_{r\in N_p}\{\kappa(r-1)+1,\ldots,\kappa r\}$ collect the blocks of cell $p$ and of its neighbors on the tape, and let $I_j:=\{p\}$ if $p\le m$ and $I_j:=\emptyset$ otherwise. The precision $b=b(m)$ is fixed at the end of the proof. Everything below holds for any choice with $b\ge3$ and $b=O(\log m)$. The next three paragraphs construct $u$ and verify conditions~(a) and~(b) of Theorem~\ref{thm:selfcondlowerbound}, the paragraph \emph{The run} proves acceptance, and the paragraph \emph{Parameters} fixes $b$ to meet condition~(c).

\paragraph{Encoding a bit.} A coordinate of the prediction $z_i$ encodes one bit of a configuration, or the placeholder $\bot$ when not yet initialized, via the \emph{encoding set} $E$ and the decoding map $\dec$ given by
\begin{gather*}
    E:=\{0\}\cup\{\zeta\in F_b:2\le|\zeta|\le5\},\\
    \dec(\zeta):=\begin{cases}
        \bot & \zeta =0,\\
        1 & \zeta\ge 2,\\
        0 & \zeta\le -2.
    \end{cases}
\end{gather*}
The three bands $\{0\}$, $E\cap[2,5]$, and $E\cap[-5,-2]$ are deliberately wide and pairwise at distance at least $2$. Therefore, the encoding can tolerate the sampler's perturbation, while an $8$-Lipschitz model can still tell the bands apart.

\paragraph{The local maps.} Define $\tau_j:\{0,1\}^{I_j}\times\{0,1,\bot\}^{J_j}\to\{0,1\}$ by
\[
    \tau_j(\omega,\psi):=\begin{cases}
        \text{bit }q\text{ of }\enc(\omega_p) & \psi=\bot^{J_j}\text{ and }p\le m,\\
        \text{bit }q\text{ of }\enc\bigl((\hat a_m)_{p-m}\bigr) & \psi=\bot^{J_j}\text{ and }p>m,\\
        \text{bit }q\text{ of }\enc\bigl(\varphi(\sigma_{p-1},\sigma_p,\sigma_{p+1})\bigr) & \psi=(\enc(\sigma_r))_{r\in N_p},\\
        0 & \text{otherwise},
    \end{cases}
\]
with the convention $\sigma_0=\sigma_{P+1}=\#$ at the tape ends. The first two cases initialize the tape and the third advances the machine by one step, so that with $\bar c^{(0)}:=\bot^n$ and $\bar c^{(i)}:=\enc(c^{(i)})$ for $i\ge1$,
\begin{equation}\label{eq:steprule}
    \bar c^{(i+1)}_j=\tau_j\bigl(w_{I_j},\bar c^{(i)}_{J_j}\bigr)
    \qquad\text{for all }j\text{ and all }i\ge0.
\end{equation}

\paragraph{The model.} Write $D_j:=\{0,1\}^{I_j}\times E^{J_j}$, a finite nonempty subset of $F_b^{I_j}\times F_b^{J_j}$, and define $f_j:D_j\to\{-4,4\}$ by
\[
    f_j(\omega,\psi):=\begin{cases}
        +4 & \tau_j\bigl(\omega,\dec(\psi)\bigr)=1,\\
        -4 & \tau_j\bigl(\omega,\dec(\psi)\bigr)=0,
    \end{cases}
\]
where $\dec$ is applied coordinate-wise to $\psi$. Then $f_j$ is $8$-Lipschitz on $D_j$ in the maximum norm: if two points have different values, they differ in a Boolean coordinate or lie in different bands at some $E$ coordinate, so they differ by at least $1$ in that coordinate, while the values differ by $8$.

Extend $f_j$ to all of $F_b^{I_j}\times F_b^{J_j}$ by a clipped McShane extension~\citep{mcshane1934extension},
\[
    \bar f_j(\omega,\psi):=\max\Bigl\{-4,\ \min\Bigl\{4,\ \min_{(\omega',\psi')\in D_j}\bigl(f_j(\omega',\psi')+8\max\{\|\omega-\omega'\|_\infty,\|\psi-\psi'\|_\infty\}\bigr)\Bigr\}\Bigr\}.
\]
Then $\bar f_j$ is $8$-Lipschitz and $4$-bounded, and agrees with $f_j$ on $D_j$. It needs no rounding, because $F_b$ is closed under the arithmetic involved.

For every $m$, define the restriction of the single model $u$ with
$b=b(m)$, $n=n(m)$, and $|\ctx|=m+n$ by
\[
    \bigl(u(x,t;\ctx)\bigr)_j
    :=\bar f_j\bigl(\ctx_{I_j},
    (\ctx_{m+k})_{k\in J_j}\bigr),
    \qquad j=1,\ldots,n,
\]
and set $u(x,t;\ctx):=\mathbf0^n$ on all other encoded inputs.
Since each output coordinate is $\bar f_j$ of a fixed subtuple of the condition, the model is $4$-bounded and $8$-Lipschitz continuous in the sense of
\Sref{sec:fmsetup}. Each $\bar f_j$ reads
$|I_j|+|J_j|=O(1)$ scalars, and an $\mathsf{AC}^0$ circuit first clips them to $[-8,8]$. This leaves $\bar f_j$ unchanged, since $D_j\subseteq[-5,5]^{I_j\cup J_j}$ and a scalar of magnitude at least $8$ forces $\bar f_j=4$. The clipped scalars have $O(\log m)$ bits, so the table bound (Appendix~\ref{app:proofconventions}) computes $\bar f_j$ at depth two
and size $\poly(m)$. Hence $u$ is computed by $\mathsf{AC}^0$ circuits, and conditions~(a) and~(b) of Theorem~\ref{thm:selfcondlowerbound} hold.

\paragraph{The run.} We now prove that $u$ accepts $\mathcal{L}$ with margin $1$ under $\SelfCondFM$. Fix $w\in\{0,1\}^m$ and run $\SelfCondFM(w;u,K)$. The estimates below hold for every realization of $x_0$, which is the only randomness in the run, and therefore hold with probability~$1$. By line~6 of Algorithm~\ref{alg:selfcondfm}, its condition at step $i$ is $\operatorname{concat}(w,z_i)$, of length exactly $m+n$, so $(v_i)_j=\bar f_j(w_{I_j},(z_i)_{J_j})$. The run splits at step $H$. For the first phase, when $0\le i\le H$, one sampling step advances the tape by one machine step, so that $z_i$ decodes to $\bar c^{(i)}=\enc(c^{(i)})$. For the second phase, from $H$ on, Lemma~\ref{lem:local} freezes the tape on the verdict, the velocity stays pinned at $\pm4$, and $x_K$ inherits this sign. We prove these two claims in turn.

For the first phase, we show that $z_i$ decodes to $\bar c^{(i)}$ for every $0\le i\le H$. This phase runs at small times: since $b\ge3$, we have $\delta/2\le2^{-4}$, so $t_i\le i/K+2^{-4}$, and $K=16H$ gives
\begin{equation}\label{eq:tsmall}
    t_i\le\frac{1}{16}+\frac{1}{16}=\frac18\qquad\text{for all }0\le i\le H.
\end{equation}
We prove by induction on $i$ the stronger claim that for every $0\le i\le H$: (a) $(z_i)_j\in E$ with $\dec((z_i)_j)=\bar c^{(i)}_j$ for all $j$; (b) $\|x_i\|_\infty\le1+4t_i$; and (c) all arithmetic up to step $i$ is exact. All three hold at $i=0$. Given them at some $i<H$, claim (a) puts the argument of $\bar f_j$ in $D_j$, where $\bar f_j$ agrees with $f_j$, so that $(v_i)_j$ is $+4$ if $\bar c^{(i+1)}_j=1$ and $-4$ otherwise, by~\eqref{eq:steprule}. As $F_b$ is closed under multiplication by integers, both updates in lines~7 and~8 of Algorithm~\ref{alg:selfcondfm} are exact, giving (c) and $\|x_{i+1}\|_\infty\le\|x_i\|_\infty+4(t_{i+1}-t_i)\le1+4t_{i+1}$, which is (b). For (a), if $\bar c^{(i+1)}_j=1$ then $(z_{i+1})_j=(x_i)_j+4(1-t_i)$ lies in $[3-8t_i,\,5]\subseteq[2,5]$ by (b) and~\eqref{eq:tsmall}, and symmetrically in $[-5,-2]$ if $\bar c^{(i+1)}_j=0$. This completes the induction, so $z_H$ decodes to $\bar c^{(H)}$, the encoded verdict tape.

For the second phase, we show that the velocity stays at $\pm4$ from step $H$ on, so that every coordinate of $x_K$ has the sign of the verdict and magnitude at least $1$. The first-phase bound on $z_{i+1}$ relied on $t_i\le1/8$, which fails as $t_i$ approaches $1$, so we use instead that the tape is frozen. Suppose that $w\in\mathcal{L}$. Then Lemma~\ref{lem:local} and $\enc(\mathsf{acc})=\mathbf1^\kappa$ give $\bar c^{(i)}=\mathbf1^n$ for all $i\ge H$. We show by induction on $H\le i\le K$ that (a$'$) $(z_i)_j\in E\cap[2,5]$ for all $j$, and (b$'$) $x_i=x_H+4(t_i-t_H)\mathbf1^n$ exactly. At $i=H$, (a$'$) is (a) with $\bar c^{(H)}=\mathbf1^n$, and (b$'$) is trivial. Given them at some $i<K$, (a$'$) puts the argument of $\bar f_j$ in $D_j$ with $\dec((z_i)_{J_j})=\mathbf1^{J_j}=\bar c^{(i)}_{J_j}$, so~\eqref{eq:steprule} and $\bar c^{(i+1)}=\mathbf1^n$ give $(v_i)_j=+4$ for every $j$. Line~7 of Algorithm~\ref{alg:selfcondfm} then gives (b$'$) at $i+1$, and in line~8 the two $t_i$ terms cancel:
\[
    (z_{i+1})_j=(x_i)_j+4(1-t_i)=(x_H)_j+4(1-t_H)\in[3-8t_H,\,5]\subseteq[2,5]
\]
by (b) and~\eqref{eq:tsmall} at $i=H$, which is (a$'$) at $i+1$. Thus $z_i$ stops moving and the velocity stays at $+4$ until the end. Taking $i=K$ in (b$'$) and using $t_K=1$ gives $(x_K)_j=(x_H)_j+4(1-t_H)\ge2\ge1$ for every $j$. Symmetrically, if $w\notin\mathcal{L}$, then $\bar c^{(i)}=\mathbf0^n$ for $i\ge H$, the velocity stays at $-4$, and $(x_K)_j\le-1$. Hence $u$ accepts $\mathcal{L}$ with margin $1$ under $\SelfCondFM$.

\paragraph{Parameters.} Finally, we fix $b$ and read off the two cases of Theorem~\ref{thm:selfcondlowerbound}. For $\mathcal{L}\in\mathsf{PSPACE/poly}$, Lemma~\ref{lem:local} gives $n(m)=\kappa P(m)=\poly(m)$ and $K(m)=16H(m)\le2^{\poly(m)}$, and $b(m):=3=O(\log m)$ meets the sole requirement $b\ge3$ made above. If moreover $\mathcal{L}\in\mathsf{P/poly}$, the same lemma gives $H(m)=\poly(m)$ and hence $K(m)=\poly(m)$, while $b$ must in addition satisfy the bound~\eqref{eq:vanillaprecision} on its fractional bits. The choice
\[
    b(m):=\max\bigl\{3,\ \lceil2\log_2K(m)\rceil\bigr\}
\]
satisfies $b(m)\ge2\log_2K(m)$, and hence~\eqref{eq:vanillaprecision}. Since $K(m)$ is now polynomial, this $b(m)$ is $O(\log m)$, and the construction remains valid. In both cases condition~(c) holds, which completes the proof.
\end{proof}

\clearpage

\section{Proofs of Theoretical Extensions}
\label{app:extensions}

This appendix proves the extension results presented in~\Sref{sec:theoryextensions}. Appendix~\ref{app:polylogl} establishes Corollary~\ref{cor:vanillapolylogl}, the circuit analogue of Lemma~\ref{lem:picard} for a growing Lipschitz constant $L$. Appendix~\ref{app:lplipschitz} extends the analyses from $\ell_\infty$ to $\ell_p$ Lipschitz continuity for any fixed $p\in[1,\infty]$.
Appendix~\ref{app:rk} extends the upper bounds for $\VanillaFM$ from the Euler method to fixed explicit Runge--Kutta samplers. Finally, Appendix~\ref{app:depthmatch} encodes majority circuits into $\VanillaFM$, demonstrating that the $O(L)$ rounds of Lemma~\ref{lem:picard} match circuit depth and cannot be improved without collapsing the majority-circuit depth hierarchy.

\subsection{The Growing Lipschitz Constant}
\label{app:polylogl}

The proof of Lemma~\ref{lem:picard} bounds the number of Picard rounds by $O(L)$ without using the constancy of $L$, so the lemma already allows a growing $L(m)$. Since the exact prefix sums of each round are computable in $\mathsf{TC}^0$ when $K=\poly(m)$ and $b=\poly(m)$, the whole run then becomes a majority circuit of depth $O(L)$ instead of $O(1)$.

\begin{corollary}\label{cor:vanillapolylogl}
Let $u$ be a model that can be computed by circuits in $\mathsf{TC}^0$, and $\mathcal{L}$ be any language that can be accepted by $u$ with margin $1$ under $\VanillaFM$. Assume that $u$ is $B$-bounded for a constant $B>0$ and $L(m)$-Lipschitz continuous for some $L(m)\ge1$ that may grow with $m$. Then there exists a constant $c_0=c_0(B)>0$, independent of $m$, such that as long as the state dimension and the number of steps satisfy $n=\poly(m)$ and $K=\poly(m)$, and the number of fractional bits satisfies
\begin{equation}\label{eq:vanillapolyloglprecision}
\log_2 K(m)+(\log_2\mathrm e)L(m)+c_0\le b(m)\le\poly(m)\qquad\text{for every }m,
\end{equation}
$\mathcal{L}$ can be decided by majority circuits of size $\poly(m)$ and depth $O(L(m))$. In particular, $\mathcal{L}\in\mathsf{TC}^{k}$ if $L(m)=O(\log^k m)$ for a constant integer $k\ge0$, and $\mathcal{L}\in\mathsf{NC}$ if $L(m)=\poly(\log m)$.
\end{corollary}

\begin{proof}
We turn the algorithm of Lemma~\ref{lem:picard} into a majority circuit. The lower bound in~\eqref{eq:vanillapolyloglprecision} is exactly the hypothesis~\eqref{eq:picardprecision} of Lemma~\ref{lem:picard}, with $c_0=c_0(B)$ the constant of that lemma. Hence the lemma's algorithm, namely Algorithm~\ref{alg:parallelpicard} with $C:=\lceil\mathrm eL+4B\rceil$ rounds, run from a hardwired initialization $\hat x_0$ (Appendix~\ref{app:proofconventions}), accepts $\mathcal{L}$ with margin $\frac12$: every coordinate of its output $y_K^{(C)}$ is at least $\frac12$ when $w\in\mathcal{L}$ and at most $-\frac12$ otherwise.

As noted after the proof of Lemma~\ref{lem:picard}, one round of that algorithm consists of $K$ parallel calls to $u$ followed by prefix sums and normalizations, which $\mathsf{TC}^0$ computes exactly since $K$ and $b$ are $\poly(m)$~\citep{vollmer1999introduction,hesse2002uniform}. Since $u$ is itself computed by $\mathsf{TC}^0$ circuits, composing the $C$ rounds with the indicator $\mathbf{1}\{(y_K^{(C)})_1\ge0\}$ yields a majority circuit deciding $\mathcal{L}$. Its depth is $O(C)=O(L(m))$, and its size is polynomial because each round has size $\poly(m)$ and the precision hypothesis~\eqref{eq:vanillapolyloglprecision} bounds $(\log_2\mathrm e)L(m)$ by $b(m)=\poly(m)$, so that $C=\poly(m)$.

If $L(m)=O(\log^k m)$, the depth is $O(\log^k m)$ and $\mathcal{L}\in\mathsf{TC}^k$ by definition. If $L(m)=\poly(\log m)$, then $L(m)=O(\log^k m)$ for some constant $k$, so $\mathcal{L}\in\mathsf{TC}^k\subseteq\mathsf{NC}$ (Appendix~\ref{app:proofconventions}).
\end{proof}

\subsection{\texorpdfstring{$\ell_p$}{lp} Lipschitz Condition}
\label{app:lplipschitz}

We now extend Theorems~\ref{thm:vanillaupperbound} and~\ref{thm:selfcondlowerbound} to the $\ell_p$ norm for a fixed $p\in[1,\infty]$. We say that $u$ is \emph{$L$-Lipschitz in $\ell_p$} if
\[
\|u(x,t;\ctx)-u(x',t;\ctx')\|_p\le L\max\{\|x-x'\|_p,\|\ctx-\ctx'\|_p\}
\]
for all $x,x'\in F_b^n$, equal-length $\ctx,\ctx'\in F_b^*$, and $t\in[0,1]\cap F_b$.

\begin{proposition}[Upper bound under $\ell_p$]\label{prop:vanillalp}
Theorem~\ref{thm:vanillaupperbound} remains true if its $\ell_\infty$ Lipschitz hypothesis is replaced by an $\ell_p$ Lipschitz hypothesis for any fixed $p\in[1,\infty]$.
\end{proposition}

\begin{proof}
The small-$K$ case of Appendix~\ref{app:pfvanilla} re-runs the sampler exactly and is independent of the Lipschitz norm, so it carries over unchanged. For the large-$K$ case, we follow the proof of Theorem~\ref{thm:vanillaupperbound} in Appendix~\ref{app:pfvanilla}, measuring every error in the normalized norm $e_i^{(r)}:=n^{-1/p}\|y_i^{(r)}-x_i\|_p$, with $n^{-1/\infty}:=1$. Two observations carry the argument over. Boundedness gives $\|v_j\|_p\le n^{1/p}B$, and a perturbation bounded by $\varepsilon$ in every coordinate changes the normalized norm by at most $\varepsilon$, both because $\|\cdot\|_p\le n^{1/p}\|\cdot\|_\infty$. The $\ell_p$ Lipschitz hypothesis gives $\|v_j^{(r)}-v_j\|_p\le L\|y_j^{(r)}-x_j\|_p$, since the two velocities share $t_j$ and $w$.

We run the construction of Appendix~\ref{app:pfvanilla} with $\eta:=\frac1{16}\mathrm e^{-L}$ in place of $\frac18\mathrm e^{-L}$, choosing the constant $c_1=c_1(B)$ of~\eqref{eq:cdchoice} large enough that $(\frac{3B}{2}+1)2^{-c_1}\le\frac1{16}$ and $2^{-c_1}\le\frac1{16}$. In the large-$K$ case we have $b\ge\log_2K+(\log_2\mathrm e)L+c_1$, which is equivalent to $K\delta\le2^{-c_1}\mathrm e^{-L}$ and thus gives $\delta\le K\delta\le\eta$, so Lemma~\ref{lem:ac0prefix} applies with $\varepsilon:=\eta$ exactly as there, and every other parameter is unchanged.

The error analysis of Appendix~\ref{app:pfvanilla} then carries over: with $s_i:=i/K$ and the same $\alpha:=\frac{B+1}{2}K\delta$ and $\beta:=\eta+(B+\frac12)K\delta$ as there,
\[
e_i^{(0)}\le Bs_i+\alpha,\qquad
e_i^{(r+1)}\le\frac{L}{K}\sum_{j=0}^{i-1}e_j^{(r)}+\beta.
\]
Indeed, the four terms of the decomposition of Appendix~\ref{app:pfvanilla} are bounded exactly as there: the prefix-sum error contributes $\eta$, the velocity differences $\frac{L}{K}\sum_{j<i}e_j^{(r)}$ by the second observation, the weight mismatch at most $BK\delta$ by the first, and the rounding residuals $K\delta/2$ since $n^{-1/p}\|\cdot\|_p\le\|\cdot\|_\infty$, and similarly for the initial bound.

Lemma~\ref{lem:picarderror} with $C:=\lceil\mathrm eL+4B\rceil$ gives $e_K^{(C)}\le\frac14+\alpha+\beta\mathrm e^L$, and as in Appendix~\ref{app:pfvanilla} the precision hypothesis bounds $\alpha+\beta\mathrm e^L\le(\frac{3B}{2}+1)2^{-c_1}+\eta\mathrm e^L\le\frac1{16}+\frac1{16}=\frac18$, so that $e_K^{(C)}\le\frac38$.

It remains to read off the decision, for which a single coordinate no longer suffices. If $w\in\mathcal{L}$, every coordinate of $x_K$ is at least $1$, so a coordinate $j$ with $(y_K^{(C)})_j\le0$ satisfies $|(y_K^{(C)})_j-(x_K)_j|\ge1$. Such coordinates form at most a fraction $\xi$ of the $n$ coordinates, where $\xi\le(e_K^{(C)})^p$ if $p<\infty$, since their contribution alone gives $n^{-1}\|y_K^{(C)}-x_K\|_p^p\ge\xi$, and $\xi=0$ if $p=\infty$. In either case $\xi\le\frac38$, because $e_K^{(C)}\le\frac38<1$ and $p\ge1$.

At least $\frac58$ of the coordinates of $y_K^{(C)}$ are therefore positive, and symmetrically at least $\frac58$ are negative if $w\notin\mathcal{L}$. The table bound computes the sign of each coordinate, and approximate counting in $\mathsf{AC}^0$~\citep{ajtai1990approximate}, applied to these signs, estimates the fraction of positive coordinates to within $\frac1{16}$. The estimate is thus at least $\frac58-\frac1{16}>\frac12$ if $w\in\mathcal L$ and at most $\frac38+\frac1{16}<\frac12$ if $w\notin\mathcal L$, so comparing it with $\frac12$ decides $\mathcal L$. This count replaces the single-coordinate comparison at the end of Appendix~\ref{app:pfvanilla} and lies in $\mathsf{AC}^0\subseteq\mathsf{D}$ for all three choices of $\mathsf{D}$. Hence $\mathcal L\in\mathsf D$.
\end{proof}

\begin{proposition}[Lower bound under $\ell_p$]\label{prop:selfcondlp} For every fixed $p\in[1,\infty]$, Theorem~\ref{thm:selfcondlowerbound} remains true with the Lipschitz requirement in condition (b) replaced by $L'$-Lipschitz continuity in $\ell_p$, for some constant $L'$ depending only on $\mathcal{L}$.
In the $\mathsf{P/poly}$ case its precision can be chosen to further satisfy~\eqref{eq:vanillaprecision}.
\end{proposition}

\begin{proof}
We reuse the model $u$ of Appendix~\ref{app:pfselfcond} as it stands: each coordinate of its condition is read by constantly many of the windows $(I_j,J_j)$, and this bounded overlap already makes $u$ Lipschitz continuous in every $\ell_p$ norm. All other conditions of Theorem~\ref{thm:selfcondlowerbound}, and acceptance with margin $1$, are unaffected, since the model and the run are unchanged.

To compute the constant, write $\Delta\ctx:=\ctx-\ctx'$ for the difference of two equal-length conditions and split it into its input part $\Delta w$ and its endpoint-prediction part $\Delta z$. Since $u$ ignores $x$ and each $\bar f_j$ is $8$-Lipschitz in the $\ell_\infty$ norm of its window, each output coordinate changes by at most $8$ times the largest change inside its window. Raising to the $p$-th power, bounding that maximum by a sum, and summing over $j$ give, for $1\le p<\infty$,
\[
\|u(x,t;\ctx)-u(x,t;\ctx')\|_p^p
\le8^p\sum_{j=1}^{n}\Bigl(\sum_{k\in I_j}|\Delta w_k|^p+\sum_{k\in J_j}|\Delta z_k|^p\Bigr).
\]
In this double sum, each input coordinate appears in $I_j$ for the $\kappa$ coordinates $j$ of its block, and each endpoint-prediction coordinate appears in $J_j$ for at most $3\kappa$ coordinates $j$, those of its own block and of the two neighboring blocks. The double sum is therefore at most $\kappa\|\Delta w\|_p^p+3\kappa\|\Delta z\|_p^p\le3\kappa\|\Delta\ctx\|_p^p$, where the last step uses $\|\Delta\ctx\|_p^p=\|\Delta w\|_p^p+\|\Delta z\|_p^p$, so $u$ is $8(3\kappa)^{1/p}$-Lipschitz in $\ell_p$. The same argument with the maximum kept in place of the sum handles $p=\infty$, recovering the constant $8$ of Theorem~\ref{thm:selfcondlowerbound}. As $8(3\kappa)^{1/p}\le24\kappa$ for every $p\in[1,\infty]$, we may take $L':=24\kappa$, which depends only on the alphabet of Lemma~\ref{lem:local} and hence only on $\mathcal{L}$.

In the $\mathsf{P/poly}$ case, the parameter choice of Appendix~\ref{app:pfselfcond} sets $b(m):=\max\{3,\lceil2\log_2K(m)\rceil\}$, which does not depend on the Lipschitz constant and therefore still satisfies~\eqref{eq:vanillaprecision}.
\end{proof}

\subsection{Runge--Kutta Samplers}
\label{app:rk}

The upper bounds above are stated for vanilla flow matching with the Euler method. We extend them to any fixed explicit Runge--Kutta method. Let $\mathsf R$ be such a method with $q\in\mathbb{N}^+$ stages per step, whose tableau is specified by nodes $\tilde\omega_\ell\in[0,1]$, weights $\tilde\lambda_\ell\in\R$, and coefficients $\tilde a_{\ell,k}\in\R$ for $1\le k<\ell\le q$~\citep{hairer1993solving}. Algorithm~\ref{alg:vanillark} gives the resulting sampler, $\RKFM$, using the rounded tableau entries $\omega_\ell:=\round_b(\tilde\omega_\ell)$, $\lambda_\ell:=\round_b(\tilde\lambda_\ell)$, and $a_{\ell,k}:=\round_b(\tilde a_{\ell,k})$. The Euler method is the special case $q=1$, while Heun's method, another common flow-matching sampler~\citep{karras2022elucidating,lee2024improving}, has $q=2$.

As in Algorithm~\ref{alg:parallelpicard}, we simulate $\RKFM$ by a parallel Picard iteration for $\mathsf R$. The analysis of Appendices~\ref{app:pfpicard} and~\ref{app:pfvanilla} then carries over with $L$ scaled by a constant depending only on the tableau, at the cost of $qK$ parallel model calls per round.

\begin{algorithm}[tbh]
\caption{Runge--Kutta Flow Matching}
\label{alg:vanillark}
\begin{algorithmic}[1]
\Statex \textbf{Fixed:} an explicit Runge--Kutta method $\mathsf R$ with constant stage count $q\in \mathbb{N}^+$ and rounded tableau entries $\omega_\ell\in F_b$, $\lambda_\ell\in F_b$, and $a_{\ell,k}\in F_b$
\Function{RKFM}{$w;u,K$}
    \State $t_i\leftarrow\round_b(i/K)$ for $i=0,\ldots,K$
    \State $t_{i,\ell}\leftarrow\round_b(t_i+\omega_\ell h_i)$ for $i=0,\ldots,K-1$; $\ell=1,\ldots,q$
    \State Sample $\tilde{x}_0\sim p_{\text{prior}}$; $x_0\leftarrow\round_b(\tilde{x}_0)$
    \For{$i=0,\ldots,K-1$}
        \For{$\ell=1,\ldots,q$}
            \State $Y_{i,\ell}\leftarrow\round_b\bigl(x_i+h_i{\textstyle\sum_{k<\ell}}a_{\ell,k}v_{i,k}\bigr)$
            \State $v_{i,\ell}\leftarrow u(Y_{i,\ell},t_{i,\ell};w)$
        \EndFor
        \State $x_{i+1}\leftarrow\round_b\bigl(x_i+h_i{\textstyle\sum_{\ell=1}^{q}}\lambda_\ell v_{i,\ell}\bigr)$
    \EndFor
    \State \Return $x_K$
\EndFunction
\end{algorithmic}
\end{algorithm}

\begin{algorithm}[tbh]
\caption{Parallel Runge--Kutta Picard Inference}
\label{alg:parallelrk}
\begin{algorithmic}[1]
\Statex \textbf{Fixed:} an explicit Runge--Kutta method $\mathsf R$ with constant stage count $q\in \mathbb{N}^+$ and rounded tableau entries $\omega_\ell\in F_b$, $\lambda_\ell\in F_b$, and $a_{\ell,k}\in F_b$
\Function{ParallelRK}{$w;u,K,C$}
    \State $t_i\leftarrow\round_b(i/K)$ for $i=0,\ldots,K$
    \State $t_{i,\ell}\leftarrow\round_b(t_i+\omega_\ell h_i)$ for $i=0,\ldots,K-1$, $\ell=1,\ldots,q$
    \State Sample $\tilde{x}_0\sim p_{\text{prior}}$; $x_0\leftarrow\round_b(\tilde{x}_0)$
    \State $Y^{(0)}_{i,\ell}\leftarrow x_0$ for all $i,\ell$; $X^{(0)}\leftarrow x_0$
    \For{$r=0,\ldots,C-1$}
        \State \textbf{in parallel for} all $i,\ell$: \ $v^{(r)}_{i,\ell}\leftarrow u(Y^{(r)}_{i,\ell},t_{i,\ell};w)$
        \State $S^{(r)}_i\leftarrow\sum_{j=0}^{i-1}\sum_{k=1}^{q}\round_b(\lambda_k v^{(r)}_{j,k})$ for $i=0,\ldots,K$
        \State $T^{(r)}_{i,\ell}\leftarrow\sum_{k=1}^{\ell-1}\round_b(a_{\ell,k}v^{(r)}_{i,k})$ for all $i,\ell$
        \State \textbf{in parallel for} all $i,\ell$: \ $Y^{(r+1)}_{i,\ell}\leftarrow\round_b\bigl(x_0+\bigl(S^{(r)}_i+T^{(r)}_{i,\ell}\bigr)/K\bigr)$
        \State $X^{(r+1)}\leftarrow\round_b\bigl(x_0+S^{(r)}_K/K\bigr)$
    \EndFor
    \State \Return $X^{(C)}$
\EndFunction
\end{algorithmic}
\end{algorithm}

\begin{proposition}[Runge--Kutta samplers]\label{prop:rk}
Let $q$ be a constant integer and $\mathsf R$ be any fixed explicit Runge--Kutta method with $q$ stages, whose nodes lie in $[0,1]$. There is a constant $\gamma_{\mathsf R}\ge1$, depending only on $\mathsf R$, such that the following hold with $\VanillaFM$ replaced by $\RKFM$:
\begin{itemize}[leftmargin=*]
    \item Lemma~\ref{lem:picard} remains true once the precision condition~\eqref{eq:picardprecision} is replaced by
    \begin{equation}\label{eq:rkprecision}
        b(m)\ge \log_2 K(m)+\gamma_{\mathsf R}(\log_2\mathrm e)L(m)+c_0\qquad\text{for every }m,
    \end{equation}
    where the constant $c_0=c_0(B,\mathsf R)$ now also depends on $\mathsf R$, and the bound $K(m)$ on the parallel calls per round is replaced by $qK(m)$.
    \item Corollary~\ref{cor:vanillapolylogl} remains true once the lower bound in~\eqref{eq:vanillapolyloglprecision} is replaced by~\eqref{eq:rkprecision}.
    \item Theorem~\ref{thm:vanillaupperbound} remains true.
\end{itemize}
\end{proposition}

\begin{proof}
We simulate the Runge--Kutta sampler of Algorithm~\ref{alg:vanillark} by the parallel Picard iteration of Algorithm~\ref{alg:parallelrk}, following the same strategy as the Euler case: write the $\RKFM$ trajectory as a fixed-point system, solve it by Algorithm~\ref{alg:parallelrk}, and bound the error with Lemma~\ref{lem:picarderror}. The only new feature is that each step has $q$ stages, so the system has $qK$ positions instead of $K$. Each stage depends on earlier stages in the same step as well as on all stages of previous steps.

\paragraph{The fixed-point system.}
Fix $w$ and couple the two algorithms through a common initial $x_0$. Let $Y^*_{i,\ell}$ and $x^*_i$ denote the true stages and states of Algorithm~\ref{alg:vanillark}, and write $v^*_{i,\ell}:=u(Y^*_{i,\ell},t_{i,\ell};w)$. Since $\omega_\ell\in[0,1]$ and $\round_b$ is nondecreasing and fixes $F_b$, every stage time satisfies $t_{i,\ell}\in[t_i,t_{i+1}]\subseteq[0,1]$, so the boundedness and Lipschitz hypotheses apply at every stage. Each $\round_b$ displayed in Algorithms~\ref{alg:vanillark} and~\ref{alg:parallelrk} evaluates its argument over the reals and rounds the result once. Unrolling the rounded updates gives the fixed-point equations
\begin{equation}
Y^*_{i,\ell}=x_0+\sum_{j<i}h_j\sum_{k=1}^{q}\lambda_k v^*_{j,k}
    +h_i\sum_{k<\ell}a_{\ell,k}v^*_{i,k}+\rho_{i,\ell},
    \qquad\|\rho_{i,\ell}\|_\infty\le\frac{K\delta}{2},
\label{eq:rk-stage}
\end{equation}
and
\begin{equation}
x^*_K=x_0+\sum_{j=0}^{K-1}h_j\sum_{k=1}^{q}\lambda_k v^*_{j,k}+\rho_K,
    \qquad\|\rho_K\|_\infty\le\frac{K\delta}{2}.
\label{eq:rk-endpoint}
\end{equation}

We now index all positions in dependency order. Put index $0$ for $x_0$, index $\nu = qi+\ell$ for stage $Y^*_{i,\ell}$ with $0\le i\le K-1$ and $1\le \ell\le q$, and index $M:=qK+1$ for the endpoint $x_K^*$. For each $\nu\ge1$, write $y_\nu^*$ for the corresponding quantity (the stage $Y^*_{i,\ell}$ when $\nu=qi+\ell$, and the endpoint $x_K^*$ when $\nu=M$), and for $1\le\nu\le qK$ write $v_\nu^*$ for the corresponding stage velocity $v^*_{i,\ell}$. Then the fixed-point equations for all positions can be written uniformly as
\[
y_\nu^* = x_0 + \sum_{1\le\nu'<\nu} \Lambda_{\nu,\nu'}\, v_{\nu'}^* + \rho_\nu,
\qquad \|\rho_\nu\|_\infty\le \frac{K\delta}{2},
\]
where, by \eqref{eq:rk-stage} and \eqref{eq:rk-endpoint}, each coefficient $\Lambda_{\nu,\nu'}$ is either $0$, some $h_j\lambda_k$, or some $h_i a_{\ell,k}$. Put $\mu:=3\max_{\ell,k}\{1,|\tilde\lambda_\ell|,|\tilde a_{\ell,k}|\}$, a constant depending only on the tableau, so that every rounded tableau entry has magnitude at most $\mu/2$. The precision condition~\eqref{eq:rkprecision} gives $\delta\le1/K$, so $h_j\le2/K$, and hence every coefficient satisfies
\[
|\Lambda_{\nu,\nu'}|\le \frac{\mu}{K}.
\]

\paragraph{Error analysis.}
For each position $\nu$, define $e_\nu^{(r)}$ as follows: if $\nu=qi+\ell$ is a stage, then $e_\nu^{(r)}:=\|Y_{i,\ell}^{(r)}-Y^*_{i,\ell}\|_\infty$; if $\nu=M$ is the endpoint, then $e_M^{(r)}:=\|X^{(r)}-x^*_K\|_\infty$; and $e_0^{(r)}:=0$, since both runs share $x_0$. Put $s_\nu:=\nu/M$, $\gamma_{\mathsf R}:=2q\mu$, $B_{\mathsf R}:=2q\mu B$, and $A_{\mathsf R}:=q(\mu B+1)+1$, constants depending only on $B$ and $\mathsf R$. Since the iteration starts every position at $x_0$, bounding the at most $\nu$ velocities in the uniform equation by $B$ gives
\[
e_\nu^{(0)}=\|x_0-y_\nu^*\|_\infty
\le\sum_{1\le\nu'<\nu}|\Lambda_{\nu,\nu'}|\,\|v_{\nu'}^*\|_\infty+\|\rho_\nu\|_\infty
\le\nu\cdot\frac{\mu B}{K}+\frac{K\delta}{2}
\le B_{\mathsf R}s_\nu+A_{\mathsf R}K\delta,
\]
the last step because $\nu/K=s_\nu M/K\le2qs_\nu$.

For the inductive step, write $y^{(r)}_\nu$ and $v^{(r)}_\nu$ for the iterate and velocity of Algorithm~\ref{alg:parallelrk} at position $\nu$, in the indexing above, and let $\Lambda^\circ_{\nu,\nu'}$ be the nominal coefficients, obtained from $\Lambda_{\nu,\nu'}$ by replacing the step $h_j$ or $h_i$ with $1/K$. Each rounded tableau entry has magnitude at most $\mu/2$, so $|\Lambda^\circ_{\nu,\nu'}|\le\mu/(2K)$, and $|1/K-h_j|\le\delta$ bounds the mismatch by $|\Lambda^\circ_{\nu,\nu'}-\Lambda_{\nu,\nu'}|\le\mu\delta/2$. Expanding the update of Algorithm~\ref{alg:parallelrk} in the same indexing,
\[
y^{(r+1)}_\nu=x_0+\sum_{1\le\nu'<\nu}\Lambda^\circ_{\nu,\nu'}v^{(r)}_{\nu'}+\rho^{(r)}_\nu,
\qquad
\|\rho^{(r)}_\nu\|_\infty\le\frac{(2q+1)\delta}{2},
\]
where $\rho^{(r)}_\nu$ collects the roundings of the at most $qK$ products in $S^{(r)}_i$ and the at most $q$ products in $T^{(r)}_{i,\ell}$, each of size at most $\delta/2$ and divided by $K$, together with the outer rounding. Subtracting the uniform equation from this expansion and taking norms,
\[
e^{(r+1)}_\nu
\le\sum_{1\le\nu'<\nu}|\Lambda^\circ_{\nu,\nu'}|\,\|v^{(r)}_{\nu'}-v^*_{\nu'}\|_\infty
+\sum_{1\le\nu'<\nu}|\Lambda^\circ_{\nu,\nu'}-\Lambda_{\nu,\nu'}|\,\|v^*_{\nu'}\|_\infty
+\|\rho^{(r)}_\nu\|_\infty+\|\rho_\nu\|_\infty.
\]
In the first sum, the two velocities at position $\nu'$ share the stage time and the condition, so Lipschitz continuity gives $\|v^{(r)}_{\nu'}-v^*_{\nu'}\|_\infty\le Le^{(r)}_{\nu'}$, and $|\Lambda^\circ_{\nu,\nu'}| \le \mu/(2K)\le\gamma_{\mathsf R}/M$ because $M\le2qK$. The second sum has fewer than $M\le2qK$ terms, each at most $\frac{\mu\delta}{2}B$, so it is at most $q\mu BK\delta$. The two residuals add at most $(2q+1)\delta/2+K\delta/2\le(q+1)K\delta$. Hence, by the choice of $A_{\mathsf R}$,
\[
e_\nu^{(r+1)}\le \frac{\gamma_{\mathsf R}L}{M}\sum_{\nu'<\nu}e_{\nu'}^{(r)}+A_{\mathsf R}K\delta.
\]
This is exactly the recurrence of Lemma~\ref{lem:picarderror}, with its parameters set to $M$, $B_{\mathsf R}$, $\gamma_{\mathsf R}L$, $\alpha=A_{\mathsf R}K\delta$, and $\beta=A_{\mathsf R}K\delta$. Applying that lemma with $C:=\lceil\mathrm e\gamma_{\mathsf R}L+4B_{\mathsf R}\rceil$ gives
\[
e_M^{(C)}\le \frac14 + A_{\mathsf R}K\delta\bigl(1+\mathrm e^{\gamma_{\mathsf R}L}\bigr).
\]
The precision condition~\eqref{eq:rkprecision} implies $K\delta\le2^{-c_0}\mathrm e^{-\gamma_{\mathsf R}L}$.  Choosing $c_0=c_0(B,\mathsf R)$ with $2A_{\mathsf R}2^{-c_0}\le\frac14$ makes $e_M^{(C)}\le\frac12$. Thus $u$ accepts $\mathcal{L}$ with margin $\frac12$ under Algorithm~\ref{alg:parallelrk} whenever it accepts $\mathcal{L}$ with margin $1$ under $\RKFM$, and the algorithm runs for $C=O(L)$ rounds with $qK$ parallel model calls per round. This proves the first item.

\paragraph{Expressiveness upper bounds.}
For the analogue of Theorem~\ref{thm:vanillaupperbound}, where $b=O(\log m)$, we split as in Appendix~\ref{app:pfvanilla}, now at $b\ge\log_2K+\gamma_{\mathsf R}(\log_2\mathrm e)L+c_1$. The small-$K$ case again re-runs the sampler directly, now making $q$ stage calls per step, and lies in $\mathsf D$ by the same argument as in Appendix~\ref{app:pfvanilla}. In the large-$K$ case, we replace the normalized prefix sums by the approximate ones of Lemma~\ref{lem:ac0prefix}, applied with $N:=qK$ and accuracy $\varepsilon:=\eta/q$ to the rounded products $\round_b(\lambda_kv^{(r)}_{j,k})$, which are bounded by a constant, and rescaled by the integer $q$, as in Appendix~\ref{app:pfvanilla}. All remaining updates depend on only $O(1)$ scalars of $O(\log m)$ bits. With the time grid and stage times hardwired, they can be computed exactly by the table bound. The approximation adds $\eta$ to the additive constant of the recurrence, so Lemma~\ref{lem:picarderror}, now with $\beta=A_{\mathsf R}K\delta+\eta$, gives
\[
e_M^{(C)}\le\frac14+A_{\mathsf R}K\delta+\bigl(A_{\mathsf R}K\delta+\eta\bigr)\mathrm e^{\gamma_{\mathsf R}L}.
\]
Taking $\eta:=\tfrac18\mathrm e^{-\gamma_{\mathsf R}L}$ and choosing $c_1=c_1(B,\mathsf R)$ with $2A_{\mathsf R}2^{-c_1}\le\frac18$ gives $e_M^{(C)}\le\frac12$. Since $A_{\mathsf R}\ge q$, this choice also gives $\delta\le K\delta\le2^{-c_1}\mathrm e^{-\gamma_{\mathsf R}L}\le\eta/q$, as Lemma~\ref{lem:ac0prefix} requires. The size and depth argument of Appendix~\ref{app:pfvanilla} then places $\mathcal{L}$ in $\mathsf{D}$, which proves the third item.

For the analogue of Corollary~\ref{cor:vanillapolylogl}, where $b=\poly(m)$, arithmetic with the hardwired rounded tableau entries and the exact prefix sums lies in $\mathsf{TC}^0$~\citep{vollmer1999introduction,hesse2002uniform}, so the depth is $O(L)$ as in Corollary~\ref{cor:vanillapolylogl}. This proves the second item and hence the proposition.
\end{proof}

\subsection{Matching the Circuit Depth}
\label{app:depthmatch}

Corollary~\ref{cor:vanillapolylogl} shows that any run of $\VanillaFM$ can be compressed into $O(L)$ sequential rounds. Here we show that these rounds match circuit depth up to constant factors: given any polynomial-size target majority circuit of depth $d$, $\VanillaFM$ can simulate it using a model with Lipschitz constant $O(d)$, while each model evaluation can be computed by a constant-depth circuit.
The construction proceeds in three phases, determined by the current time $t$: resetting the sampler state, evaluating one layer of the target circuit per sampling step, and finally amplifying the resulting verdict to satisfy the acceptance margin.

\begin{proposition}[Encoding majority circuits into $\VanillaFM$]\label{prop:depthmatch}
Let $\{d(m)\}_{m=1}^\infty$ be a sequence of positive integers, and $\mathcal{L}$ be a language decided by polynomial-size majority circuits of depth $d(m)$. Then there are $K(m)=\Theta(d(m))$, $n(m)=\poly(m)$, and $b(m)=\poly(m)$ satisfying~\eqref{eq:picardprecision} with $B=8$ and $L(m):=8K(m)$, and a model $u$, computed by $\mathsf{TC}^0$ circuits, such that $u$ is $8$-bounded and $8K(m)$-Lipschitz continuous and accepts $\mathcal{L}$ with margin $1$ under $\VanillaFM$.
\end{proposition}

\begin{proof}
Every gate of the circuit is a threshold gate with weights in $\{-1,1\}$ and an integer threshold. An AND, OR, or MAJORITY gate of fan-in $r$ has unit weights and threshold $r$, $1$, or $\lceil r/2\rceil$, respectively, and a NOT gate has weight $-1$ and threshold $0$.

We construct the model $u$ explicitly. Recode Boolean values by $0\mapsto-1$ and $1\mapsto1$. A gate keeps its $\pm1$ weights $q_1,\ldots,q_r$, and its threshold $\theta_{\text{old}}$ becomes the integer $\theta:=2\theta_{\text{old}}-\sum_{j=1}^rq_j$. Index the state coordinates by the gates, let $n$ be the number of gates, and write $*$ for the output gate. The run stores the output of each gate $\gamma$ in the state coordinate $x_\gamma$. The layer of a gate $\gamma$ is its depth $\ell(\gamma)\in\{1,\ldots,d\}$. Every gate reads only primary inputs and gates of earlier layers.

Let $K$ be the least power of two with $K\ge8(d+1)$, so that $d+1\le K/8$ and $K\le16(d+1)$. We will choose $b\ge\log_2 K$ at the end of the proof. Since $K$ is a power of two and $\round_b$ fixes $F_b$, the time grid is exact: $t_i=i/K$ and $h_i=1/K$.

\paragraph{Encodings and gate extensions.}
Write $\clip(y;c):=\max\{-c,\min\{c,y\}\}$ for the truncation of a scalar $y$ to $[-c,c]$. Gate outputs are stored at scale $1/K$: one Euler step with velocity $\pm1$ writes exactly $\pm1/K$ into a coordinate. A state coordinate $x_\gamma$ is therefore decoded by $\clip(Kx_\gamma;1)$, which maps the stored encodings $\pm1/K$ to $\pm1$. An input bit $w_j\in\{0,1\}$ is decoded by $\clip(2w_j-1;1)$, which maps $0$ to $-1$ and $1$ to $1$.

Consider a gate with weights $q_1,\ldots,q_r\in\{-1,1\}$ and integer threshold $\theta$, and set $k:=\lceil(\theta+r)/2\rceil$. On a vertex $a\in\{-1,1\}^r$, the gate fires if and only if $\sum_jq_ja_j=2\,|\{j:q_ja_j=1\}|-r$ reaches $\theta$, that is, if and only if at least $k$ of the $\pm1$ values $q_ja_j$ equal $1$. The recoded gate value is therefore the $k$-th largest of the $q_ja_j$ when $1\le k\le r$, and a constant otherwise. Accordingly, extend the gate to $[-1,1]^r$ by letting $\hat g(y)$ be the $k$-th largest of $q_1y_1,\ldots,q_ry_r$, with $\hat g\equiv1$ if $k\le0$ and $\hat g\equiv-1$ if $k>r$. Then $\hat g$ equals the recoded gate value on the vertices, and every value of $\hat g$ is one of the $q_jy_j$ or $\pm1$, so $\hat g$ maps $F_b^r$ into $F_b$ and the model needs no rounding.

\paragraph{The model.}
Let $G_\gamma(w,x)$ denote $\hat g$ of gate $\gamma$ applied to its decoded inputs. We define $u$ coordinatewise by
\begin{equation}
\label{eq:depthmatch-model}
\bigl(u(x,t;w)\bigr)_\gamma :=
\begin{cases}
-\clip(Kx_\gamma;8) & t=i/K\text{ with }0\le i<K/4,\\[2mm]
G_\gamma(w,x) & t=(K/4+\ell(\gamma)-1)/K,\\[2mm]
8\clip(Kx_*;1) & t=i/K\text{ with }K/4+d\le i<K,\\[2mm]
0 & \text{otherwise},
\end{cases}
\end{equation}
on valid inputs with parameters $(n(m),b(m),m)$, and $u(x,t;w):=\mathbf0^n$ on all other encoded inputs.

\paragraph{The run.}
We now show that $u$ accepts $\mathcal L$ with margin $1$ under $\VanillaFM$. In the first phase, one step maps $x_\gamma$ to $0$ if $|x_\gamma|\le8/K$ and otherwise moves it toward $0$ by $8/K$, so from $\|x_0\|_\infty\le1$ the state reaches $\mathbf0$ within $K/8+1\le K/4$ steps and stays there.

The second phase evaluates one circuit layer per step. A coordinate has zero velocity at every second-phase step except its own layer step, so a value once written persists. By construction, step $K/4+\ell-1$ writes the value of every layer-$\ell$ gate, encoded as $\pm1/K$, into its coordinate.

The third phase amplifies the verdict: after step $K/4+d-1$ every gate value is written, so the decoded output $\clip(Kx_*;1)$ is the verdict $\pm1$. The shared velocity $8\clip(Kx_*;1)$ stays pinned at $\pm8$ for the $K-K/4-d\ge\frac58K$ remaining steps. Every final coordinate is then at least $8\cdot\frac58-\frac1K\ge1$ on acceptance and at most $-1$ on rejection, so $u$ accepts $\mathcal{L}$ with margin $1$.

The run is exact: every velocity lies in $\{0,\pm1,\pm8\}$ or equals $-Kx_\gamma$, so every Euler increment is $0$, $\pm1/K$, $\pm8/K$, or $-x_\gamma$, all in $F_b$. 

\paragraph{Lipschitz constant and circuit complexity.}
We first check that $\hat g$ is $1$-Lipschitz. Suppose that $\|y-y'\|_\infty\le\varepsilon$, so that $q_jy'_j\le q_jy_j+\varepsilon$ for every $j$. The $k$-th largest is nondecreasing in every coordinate and grows by exactly $\varepsilon$ when every coordinate grows by $\varepsilon$. Hence $\hat g(y')\le\hat g(y)+\varepsilon$, and exchanging $y$ and $y'$ gives the reverse bound.

The model is $8$-bounded. For fixed $t$, the velocity is $K$-Lipschitz in the reset phase and $8K$-Lipschitz in the amplification phase, as $\clip$ is $1$-Lipschitz. In the layer phase, the velocity is $K$-Lipschitz: the state decoder $\clip(Kx_\gamma;1)$ is $K$-Lipschitz, the input decoder $\clip(2w_j-1;1)$ is $2$-Lipschitz, and $\hat g$ is $1$-Lipschitz. Since the Lipschitz condition of~\Sref{sec:fmsetup} constrains each $t$ separately, $u$ is $8K$-Lipschitz continuous.

It remains to check that $u$ is computed by $\mathsf{TC}^0$ circuits. Every scalar involved occupies polynomially many bits in the input length. Comparisons, additions, and multiplexing lie in $\mathsf{AC}^0$~\citep{vollmer1999introduction}, and multiplication by $K$ is a bit shift, as $K$ is a power of two.

The only nontrivial computation is $\hat g$, the $k$-th largest of the $r$ scalars $q_jy_j$. This value is the largest candidate $c\in\{q_1y_1,\ldots,q_ry_r\}$ with $|\{j:q_jy_j\ge c\}|\ge k$. The circuit compares all pairs of the scalars in parallel, tests each candidate by one threshold gate over its $r$ comparison bits, and selects the largest passing candidate by multiplexing. These blocks have constant depth and polynomial size, so $u$ is computed by $\mathsf{TC}^0$ circuits.

\paragraph{Parameters.}
Finally, we fix the parameters. The quantities $n(m)$ and $K(m)\le16(d(m)+1)$ are $\poly(m)$, and
\[
b(m):=\bigl\lceil\log_2K(m)+8(\log_2\mathrm e)K(m)+c_0(8)\bigr\rceil=\poly(m)
\]
satisfies $b\ge\log_2K$ and~\eqref{eq:picardprecision} with $B=8$ and $L(m)=8K(m)$.
\end{proof}

\textbf{Remark.} Proposition~\ref{prop:depthmatch} and Corollary~\ref{cor:vanillapolylogl} show that the $O(L)$ rounds of Lemma~\ref{lem:picard} cannot be improved in general without collapsing the depth hierarchy of majority circuits. Suppose that some inference algorithm simulates $\VanillaFM$ in $f(L)=o(L)$ rounds, each consisting of polynomially many parallel calls to $u$ and $\mathsf{TC}^0$ arithmetic. The model of Proposition~\ref{prop:depthmatch} has $L(m)=\Theta(d(m))$ and is computed by majority circuits of constant depth, so polynomial-size majority circuits of depth $d(m)$ would collapse to polynomial-size majority circuits of depth $O(f(L(m)))=o(d(m))$, which is unlikely. For example, the depth hierarchy of majority circuits within $\mathsf{TC}^0$ would collapse to a constant depth, whereas it is widely conjectured to be infinite~\citep{hajnal1993threshold}.

\clearpage

\section{Formalization and Circuit Complexity of the Tasks in \Sref{sec:theory_aligned_tasks}}
\label{app:task_complexity}

This appendix formalizes the three synthetic tasks introduced in \Sref{sec:theory_aligned_tasks} and establishes the circuit-complexity claims used there. Specifically, we show that line tracing is \(\mathsf{NC}^1\)-complete, while long multiplication and long text rendering are \(\mathsf{TC}^0\)-complete, matching the separation barriers between $\VanillaFM$ and $\SelfCondFM$ in \Sref{sec:vanillabarrier} for \(\mathsf{TC}^0\) and \(\mathsf{AC}^0\) backbones, respectively. As in Appendix~\ref{app:proofconventions}, all circuit classes considered below are non-uniform.

\subsection{Task Descriptions}

We first describe the three task families before giving their formal definitions and complexity analyses. The tasks are illustrated in Figure~\ref{fig:theory_aligned_tasks}.

\paragraph{Line tracing.} The input is an image containing five black polylines that traverse \(m\) permutation stages from left to right, with one polyline starting in each of the five rows. At stage \(k\), a permutation \(g_k\in A_5\) reroutes the polyline entering row \(i\) to row \(g_k(i)\), where \(A_5\) denotes the group of even permutations on five elements. The target is a colored version of the input in which each polyline is assigned a fixed color determined by its starting row. Whenever two polylines cross within a stage, we resolve occlusion by rendering the polyline entering the stage from the higher row on top.

\paragraph{Long multiplication.} The input is an image containing a multiplication problem of two positive integers \(A\) and \(B\), each with at most \(m\) digits and no leading zeros. The two integers are rendered with their digits aligned by column and their least significant digits in the rightmost column. The target is a complete step-by-step long-multiplication worksheet: writing \(B=\sum_i B_i10^i\), it displays the shifted partial products \(AB_i10^i\), the digitwise carries, their vertical summation, and the final product \(AB\). The worksheet follows a digit-aligned grid layout, with each digit rendered in a separate cell and carry digits placed at the upper left of the corresponding cells. Additional line segments indicate the summation of the partial products.

\paragraph{Long text rendering.} The input is a text passage of \(m\) words sampled from the TinyStories~\citep{eldan2023tinystories} dataset, given to the model as a text condition. The target is a monospaced rendering of the passage on a canvas of fixed usable width \(W\). The first word begins at the left margin; each subsequent word is placed on the current line if it fits after a single interword space, and otherwise starts a new line. At each paragraph break, one line is left blank, and the next paragraph begins at the left margin.

\subsection{Preliminaries for the Complexity Analysis}

Next, we formalize these three image-space tasks and establish their circuit complexities through their underlying symbolic functions.

\paragraph{Completeness for functions.} Each of the three tasks produces a long output rather than a single bit, so we
phrase completeness through the output bits. For a function \(f\) whose output
has polynomial length, put
\[
    \operatorname{Bit}(f)
    :=\{(x,j):\text{the \(j\)-th bit of \(f(x)\) is \(1\)}\}.
\]

As in Theorem~\ref{thm:vanillaupperbound}, we write \(\mathsf{D}\) for a
circuit class. We call \(f\)
\emph{\(\mathsf{D}\)-complete} when \(\operatorname{Bit}(f)\in\mathsf{D}\) and
every language in \(\mathsf{D}\) is decided by \(\mathsf{AC}^0\) circuits with
oracle gates for \(\operatorname{Bit}(f)\), that is, when \(f\) is complete for
\(\mathsf{D}\) under \(\mathsf{AC}^0\) Turing reductions. The three
propositions below use this notion with \(\mathsf{D}=\mathsf{NC}^1\) for line
tracing and \(\mathsf{D}=\mathsf{TC}^0\) for the other two tasks.

\paragraph{Symbolic and image-space tasks.}
Line tracing and long multiplication map images to images, whereas long text rendering maps a word sequence with paragraph boundaries to an image. We call these three image-output mappings \emph{image-space tasks}, distinguishing them from their underlying symbolic functions. For each task, we establish the circuit complexity of its symbolic function, then give bidirectional \(\mathsf{AC}^0\) Turing reductions between that function and the corresponding image-space task.

\paragraph{Asymptotic interpretation.} The propositions below concern asymptotic task families as \(m\) grows. In practice, our experiments in \Sref{sec:theory_aligned_tasks} instantiate these families over a finite range of problem sizes \(m\)
and
report a single score for each task.
A line-tracing instance with \(m=50\)
contains every shorter prefix length as a subproblem. Long multiplication
draws operands of five to seven digits, and long text rendering draws
passages of varying word count \(m\). Circuit complexity is a property of the
family, so the propositions describe the computation that the tasks demand as
\(m\) grows.

\subsection{Line Tracing}
\label{app:line_tracing_complexity}

A line-tracing instance, as illustrated in the left column of
Figure~\ref{fig:theory_aligned_tasks}, can be viewed as a sequence
\((g_1,\ldots,g_m)\in A_5^m\), where \(A_5\) denotes the alternating group on
five elements. Here \(g_k\) reroutes the polyline entering row \(i\) of stage
\(k\) to row \(g_k(i)\). With \(e\) the identity of \(A_5\), set
\[
    \pi_0=e,
    \qquad
    \pi_k=g_k\circ\pi_{k-1}
    \quad (1\le k\le m),
\]
so that the polyline starting in row \(i\) occupies row \(\pi_k(i)\) after
\(k\) stages. Each polyline keeps the color \(c_i\) of its starting row \(i\),
so row \(j\) carries the color \(c_{\pi_k^{-1}(j)}\) after \(k\) stages. The
colored target is exactly the sequence of prefix permutations
\(\operatorname{Trace}_{A_5}(g_1,\ldots,g_m):=(\pi_0,\ldots,\pi_m)\).

\begin{proposition}[Line tracing]
\label{prop:line_tracing_complexity}
The line-tracing task, mapping a black-and-white input image to its colored target, is
\(\mathsf{NC}^1\)-complete under \(\mathsf{AC}^0\) Turing reductions.
\end{proposition}

\begin{proof}
We first consider the symbolic function \(\operatorname{Trace}_{A_5}\).
Since \(A_5\) is fixed, multiplying two encoded permutations is a
constant-size lookup, so a balanced tree computes any \(\pi_k\) with depth
\(O(\log m)\). Computing all \(m\) prefix products independently still uses
only polynomial size, so the output-bit relation of
\(\operatorname{Trace}_{A_5}\) is in \(\mathsf{NC}^1\).

To establish \(\mathsf{NC}^1\)-hardness, we use Barrington's result that the
word problem for every fixed non-solvable group is \(\mathsf{NC}^1\)-complete
under \(\mathsf{AC}^0\) reductions~\citep{barrington1986bounded}. Since \(A_5\)
is non-solvable, that result applies to the word problem of deciding whether
\(g_m\circ\cdots\circ g_1=e\). An encoded group element occupies a constant
number of bits, so a constant number of oracle queries returns the final
component \(\pi_m\) of the trace, and an \(\mathsf{AC}^0\) test compares it with
\(e\). Hence \(\operatorname{Trace}_{A_5}\) is \(\mathsf{NC}^1\)-complete.

Parallel local lookups decode \((g_k)_{k=1}^m\) from the black-and-white stage templates
and \((\pi_k)_{k=0}^m\) from the colored target, and conversely render both sequences as
images. These conversions lie in \(\mathsf{AC}^0\), so the
classification transfers to the line-tracing task.
\end{proof}

\subsection{Long Multiplication}
\label{app:long_multiplication_complexity}

We write \(\operatorname{LMul}(A,B)\) for the complete decimal worksheet of two
positive integers \(A\) and \(B\), shown in the middle column of
Figure~\ref{fig:theory_aligned_tasks}. For
\(B=\sum_i B_i10^i\), put \(P_i=AB_i10^i\) and \(S_i=\sum_{j=0}^i P_j\). The
worksheet contains the shifted partial products \(P_i\), the successive partial
sums \(S_i\) for \(i\ge1\), the digitwise carries used in these multiplications
and additions, and the final product \(AB\), all in decimal.

\begin{proposition}[Long multiplication]
\label{prop:long_multiplication_complexity}
The long-multiplication task, mapping an operand image to its rendered
worksheet, is \(\mathsf{TC}^0\)-complete under \(\mathsf{AC}^0\) Turing
reductions.
\end{proposition}

\begin{proof}
We first analyze the symbolic worksheet function \(\operatorname{LMul}\).
To show that \(\operatorname{LMul}\in\mathsf{TC}^0\), recall that fixed-base
conversion, iterated addition, multiplication, division, and remainder are
computable by constant-depth threshold
circuits~\citep{chandra1984constant,hesse2002uniform}. These operations compute every \(P_i\), every \(S_i\), and every decimal output digit in
parallel, and each displayed carry is obtained from the corresponding
fixed-base quotient and remainder. The number of worksheet entries and the
length of each entry are polynomial in the input length, so the complete
worksheet is in \(\mathsf{TC}^0\).

We now prove \(\mathsf{TC}^0\)-hardness. For
\(x=(x_0,\ldots,x_{N-1})\in\{0,1\}^N\), the exact-counting function is
\[
    \operatorname{Count}_N(x):=\sum_{i=0}^{N-1}x_i,
\]
with its output written in binary. This function is
\(\mathsf{TC}^0\)-complete under \(\mathsf{AC}^0\) Turing
reductions~\citep{chandra1984constant}. To establish
\(\mathsf{TC}^0\)-hardness, it therefore suffices to reduce
\(\operatorname{Count}_N\) to \(\operatorname{LMul}\). 

Given \(x\), we set
\[
    d=\left\lceil\log_{10}(N+2)\right\rceil,
    \qquad Q=10^d,
\]
and define the positive decimal integers \(A\) and \(B\) by
\[
    A=1+\sum_{i=0}^{N-1}x_iQ^{i+1},
    \qquad
    B=\sum_{j=0}^{N-1}Q^j.
\]

Their decimal encodings consist of \(d\)-digit blocks and are produced by an
\(\mathsf{AC}^0\) circuit, since every output digit is either a constant or an
input bit \(x_i\). When \(AB\) is written in base \(Q\), every coefficient is
at most \(N+1<Q\), so no carry crosses a \(d\)-digit block, and the coefficient
of \(Q^N\) is exactly \(\operatorname{Count}_N(x)\).

The final row of the worksheet contains \(AB\), from which an \(\mathsf{AC}^0\)
readout extracts this block. Its \(d=O(\log N)\) decimal digits have an
\(O(\log N)\)-bit encoding. A non-uniform lookup on these bits has
\(2^{O(\log N)}=N^{O(1)}\) entries and converts the block to binary in
\(\mathsf{AC}^0\). Thus \(\operatorname{Count}_N\), and hence every language
in \(\mathsf{TC}^0\), reduces to \(\operatorname{LMul}\).

The fixed digit glyphs and worksheet grid similarly allow parallel local
lookups to decode \((A,B)\) from the operand image and
\(\operatorname{LMul}(A,B)\) from the worksheet image, and conversely to render both as
images. These conversions lie in \(\mathsf{AC}^0\), so the same classification
holds for the long-multiplication task.
\end{proof}

\subsection{Long Text Rendering}
\label{app:text_rendering_complexity}

As illustrated in the right column of Figure~\ref{fig:theory_aligned_tasks},
long text rendering uses a fixed monospaced raster font over a finite alphabet
\(\Sigma\) with \(\lvert\Sigma\rvert\ge2\). We call the font nontrivial if its
glyph-raster map is not constant, that is, if some two characters render
differently. We normalize horizontal distances to character cells, so the
width budget is an integer \(W \ge 3\).

An input consists of a sequence $((s_i, p_i))_{i=1}^m$ where
$s_i \in \Sigma^+$ is the $i$-th word and $p_i \in \{0,1\}$ is a
paragraph-boundary indicator. By convention, $p_1 = 1$. For $i > 1$,
$p_i = 1$ means that $s_i$ begins a new paragraph after one blank
line, whereas $p_i = 0$ means that it continues the current paragraph.

The words \(s_i\) determine the glyph content, while the auxiliary wrapping
function below determines their positions. Let
$a_i := |s_i| \in \{1, \dots, W\}$ be the length of word $i$. We
define three auxiliary variables for each word: $\ell_i$ is its line
index, starting from $1$; $\zeta_i$ is its horizontal starting
position, starting from $0$; and $r_i$ is the occupied width of the
line after placing $s_i$, equivalently the exclusive right boundary
of $s_i$.

With
\((\ell_1,\zeta_1,r_1):=(1,0,a_1)\), we define, for \(i>1\),
\[
    (\ell_i,\zeta_i,r_i)=
    \begin{cases}
        (\ell_{i-1}+2,\,0,\,a_i),
        & p_i=1,\\
        (\ell_{i-1},\,r_{i-1}+1,\,r_{i-1}+1+a_i),
        & p_i=0\ \text{and}\ r_{i-1}+1+a_i\le W,\\
        (\ell_{i-1}+1,\,0,\,a_i),
        & p_i=0\ \text{and}\ r_{i-1}+1+a_i>W,
    \end{cases}
\]
and denote the sequence \(((\ell_i,\zeta_i))_{i=1}^m\) by
\(\operatorname{Wrap}_W\bigl((a_i,p_i)_{i=1}^m\bigr)\). The three cases are the
greedy rule stated in the task descriptions above, where a word joins the
current line whenever it fits with a single interword space.

The long text rendering task is the function \(\operatorname{Render}_W\). On
input \(((s_i,p_i))_{i=1}^m\), it returns the full-height rasterization of these
words under the rule above. Its canvas has \(W\) character cells per line and \(2m\) line
slots, with unused slots left white. The largest line index is \(2m-1\), so the
canvas accommodates every input and has a fixed output length given \(m\).

\begin{proposition}[Long text rendering]
\label{prop:text_rendering_complexity}
For every fixed \(W\ge3\) and fixed nontrivial monospaced raster font over an
alphabet \(\Sigma\) with \(\lvert\Sigma\rvert\ge2\), the function
\(\operatorname{Render}_W\) is
\(\mathsf{TC}^0\)-complete under \(\mathsf{AC}^0\) Turing reductions.
\end{proposition}

\begin{proof}
We first classify the auxiliary wrapping function.
To show that \(\operatorname{Wrap}_W\in\mathsf{TC}^0\), we track the
occupied width with a finite automaton. For each input pair \((a,p)\), its transition on
the state set \(\{1,\ldots,W\}\) is
\[
    \tau_{a,p}(r)=
    \begin{cases}
        a, & p=1,\\
        r+1+a, & p=0\ \text{and}\ r+1+a\le W,\\
        a, & p=0\ \text{and}\ r+1+a>W,
    \end{cases}
\]
and we let \(M_W\) be the transition monoid generated under composition by
\(\{\tau_{a,p}:1\le a\le W,\ p\in\{0,1\}\}\).

For any initial state \(r_0\), reading the first \(i\) input pairs yields
\[
    r_i=(\tau_{a_i,p_i}\circ\cdots\circ\tau_{a_1,p_1})(r_0).
\]

We arrange the states in cyclic order. For \(p=0\), the list
\(\tau_{a,0}(1),\ldots,\tau_{a,0}(W)\) advances in that order, wraps around at
most once, and may repeat \(a\); for \(p=1\), the transition is constant. Both
are orientation-preserving transformations of the cycle, and such
transformations are closed under composition.

Every subgroup \(G\) of \(M_W\) is cyclic. To see this, let \(e\) be the
identity of \(G\) and \(S=e(\{1,\ldots,W\})\). Although \(e\) need not fix
every state, it fixes every element of \(S\). Since \(g=g\circ e\), each
\(g\in G\) is determined by its restriction to \(S\), where it acts as a
cyclic-order-preserving permutation.
Restriction to \(S\) therefore embeds \(G\) into the rotation group of \(S\).
Since this rotation group is cyclic, so is \(G\). In particular, every subgroup
of \(M_W\) is solvable.

By the characterization of regular languages recognized by finite monoids with
solvable subgroups~\citep{barrington1992regular}, every state predicate of this
automaton lies in \(\mathsf{ACC}^0\subseteq\mathsf{TC}^0\). Applying these
state-predicate circuits to every prefix gives a one-hot encoding of every
\(r_i\), still with constant depth and polynomial size. The first word needs no
special treatment: since \(p_1=1\), the map \(\tau_{a_1,1}\) is constant with
value \(a_1\), so applying it to any initial state returns \(r_1=a_1\).

For each \(i\ge2\), the automatic-wrap indicator \(\chi_i\) and the total line
increment \(\Delta_i\) are
\[
    \chi_i:=(1-p_i)\mathbf{1}\{r_{i-1}+1+a_i>W\},
    \qquad
    \Delta_i:=\chi_i+2p_i,
\]
and computing them requires only a constant-size lookup once \(r_{i-1}\),
\(a_i\), and \(p_i\) are known. Threshold circuits compute all prefix sums
\[
    \ell_i=1+\sum_{j=2}^{i}\Delta_j
\]
in constant depth, and the starting position is \(\zeta_i=0\) when
\(\chi_i+p_i>0\) and \(\zeta_i=r_{i-1}+1\) otherwise. Thus the entire wrapping
function is in \(\mathsf{TC}^0\).

To establish \(\mathsf{TC}^0\)-hardness, we use the single-paragraph subfamily,
setting \(p_1=1\) and \(p_i=0\) for every \(i>1\), and reduce
\(\mathrm{MAJORITY}\) to a single query of \(\operatorname{Wrap}_W\). Given \(x\in\{0,1\}^N\), we begin with a
sentinel word of width \(W\), so that the first line is full, and encode each
input bit \(x_i\) by the two-word gadget \(\gamma(x_i)\) given by
\[
    \gamma(0):=(1,W-2),
    \qquad
    \gamma(1):=(W,W).
\]
The gadgets \(\gamma(x_0),\ldots,\gamma(x_{N-1})\) are appended after the
sentinel.

Whenever a gadget begins, the current line is full. The \(0\)-gadget starts one
new line and fills it exactly, because \(1+1+(W-2)=W\), while the \(1\)-gadget
starts two new lines, each filled by one width-\(W\) word. Both gadgets
therefore end with a full line, contributing one and two line breaks
respectively, so the final word has line index
\[
    1+N+\sum_{i=0}^{N-1}x_i.
\]

The width sequence is an \(\mathsf{AC}^0\) encoding of the input bits, and an
\(\mathsf{AC}^0\) comparison of the returned line index with a hardwired
threshold decides \(\mathrm{MAJORITY}\). Since \(\mathrm{MAJORITY}\) generates
\(\mathsf{TC}^0\) under \(\mathsf{AC}^0\) Turing
reductions~\citep{chandra1984constant}, the wrapping function is
\(\mathsf{TC}^0\)-hard.

We now transfer this classification to the rendered image. To compute
\(\operatorname{Render}_W\) with an oracle for \(\operatorname{Wrap}_W\), first
obtain each \(a_i=|s_i|\) by a constant-size lookup and query the word positions
\((\ell_i,\zeta_i)\). Each output pixel is then a fixed glyph lookup on the
characters of \(s_i\), followed by an OR over the words, all in
\(\mathsf{AC}^0\). Since
\(\operatorname{Wrap}_W\in\mathsf{TC}^0\), this reduction places
\(\operatorname{Render}_W\) in \(\mathsf{TC}^0\).

For the reverse reduction, nontriviality provides two characters
\(u,v\in\Sigma\) whose glyph rasters differ at a fixed pixel. For each word index
\(i\in\{1,\ldots,m\}\), form two rendering inputs with the same paragraph
indicators. Both inputs represent word \(j\) by \(u^{a_j}\) for every
\(j\in\{1,\ldots,m\}\). In the second input, replace the first \(u\) of the
\(i\)-th word by \(v\).

The two inputs have identical layouts because the font is monospaced.
Comparing their bitmaps at that pixel offset in every character cell
gives a one-hot encoding of the starting cell \((\ell_i,\zeta_i)\). An
\(\mathsf{AC}^0\) circuit converts this encoding to binary. All \(O(m)\) pairs
are queried in parallel with polynomially many oracle gates. This gives an
\(\mathsf{AC}^0\) Turing reduction from \(\operatorname{Wrap}_W\) to
\(\operatorname{Render}_W\). Since \(\operatorname{Wrap}_W\) is
\(\mathsf{TC}^0\)-hard, \(\operatorname{Render}_W\) is
\(\mathsf{TC}^0\)-hard as well.
\end{proof}

\clearpage

\section{Experimental Details}
\label{app:exp_details}

\subsection{Overall Setup}
\label{app:architecture}

\paragraph{Backbones.}
For line tracing, we train a multimodal diffusion transformer (MMDiT) with
9M parameters from random initialization, following the \texttt{FLUX.2}
architecture~\citep{blackforestlabs2026flux2klein}. It has three joint-attention blocks and
three single-stream blocks, each with four attention heads and head dimension
\(64\).
Other runs fine-tune pretrained backbones with LoRA adapters (rank and
alpha \(32\), no dropout). Long multiplication, long text rendering, and the
controlled comparison of \Sref{sec:diffthinker_tasks} use
\texttt{FLUX.2-klein-base-4B}~\citep{blackforestlabs2026flux2klein}. The large-scale
comparison in \Sref{sec:diffthinker_tasks} instead uses \texttt{Qwen-Image-Edit-2511}~\citep{qwenteam2025qwenimageedit2511}.

\paragraph{Implementation of self-conditioning.}
All backbones encode images into VAE latents and pack each \(2\times2\) latent
patch into one image token. In the second training pass and at each sampling
step after the first, \(\SelfCondFM\) concatenates the tokens of the endpoint
prediction with the tokens of the noisy target and of any conditioning image
before the transformer blocks. The \texttt{FLUX.2} backbones append the tokens of
the endpoint prediction with a distinct positional encoding offset, and
\texttt{Qwen-Image-Edit} places them after the tokens of the target image. No
attention mask separates these blocks of tokens. Both
placements follow the backbones' own design for image-conditioned generation.
Thus self-conditioning changes the input sequence but not the transformer
architecture.

\paragraph{Training FLOPs.} We match training compute by cumulative FLOPs, the total number of
floating-point operations a run performs. We estimate it as the
training FLOPs per micro-batch, averaged over 100 micro-batches, multiplied by the number of
micro-batches processed across all devices. Each measurement covers the
complete training computation, including VAE encoding, the second
\(\SelfCondFM\) training pass, and the longer self-conditioned token sequence,
so the extra cost of self-conditioning is fully counted. 
The plotted budgets in Figure~\ref{fig:theory_aligned_tasks}
are \(4.50\times10^{18}\) FLOPs for line tracing, \(2.32\times10^{19}\) for
long multiplication, and \(4.70\times10^{19}\) for long text rendering.

\subsection{Setup and Evaluation for the Tasks in \Sref{sec:theory_aligned_tasks}}
\label{app:tasks_setup}

\paragraph{Shared configuration.}
All training runs are conducted on four NVIDIA A100 80GB GPUs, using AdamW, bfloat16 precision, and a constant learning rate schedule with \(1{,}000\) warmup steps.
Within each task, the two methods share the data, optimizer, noise schedule,
and evaluation procedure. The \(\SelfCondFM\) runs use the two-pass training
scheme of \Sref{sec:scmechanism}.

\paragraph{Line tracing.}
The training set contains $10^6$ random instances with \(m=50\). We train with a
global batch size of \(32\) and a learning rate of \(2\times10^{-3}\). After
FLOP alignment, the plotted curves end at step 60k for \(\VanillaFM\)
and step 53k for \(\SelfCondFM\).

We evaluate checkpoints on a held-out set of 20 instances using 100 NFE. The evaluator
samples interior locations on the five expected segments of each stage and
requires their local average RGB values to match the ground-truth rendering
within a fixed tolerance. Our evaluation metric is the \emph{longest correct prefix} (LCP), defined for each instance as the length of the longest prefix of stages for which all five segments match.

\paragraph{Long multiplication.}
The training set contains $2\times 10^5$ problems with operands of five to
seven digits, where the number of digits of each operand is drawn independently and
uniformly from $\{5,6,7\}$. Each operand is then sampled uniformly at random
from the integers with the chosen number of digits, so it contains no leading
zeros.
We use a global batch size of \(8\) and a learning rate of \(10^{-4}\),
training \(\SelfCondFM\) for five epochs and \(\VanillaFM\) for 15 epochs to
cover a comparable FLOP range.

We evaluate checkpoints on a test split of 100 instances drawn from the same
distribution as the training set, using 50 NFE. We use \texttt{Qwen3.8-Flash}~\citep{qwen38flashnext}
as the judge, prompting it to compare the generated and reference worksheets. A sample counts as correct only when the two worksheets agree throughout, from the final product to every intermediate row.
The full judge prompt is given in Figure~\ref{fig:lm_judge_prompt}.

\begin{figure}[ht]
\centering
\setlength{\fboxsep}{9pt}
\hbox{{\color{black!60}\vrule width 2.5pt}\colorbox{gray!10}{\begin{minipage}{\dimexpr\linewidth-2\fboxsep-2.5pt\relax}
\small
\setlength{\parindent}{0pt}
\setlength{\parskip}{6pt}
You are evaluating a visual reasoning model on a column-multiplication task. You will see a problem setup image (the two numbers to multiply, written in column form), the correct solution image (the fully worked-out column multiplication with the correct final product), and the model's predicted solution image. 

Decide whether the predicted solution is correct: it must correspond to the same problem setup as the correct solution, the final product it shows must equal the correct product, and the intermediate steps, including the carries, must match the correct solution as well. If the predicted solution is blank, truncated, or unreadable, answer `[NO]`. Provide a brief explanation, then answer `[YES]` if the predicted solution is correct and `[NO]` if it is not.
\end{minipage}}}
\caption{Judge prompt for evaluating long multiplication.}
\label{fig:lm_judge_prompt}
\end{figure}

\paragraph{Long text rendering.}
The training set contains $3\times 10^5$ text samples as conditions,
with their \(384\times384\) monospaced renderings as targets. The corpus is
drawn from the TinyStories dataset~\citep{eldan2023tinystories}. During
preprocessing, we truncate each text to the longest prefix that fits on the canvas.
We use a global
batch size of \(16\) and a learning rate of \(10^{-4}\), training
\(\VanillaFM\) for 12 epochs and \(\SelfCondFM\) for four epochs.

We evaluate 100 held-out instances using 50 NFE, parsing the text from the
reference and generated images with \texttt{Qwen3.8-Flash}~\citep{qwen38flashnext}, prompted to
preserve all line breaks and whitespace. The rate of perfect rendering is the
fraction of instances whose two parsed texts match exactly.

\begin{figure}[htb]
    \centering
    \includegraphics[width=0.68\linewidth]{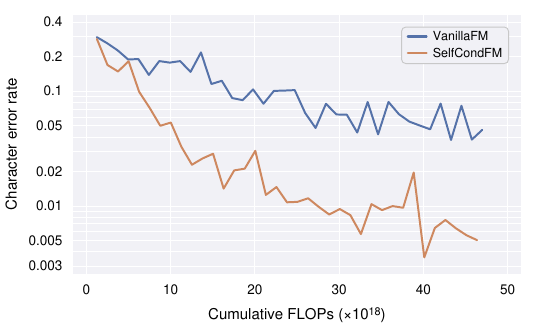}
    \caption{Character error rate (CER) of parsed text from the
    long text rendering evaluation, against cumulative training FLOPs. \(\SelfCondFM\) keeps
    improving throughout training, whereas \(\VanillaFM\) plateaus nearly an
    order of magnitude above the final \(\SelfCondFM\) value.}
    \label{fig:long_text_ocr}
\end{figure}

Beyond exact matching, we also compute the character error rate (CER) between
the two parsed texts (Figure~\ref{fig:long_text_ocr}), defined as
$\mathrm{CER} = d_{\mathrm{Lev}}(r,\hat{r}) / |r|$,
where $r$ and $\hat{r}$ are the parsed text of the
reference and generated images, respectively.
The Levenshtein distance $d_{\mathrm{Lev}}(r,\hat{r})$ is the
minimum number of character edits required to transform $r$
into $\hat{r}$, where each edit inserts, deletes, or substitutes
one character and has unit cost. \(\SelfCondFM\) ends at
CER \(0.005\), while \(\VanillaFM\) plateaus
between \(0.03\) and \(0.1\) and stands at \(0.046\) at the matched budget of
\(4.70\times10^{19}\) FLOPs.

\subsubsection{Multi-Seed Evaluation}
\label{app:multi-seed-diagnostic}

We further train \(\VanillaFM\) and \(\SelfCondFM\) with $3$ different seeds (i.e., $2$ seeds in addition to the one used in Figure~\ref{fig:theory_aligned_tasks}) on line tracing and long multiplication. We evaluate line tracing on the same held-out set of $20$ instances as the single-seed evaluation at $100$ NFE, and long multiplication on the same held-out set of $100$ instances as its single-seed evaluation at $50$ NFE. Due to limited computational resources, we train long multiplication with a smaller compute budget and do not run multi-seed experiments for long text rendering.

Figures~\ref{fig:multi_seed_line_tracing} and~\ref{fig:multi_seed_long_multiplication} show the mean evaluation curve of each method against training FLOPs, with a band of one standard deviation around the mean.  For line tracing, at the checkpoints with matched training FLOPs (\(\SelfCondFM\)@53k vs.\ \(\VanillaFM\)@60k), \(\SelfCondFM\) saturates at a mean LCP of $50.0 \pm 0.0$, while \(\VanillaFM\) reaches only $5.6 \pm 2.7$. For long multiplication, at its smaller matched budget (\(\SelfCondFM\)@50k vs.\ \(\VanillaFM\)@125k), \(\SelfCondFM\) attains $(90.0 \pm 5.3)\%$ accuracy while \(\VanillaFM\) reaches $(35.0 \pm 8.9)\%$. These multi-seed results thus confirm that \(\SelfCondFM\) outperforms \(\VanillaFM\) on line tracing and long multiplication, consistent with the single-seed evaluations in Figure~\ref{fig:theory_aligned_tasks}.

\begin{figure}[htb]
    \centering
    \includegraphics[width=0.7\linewidth]{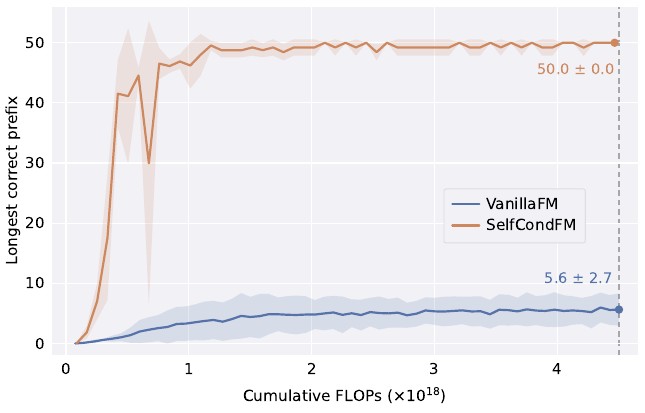}
    \caption{Training curves of line tracing with $3$ different seeds, plotting mean LCP on the held-out set of 20 instances at $100$ NFE against training FLOPs. For each method, the curve shows the mean over the $3$ runs, and the shaded band shows the mean $\pm$ one standard deviation. \(\SelfCondFM\) saturates at $50.0\pm 0.0$, while \(\VanillaFM\) reaches $5.6\pm 2.7$ at the matched budget.}
    \label{fig:multi_seed_line_tracing}
\end{figure}

\begin{figure}[htb]
    \centering
    \includegraphics[width=0.7\linewidth]{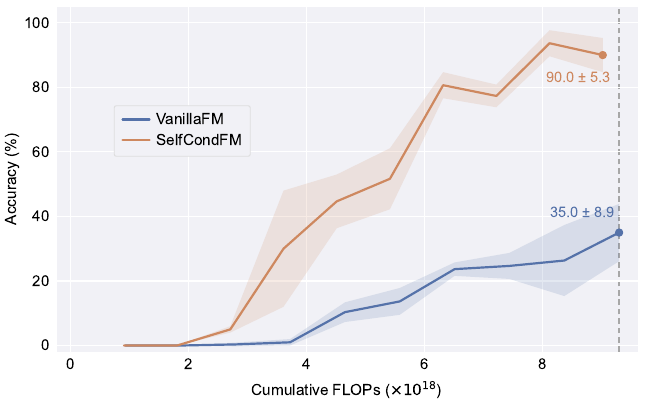}
    \caption{Training curves of long multiplication with $3$ different seeds, plotting the mean accuracy on the held-out set under $50$ NFE against training FLOPs.  As in Figure~\ref{fig:multi_seed_line_tracing}, each curve and its shaded band show the mean and the mean $\pm$ one standard deviation over the $3$ runs. \(\VanillaFM\) reaches $(35.0\pm 8.9)\%$ while \(\SelfCondFM\) is higher at $(90.0\pm 5.3)\%$ under the fixed budget.}
    \label{fig:multi_seed_long_multiplication}
\end{figure}

\subsubsection{Generated Examples}
\label{app:diagnostic_examples}
Figures~\ref{fig:qual_line_tracing} to~\ref{fig:qual_text_rendering} present selected held-out examples on which \(\VanillaFM\) fails while \(\SelfCondFM\) succeeds, with annotations highlighting the errors made by \(\VanillaFM\).
They are generated by the checkpoints at the final evaluation point of Figure~\ref{fig:theory_aligned_tasks}, namely \(\VanillaFM\)@60k and \(\SelfCondFM\)@53k for line tracing, \(\VanillaFM\)@312.5k and \(\SelfCondFM\)@125k for long multiplication, and \(\VanillaFM\)@190k and \(\SelfCondFM\)@74k for long text rendering.

\begin{figure}[p]
    \centering
    \includegraphics[width=\linewidth]{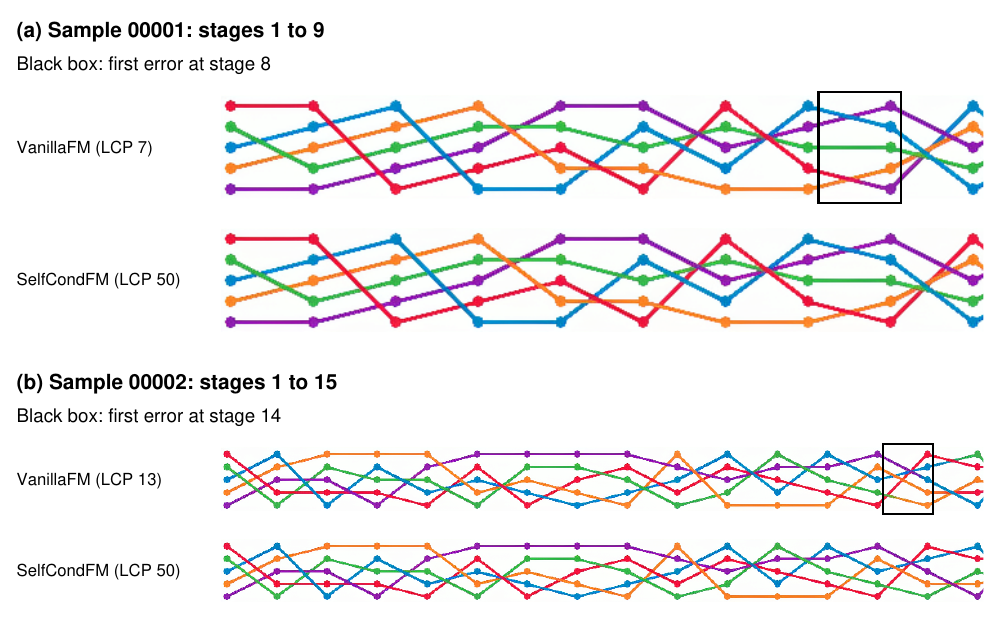}
    \caption{Examples on line tracing. For visualization, each trajectory is truncated shortly after the first error of \(\VanillaFM\); the full trajectory contains 50 stages. \(\VanillaFM\) first fails at stages 8 and 14 in (a) and (b), respectively, by incorrectly switching colors at line crossings. Black boxes mark the errors. In contrast, \(\SelfCondFM\) preserves each line's color throughout all 50 stages in both examples (LCP \(=50\)). }
    \label{fig:qual_line_tracing}
\end{figure}

\begin{figure}[p]
    \centering
    \includegraphics[width=0.90\linewidth]{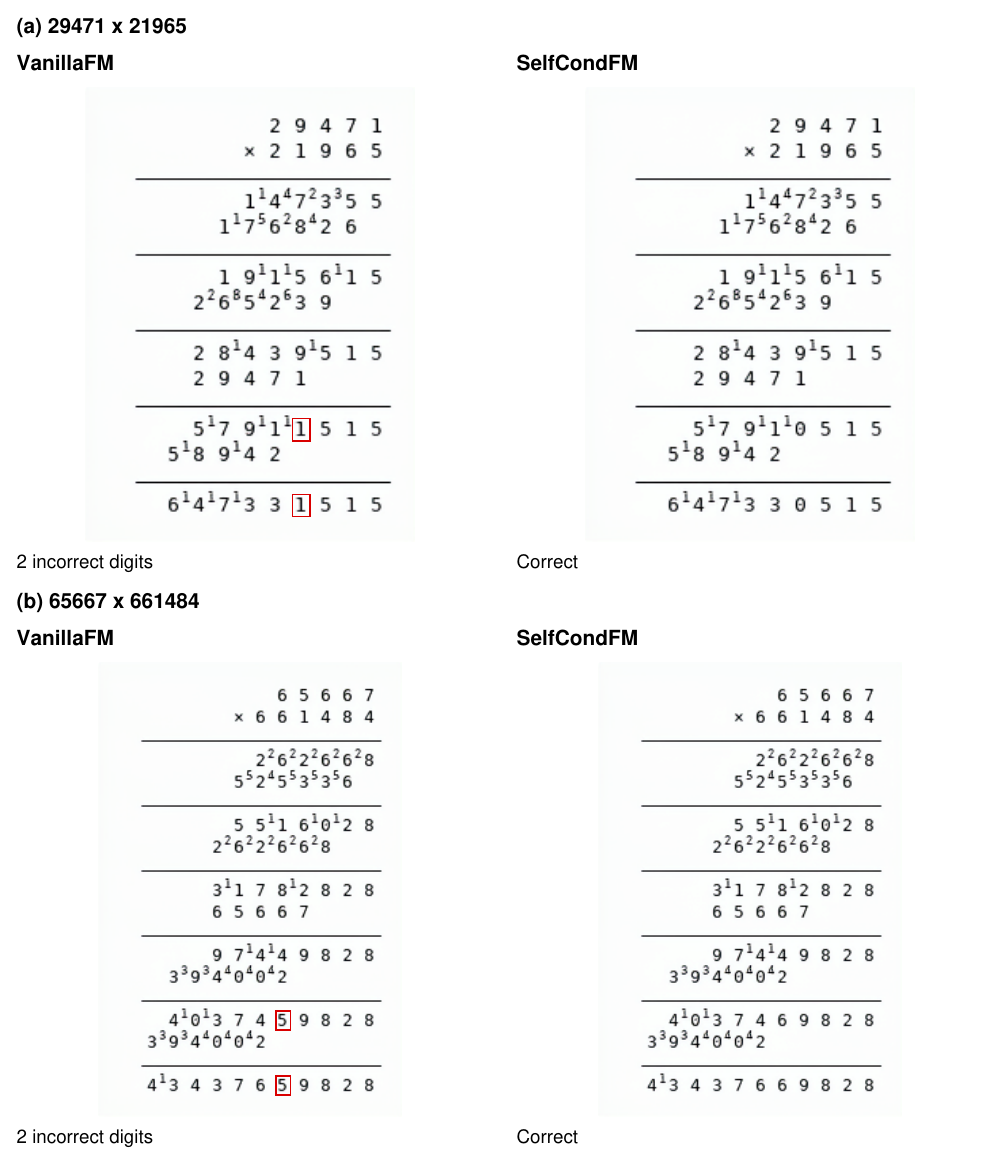}
    \caption{Examples on long multiplication. \(\VanillaFM\) makes accumulation errors in both cases, resulting in incorrect final answers, with the erroneous digits highlighted in red boxes. In contrast, \(\SelfCondFM\) produces the correct product for both examples.}
    \label{fig:qual_multiplication}
\end{figure}

\begin{figure}[p]
    \centering
    \includegraphics[width=\linewidth]{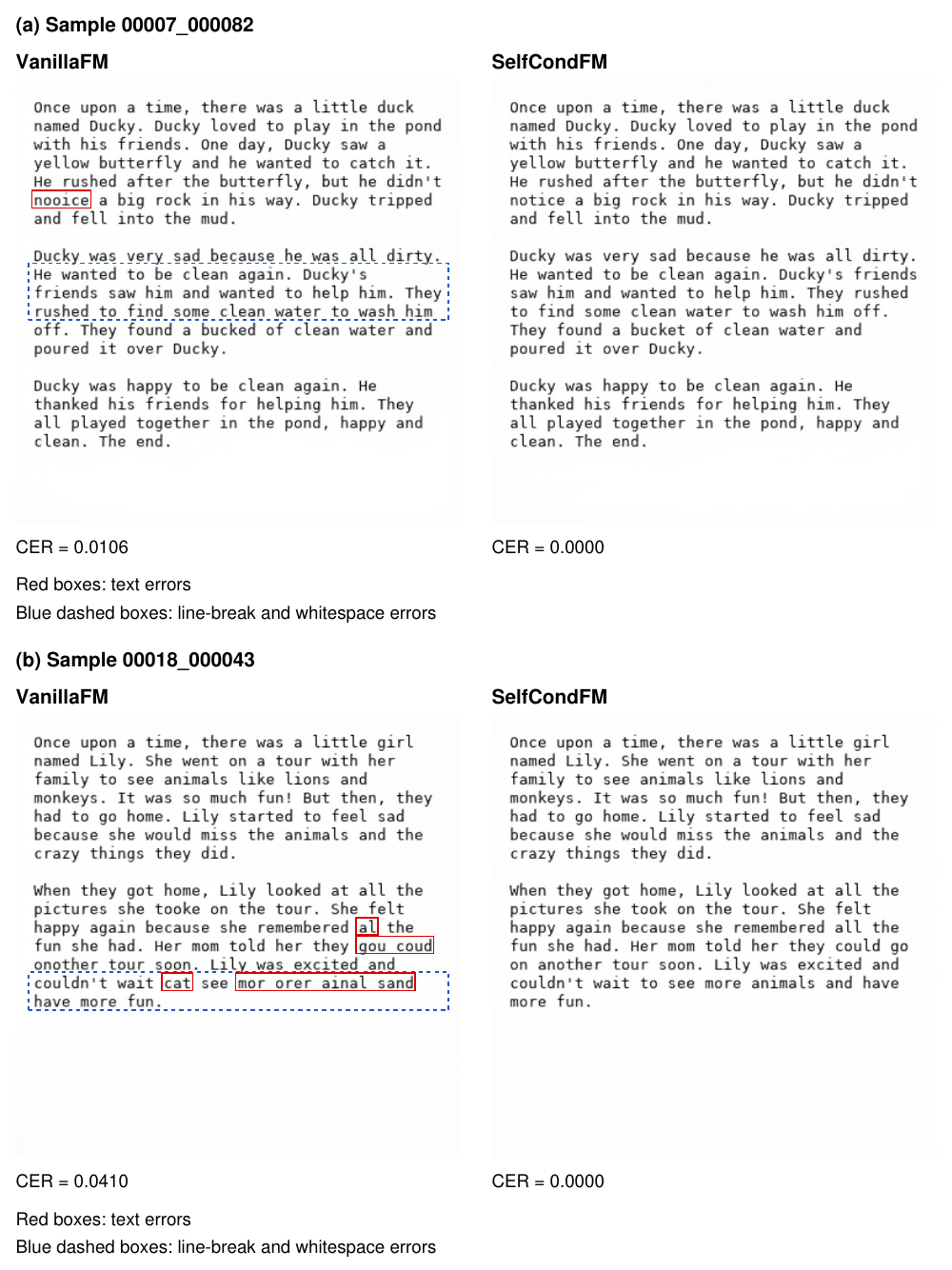}
    \caption{Examples on long text rendering. \(\VanillaFM\) produces text errors as well as line-break and whitespace errors, yielding CERs of 0.0106 and 0.0410 in (a) and (b), respectively. Red boxes highlight text errors, and blue dashed boxes indicate line-break and whitespace errors. In contrast, \(\SelfCondFM\) renders both passages correctly with zero CER. }
    \label{fig:qual_text_rendering}
\end{figure}

\clearpage

\subsection{NFE Scaling}
\label{app:nfe_scaling}

For Figure~\ref{fig:nfe_scaling}, we evaluate the checkpoints of Appendix~\ref{app:tasks_setup} for line tracing at matched training FLOPs, \(\VanillaFM\)@60k and \(\SelfCondFM\)@53k, on \(200\) held-out instances with \(m=50\). The NFE grid covers every integer from \(1\) to \(50\), and the two methods share the same initial noise for each instance at every NFE. The LCP is computed with the evaluator described in Appendix~\ref{app:tasks_setup}.

Beyond the comparison at matched NFE in Figure~\ref{fig:nfe_scaling}, we also plot the mean LCP against cumulative inference FLOPs in Figure~\ref{fig:inference_flops}.

\begin{figure}[htb]
    \centering
    \includegraphics[width=0.68\linewidth]{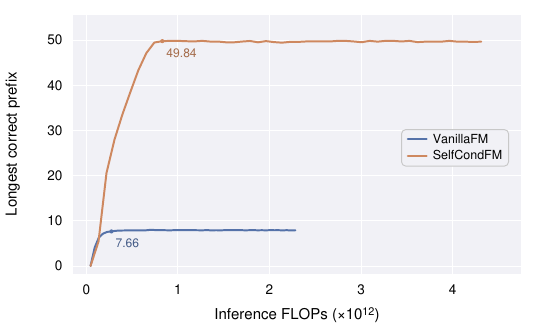}
    \caption{Mean LCP in the NFE scaling comparison of Figure~\ref{fig:nfe_scaling}, plotted against cumulative inference FLOPs. Both methods have zero mean LCP at \(\text{NFE}=1\). \(\SelfCondFM\) reaches 49.84 at \(8.29\times10^{11}\) FLOPs. \(\VanillaFM\) reaches 7.66 at \(2.74\times10^{11}\) FLOPs, ending below 8.0 up to \(2.28 \times 10^{12}\) FLOPs (which corresponds to \(\text{NFE}=50\)).}
    \label{fig:inference_flops}
\end{figure}

\subsection{Setup and Results for Structured Visual-Reasoning Tasks}
\label{app:diffthinker_setup}

\paragraph{Datasets and evaluators.} We prepare the training data with the
pipeline of \citet{dai2026endocot} for VSP-Super, Maze, TSP, and Sudoku, and
with that of \citet{he2025diffthinker} for Jigsaw. The evaluation data and
protocol for all five tasks follow \citet{he2025diffthinker}.

\paragraph{Controlled experiments.}
In the controlled experiments of \Sref{sec:diffthinker_tasks}, for each of VSP-Super, Maze, TSP, Sudoku, and Jigsaw, we train separate
\texttt{FLUX.2-klein-base-4B} LoRA adapters under the shared configuration of
the tasks in \Sref{sec:theory_aligned_tasks}, with a global batch size of \(16\) and a learning
rate of \(10^{-4}\). Checkpoints are compared at matched cumulative training
FLOPs, at which \(\SelfCondFM\) has been trained for 16k steps, and all results
use 50 NFE.

\paragraph{Large-scale comparisons.}
Following the setup of \citet{dai2026endocot}, we train two kinds of
$\SelfCondFM$ models. For task-specific training, we train one
\texttt{Qwen-Image-Edit-2511} LoRA adapter per task. For unified training, a
single adapter is trained on the union of the VSP-Super, Maze, TSP, and Sudoku
data. All adapters are trained for five epochs to ensure convergence.

\subsubsection{Jigsaw Results}
\label{app:jigsaw_benchmark}

Since EndoCoT~\citep{dai2026endocot} does not evaluate Jigsaw, we compare
our large-scale result on Jigsaw with the baselines reported by
\citet{he2025diffthinker}. Our \(\SelfCondFM\) adapter attains the best
average accuracy, \(96.0\%\) versus \(92.3\%\) for DiffThinker, with the
largest gain at the hardest level; see Table~\ref{tab:jigsaw_benchmark}.

\begin{table}[H]
    \centering
    \renewcommand{\arraystretch}{1.06}
    \caption{Jigsaw accuracy (\%) at difficulty levels 2 to 4. ``Avg.'' is the
    unweighted mean over the three levels. The best results are in bold and the
    second best are underlined.}
    \label{tab:jigsaw_benchmark}
    \begin{tabular}{lcccc}
        \toprule
        Method & 2 & 3 & 4 & Avg. \\
        \midrule
        Gemini-3-Flash & 71 & 16 & 0 & 29.0 \\
        GPT-5 & 30 & 0 & 0 & 10.0 \\
        Qwen3-VL-8B & 7 & 0 & 0 & 2.3 \\
        Qwen3-VL-8B (SFT) & 95 & 56 & 9 & 53.3 \\
        Qwen3-VL-8B (GRPO) & 6 & 0 & 0 & 2.0 \\
        DiffThinker & \textbf{99} & \tightuline{98} & \tightuline{80} & \tightuline{92.3} \\
        \textbf{Ours} & \tightuline{98} & \textbf{100} & \textbf{90} & \textbf{96.0} \\
        \bottomrule
    \end{tabular}
\end{table}

\subsubsection{Generated Examples}
\label{app:structured_examples}

To illustrate the qualitative differences between \(\VanillaFM\) and \(\SelfCondFM\) on the five structured visual-reasoning tasks, we select two held-out instances per task from the controlled comparison in Table~\ref{tab:diffthinker} on which \(\VanillaFM\) fails while \(\SelfCondFM\) succeeds.
The examples are shown in Figures~\ref{fig:qual_vsp} to~\ref{fig:qual_jigsaw}, with annotations highlighting the errors made by \(\VanillaFM\).

\begin{figure}[p]
    \centering
    \includegraphics[width=0.90\linewidth]{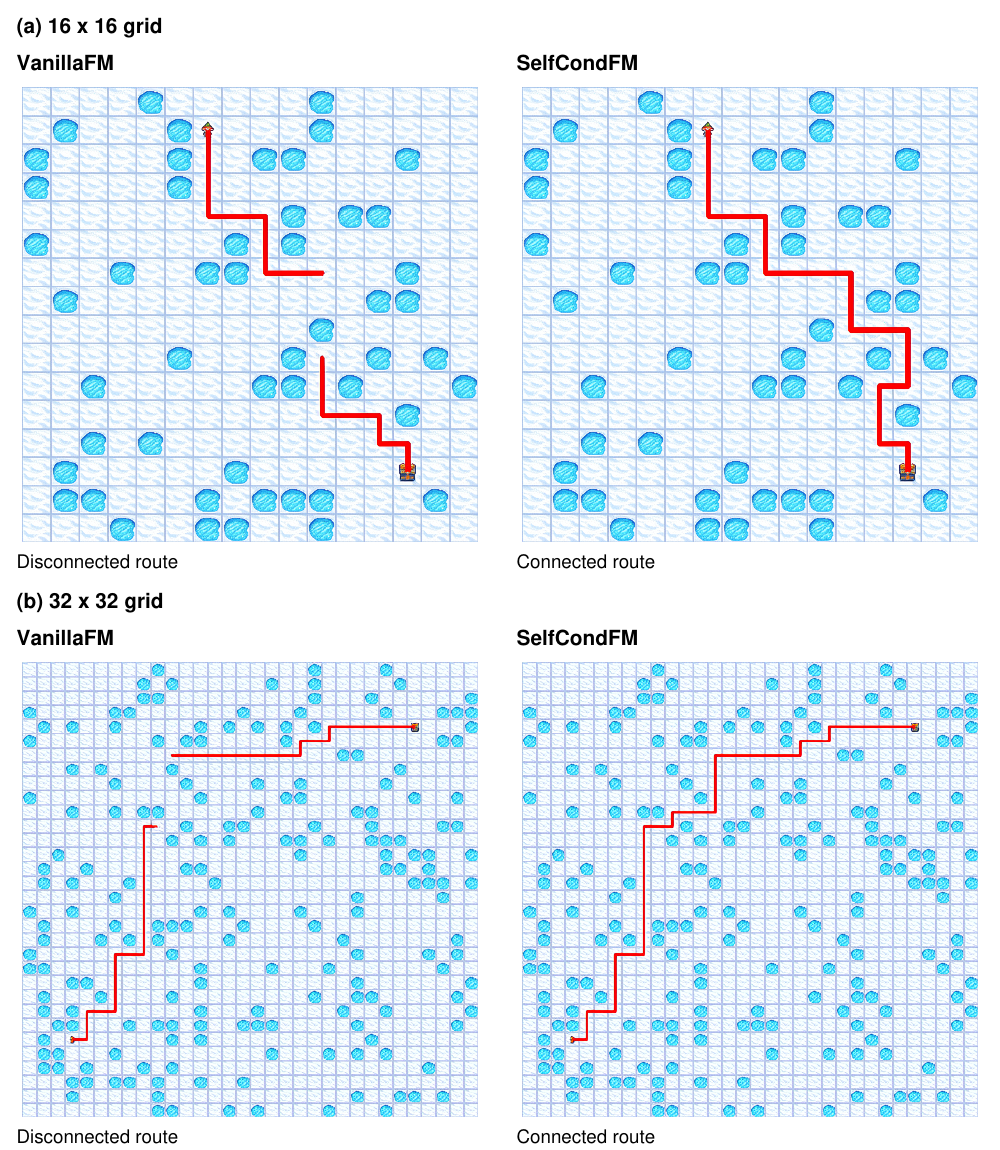}
    \caption{VSP-Super examples on \(16\times16\) and \(32\times32\) grids. The \(\VanillaFM\) 
    routes contain disconnected segments. In contrast, the \(\SelfCondFM\) routes correctly connect the
    endpoints.}
    \label{fig:qual_vsp}
\end{figure}

\begin{figure}[p]
    \centering
    \includegraphics[width=0.90\linewidth]{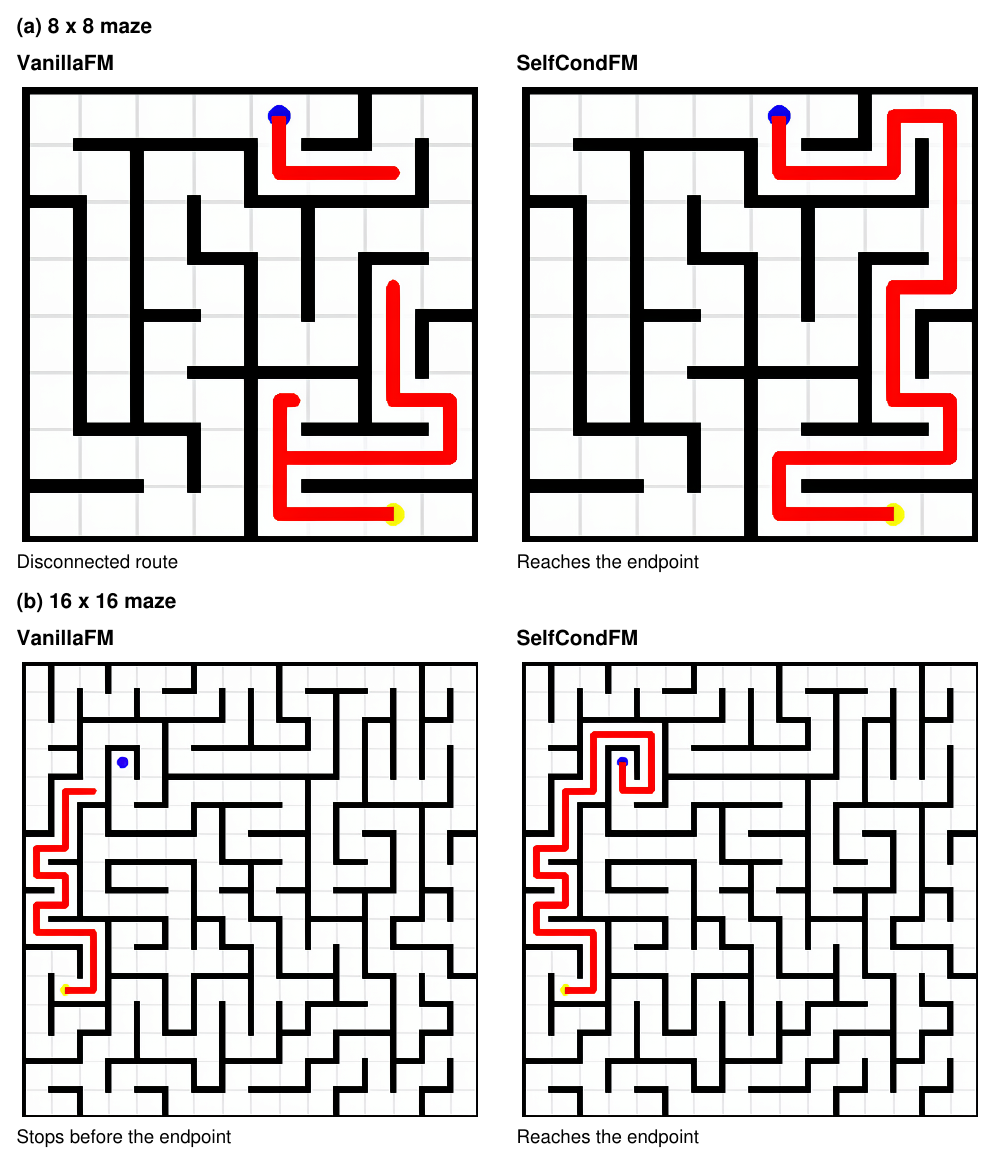}
    \caption{Maze examples on \(8\times8\) and \(16\times16\) grids. In (a), \(\VanillaFM\)
    leaves separate path fragments; in (b), its route stops before reaching
    the blue endpoint. In contrast, \(\SelfCondFM\) completes both routes.}
    \label{fig:qual_maze}
\end{figure}

\begin{figure}[p]
    \centering
    \includegraphics[width=0.90\linewidth]{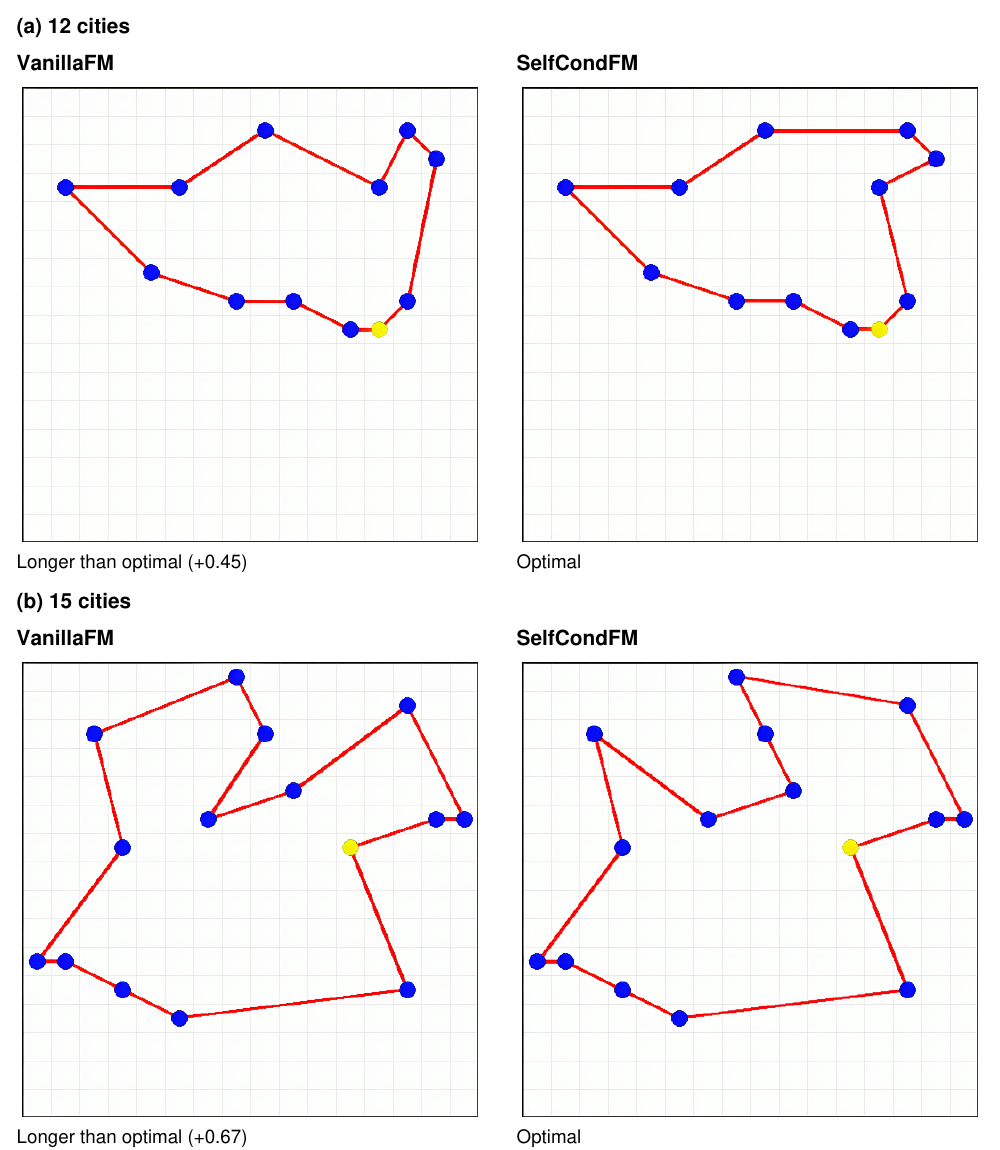}
    \caption{TSP examples with 12 and 15 cities. \(\VanillaFM\) fails to find the optimal tour in both cases, yielding excess tour lengths of +0.45 and +0.67, respectively. In contrast, \(\SelfCondFM\) produces the optimal solution for both instances.}
    \label{fig:qual_tsp}
\end{figure}

\begin{figure}[p]
    \centering
    \includegraphics[width=0.90\linewidth]{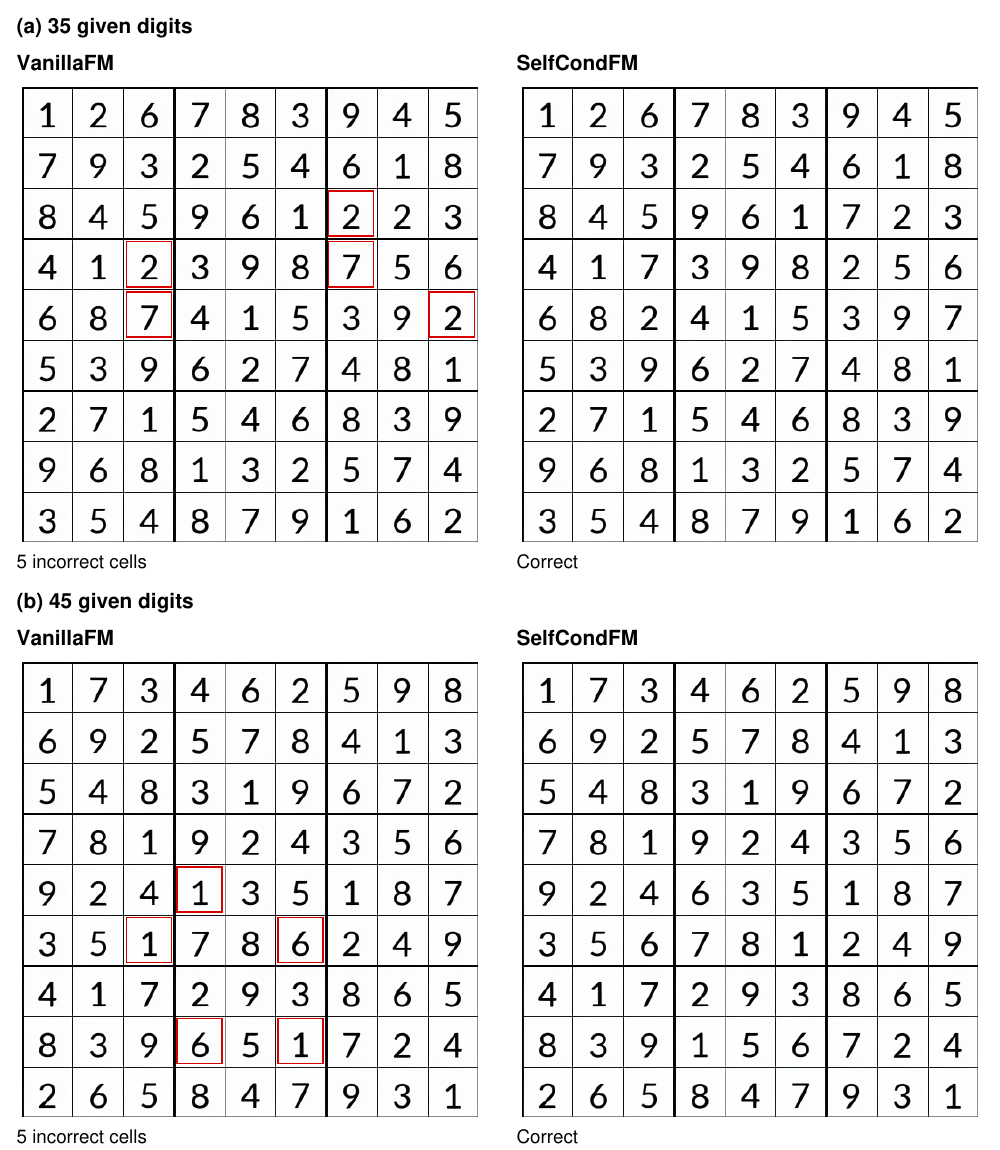}
    \caption{Sudoku examples with 35 and 45 given digits. \(\VanillaFM\) produces incorrect solutions with explicit Sudoku constraint violations: row 3 in (a) contains two 2s, while row 5 in (b) contains two 1s. The red boxes highlight these errors. In contrast, \(\SelfCondFM\) correctly solves both puzzles.}
    \label{fig:qual_sudoku}
\end{figure}

\begin{figure}[p]
    \centering
    \includegraphics[width=0.90\linewidth]{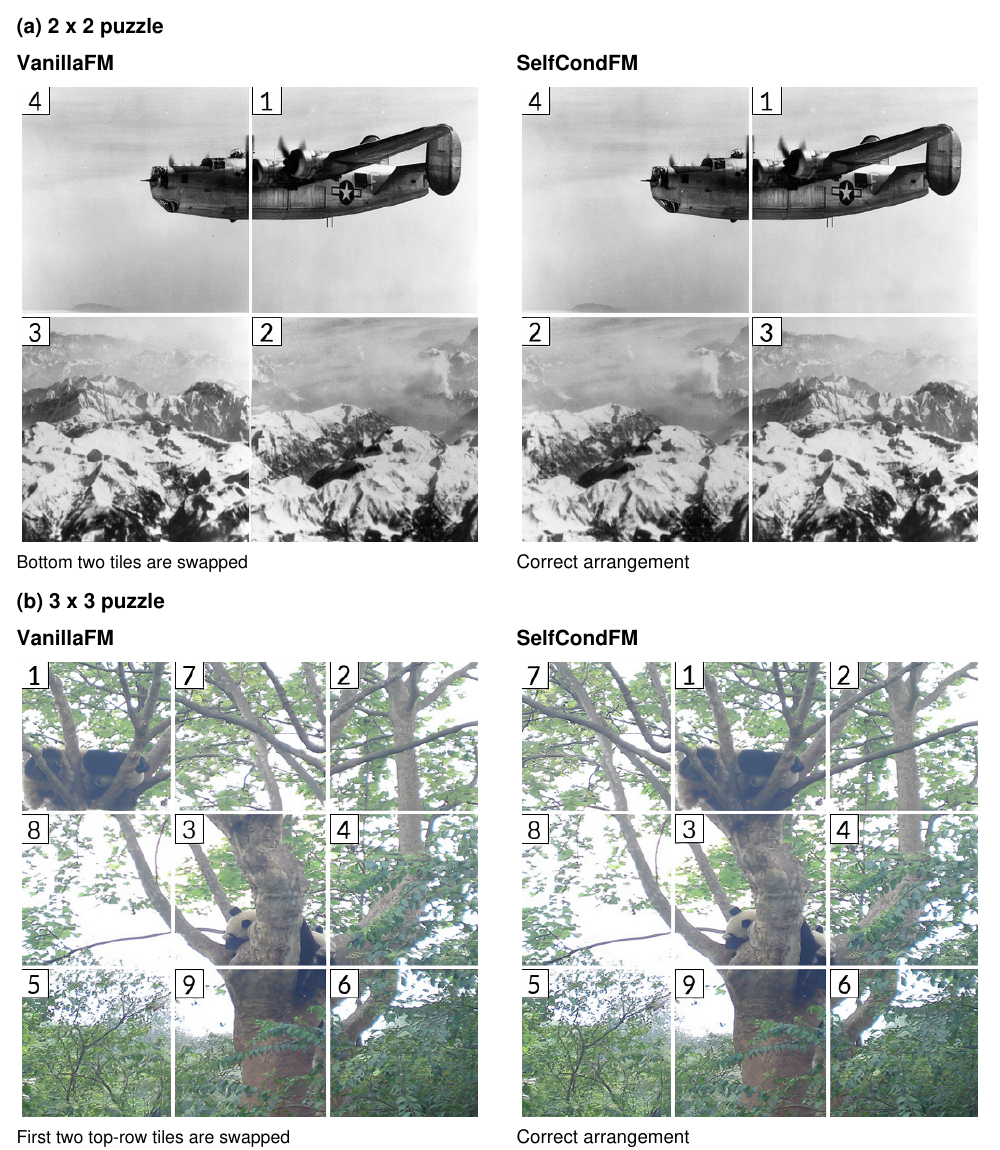}
    \caption{Jigsaw examples on \(2\times2\) and \(3\times3\) tile grids. \(\VanillaFM\) fails in both examples: it swaps the two bottom tiles of (a), giving the order 3, 2 instead of 2, 3, and swaps the first two tiles in the top row of (b), giving 1, 7 instead of 7, 1. \(\SelfCondFM\) matches the reference permutation in both examples. Tile identifiers are part of the original images.}
    \label{fig:qual_jigsaw}
\end{figure}

\clearpage

\subsection{Additional Empirical Analyses}
\label{app:secondary_analysis}

\subsubsection{Empirical Lipschitz Estimation}
\label{app:lipschitzness}

Theorem~\ref{thm:vanillaupperbound} assumes that the model $u$ has a bounded Lipschitz constant. 
To empirically support this assumption, we estimate the Lipschitz constant, in
the joint \(\ell_\infty\) norm on \((x,\ctx)\) used in \Sref{sec:fmsetup}, of a \(\VanillaFM\) model for line tracing. The
checkpoint is separately trained on approximately $3\times 10^5$ examples of mixed lengths, using
about \(2\times10^{17}\) training FLOPs. 

For a state \(x_i\) visited at time
\(t_i\), we use projected gradient ascent to maximize
\[
    \frac{\|u(x_i+\Delta x,t_i;\ctx+\Delta\ctx)-u(x_i,t_i;\ctx)\|_\infty}
    {\max\{\|\Delta x\|_\infty,\|\Delta\ctx\|_\infty\}}.
\]
We jointly perturb the noisy state and the condition latent with
radii \(0.003\), \(0.01\), \(0.03\), \(0.1\), and \(0.3\), and cover every
state of an inference trajectory with 50 steps. For each problem length
\(m=1,\ldots,50\), we estimate $L$ on 30 held-out trajectories and report the maximum
estimate over all tested states, radii, and trajectories.

\begin{figure}[htb]
    \centering
    \includegraphics[width=0.68\linewidth]{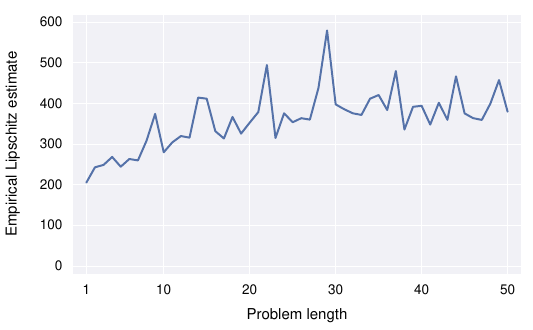}
    \caption{Maximum joint \(\ell_\infty\) empirical Lipschitz estimate among
    30 held-out trajectories at each problem length.}
    \label{fig:empirical_lipschitz}
\end{figure}

As shown in Figure~\ref{fig:empirical_lipschitz}, the maximum Lipschitz estimate at each problem length rises
for small \(m\) and levels off around \(400\) for \(m\ge30\), with an overall
maximum of \(580\). The largest observed velocity norm is below \(10\). The
estimates thus show no continued growth with \(m\) over the tested range.

\textbf{Remark.} These estimates are not certified upper bounds.
Nevertheless, estimating Lipschitz constants via adversarial search is
common practice~\citep{yang2020closer,achour2022existence}, and the absence of
continued growth in the estimated \(L\) with $m$ still provides empirical support for the
regime of bounded Lipschitz constants assumed in \Sref{sec:vanillabarrier}.

\subsubsection{Attention Allocation}
\label{app:attention_allocation}
We measure how the first joint-attention layer of the \(\SelfCondFM\)
checkpoint for line tracing allocates attention from queries at the \(x_i\) token positions to the 
keys from three input sources at sampling step \(i\): the noisy state \(x_i\), the 
conditioning image \(w\), and the endpoint prediction \(z_i\). The
experiment uses 50 test instances with \(m=50\) and 50 NFE. At every step after
the first, we normalize the attention mass over the three sources and average
it over heads and instances.

\begin{figure}[htb]
    \centering
    \includegraphics[width=0.68\linewidth]{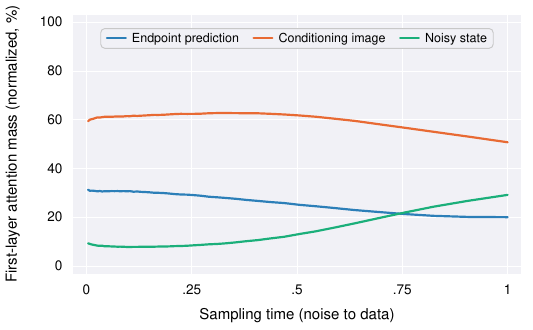}
    \caption{Normalized attention mass of the first layer to
    each input source. Endpoint prediction maintains a nontrivial share throughout the trajectory, declining nearly monotonically from $31.3\%$ to $20.0\%$, thus peaking at early steps, where noise dominates.}\label{fig:attention_allocation}
\end{figure}

As shown in Figure~\ref{fig:attention_allocation}, the attention share of the endpoint prediction stays between $20.0\%$ and $31.3\%$ and declines nearly monotonically along the sampling trajectory. Thus the model relies on $z_i$ throughout sampling and uses it most heavily at early steps, where the noisy state is less informative.

\subsubsection{Parallel and Revisable Reasoning}
\label{app:trajectory_analysis}

By inspecting the sampling trajectories of the checkpoints for long multiplication, we observe two properties of visual generative reasoning with flow matching that $\SelfCondFM$ preserves: parallel local computation and in-place error correction, which are usually not possible in autoregressive reasoning. The following paragraphs provide empirical evidence for each.

\paragraph{Parallel local computation.}
For each generated worksheet, we score every shifted partial product
$AB_i10^i$ against the reference, repeating the evaluation over every integer NFE
from $1$ to $50$. As shown in Figure~\ref{fig:parallel_reasoning}, the
accuracy of partial products for both methods stays above $98\%$ at every NFE,
even at $\text{NFE}=1$.
This shows that even a single model evaluation can produce multiple correct local results simultaneously within the same worksheet.
By contrast, common autoregressive reasoning methods, such as CoT, generate reasoning steps sequentially and do not produce multiple local results in parallel.

\begin{figure}[htb]
    \centering
    \includegraphics[width=0.68\linewidth]{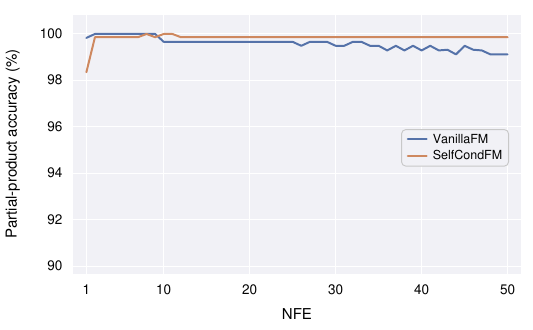}
    \caption{Parallel local computation. Accuracy of partial products on long
    multiplication at every NFE from $1$ to $50$: both methods keep nearly
    all partial products correct within a single generated worksheet.}
    \label{fig:parallel_reasoning}
\end{figure}

\paragraph{In-place error correction.}
For pairs of problems $(w,w')$ with different answers, we sample with 50 NFE, conditioning on $w$ for the first $k$ steps and swapping to $w'$ for
the rest.
We repeat this for every $k$ from $1$ to $50$, and judge each output against
the reference worksheet of $w'$.

As shown in Figure~\ref{fig:error_correction}, \(\SelfCondFM\) keeps the
accuracy of matching $w'$ around $90\%$ whenever the swap falls in the first half
of the trajectory, and the accuracy falls to zero only once $k$ exceeds
$35$. \(\VanillaFM\) also corrects early swaps in place, but its accuracy
is only $65\%$ at $k=1$ and decays to zero by $k=32$.

Both methods therefore revise a wrong generation in place in the middle of the trajectory,
with \(\SelfCondFM\) performing particularly well in the first half of the
trajectory. Autoregressive reasoning, in contrast, retains the reasoning
steps once they are generated and cannot correct them in place.

\begin{figure}[htb]
    \centering
    \includegraphics[width=0.68\linewidth]{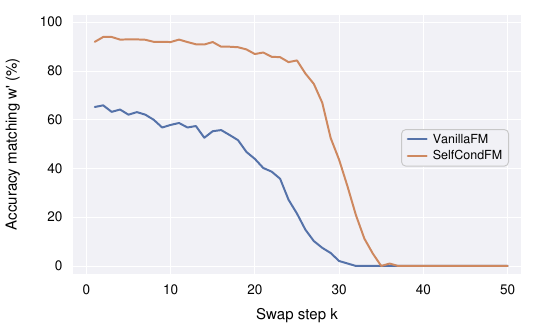}
    \caption{In-place error correction on long multiplication. The condition
    changes from $w$ to $w'$ after step $k$, and the curves show the accuracy of
    matching $w'$ as $k$ varies: the output of \(\SelfCondFM\) is redirected in
    place to the new problem when the swap falls in the first half of the
    trajectory, whereas \(\VanillaFM\) follows the new problem only
    partially.}
    \label{fig:error_correction}
\end{figure}

\end{document}